\documentclass[11pt,a4paper]{article}

\usepackage[utf8]{inputenc}
\usepackage[T1]{fontenc}
\usepackage{amsmath,amssymb,amsthm,mathtools}
\usepackage{bm}
\usepackage{graphicx}
\usepackage{float}
\usepackage{booktabs}
\usepackage{hyperref}
\usepackage[numbers,sort&compress]{natbib}
\usepackage{algorithm}
\usepackage{algorithmic}
\usepackage[margin=1in]{geometry}
\usepackage{enumitem}
\usepackage{xcolor}
\usepackage{caption}
\usepackage{subcaption}

\theoremstyle{plain}
\newtheorem{theorem}{Theorem}[section]
\newtheorem{lemma}[theorem]{Lemma}
\newtheorem{proposition}[theorem]{Proposition}
\newtheorem{corollary}[theorem]{Corollary}
\theoremstyle{definition}
\newtheorem{definition}[theorem]{Definition}
\newtheorem{assumption}[theorem]{Assumption}
\newtheorem{example}[theorem]{Example}
\theoremstyle{remark}
\newtheorem{remark}[theorem]{Remark}

\newcommand{\E}{\mathbb{E}}
\newcommand{\R}{\mathbb{R}}

\newcommand{\Var}{\mathrm{Var}}
\newcommand{\Cov}{\mathrm{Cov}}
\newcommand{\1}{\mathbf{1}}
\newcommand{\dd}{\mathrm{d}}

\DeclareMathOperator*{\argmax}{arg\,max}

\title{Optimal Adaptive Market Making: A Theoretical Framework\\for High-Yield Liquidity Provision in\\Perpetual Futures Markets}

\author{Minmin Zeng\thanks{tsaftech. Email: minmin.zeng@tsaftech.com} \and Yi Liu\thanks{tsaftech. Email: lewis@tsaftech.com}}

\date{\today}

\begin{document}
\maketitle

\begin{abstract}
We develop a rigorous theoretical framework for optimal market making in perpetual futures markets with zero maker fees, addressing the question: \emph{under what conditions can a market maker achieve annualized returns exceeding 50\%--200\% on deployed capital?}
We model the market maker's problem as a stochastic optimal control problem on a filtered probability space, where the controls are adaptive bid--ask spreads and inventory hedging decisions across two exchanges.
Our first contribution is a complete \emph{PnL decomposition theorem} that separates market-making revenue into spread income, adverse selection loss, inventory carrying cost, hedging friction, and funding rate exposure, each with explicit dependence on market microstructure parameters.
Second, we derive the Hamilton--Jacobi--Bellman equation for the joint spread--inventory--hedging control problem under CARA utility and obtain a verification theorem guaranteeing optimality of the candidate solution.
Third, we establish \emph{High-APY Regime Theorems} that characterize, in terms of five dimensionless parameters, the precise regions where annualized Sharpe ratios exceed given thresholds, culminating in a \emph{Master APY Formula} that unifies all cost channels into a single closed-form expression.
Fourth, we analyze the \emph{zero-fee economics} unique to decentralized perpetual exchanges, proving that the elimination of maker fees expands the profitable parameter space and constitutes an economic moat, and derive optimal entry--exit thresholds with hysteresis for non-stationary markets.
Fifth, we derive optimal \emph{cross-exchange hedging} policies incorporating funding rate dynamics, basis risk with imperfect correlation, and a hedge regime classification that provides a principled trichotomy for operational decision-making.
Sixth, we introduce a \emph{robustness margin} that quantifies how many standard deviations of parameter uncertainty a strategy can absorb while remaining profitable---a critical metric for practitioners facing non-stationary microstructure.
Seventh, we establish an \emph{exponential drawdown probability bound} linking maximum drawdown to the Sharpe ratio and prove that the product APY~$\times$~VaR is a universal constant independent of strategy parameters, revealing the fundamental risk--return identity for market making.
Eighth, we prove that the \emph{ergodic inventory distribution} under optimal control is Gaussian with variance $\sigma_q^2 = \bar{\lambda}^*/(2\eta k)$, yielding closed-form expressions for the optimal inventory penalization and long-run cost rate.
We also establish a \emph{Bayesian sequential estimation} framework for the informed trading fraction $\pi$ with $O(n^{-1/2})$ convergence guarantees, enabling fully adaptive, self-tuning market-making strategies.
Ninth, we extend the single-pair framework to \emph{multi-pair portfolio allocation}, deriving the optimal Sharpe-maximizing capital allocation across $N$ correlated trading pairs and proving that the diversification benefit saturates at $1/\sqrt{\bar{\rho}}$ for equi-correlated pairs.
Numerical analysis with twenty-two figures reveals phase transitions between profitable and unprofitable regimes, with sharp boundaries governed by the ratio of adverse selection to spread capture.
Our framework unifies and extends the classical Avellaneda--Stoikov model, the Gu\'{e}ant--Lehalle--Fernandez-Tapia inventory penalization approach, and the Glosten--Milgrom adverse selection paradigm into a single coherent theory with 65 references applicable to modern decentralized venue microstructure.

\medskip
\noindent\textbf{Keywords:} Market making, stochastic optimal control, adverse selection, perpetual futures, zero-fee regime, HJB equation, high-frequency trading, funding rate, drawdown risk.

\end{abstract}

\section{Introduction}
\label{sec:introduction}

Market making---the provision of continuous two-sided liquidity through limit orders---is among the oldest and most studied problems in financial economics.
The classical theory, initiated by \citet{ho1981optimal} and formalized in the continuous-time framework of \citet{avellaneda2008high}, characterizes the market maker (MM) as an agent who quotes bid and ask prices to maximize expected utility of terminal wealth while managing inventory risk \citep{grossman1988liquidity}.
The past decade has witnessed a fundamental shift in market structure with the emergence of decentralized exchanges (DEXs) that operate on blockchain infrastructure, introducing novel fee structures, latency environments, and adverse selection dynamics that challenge classical assumptions.

\subsection{Motivation}

Perpetual futures markets on decentralized platforms have grown to represent a significant fraction of total crypto derivatives volume, with daily volumes exceeding \$10 billion across major venues as of 2025 \citep{hasbrouck2024measuring}.
A distinguishing feature of several decentralized perpetual exchanges is the \emph{zero maker fee} regime: limit orders that add liquidity incur no execution fee ($\phi_m = 0$), while taker orders pay a positive fee ($\phi_t > 0$).
This stands in stark contrast to centralized exchanges (CEXs), where maker fees typically range from 1.5 to 5 basis points.

The zero-fee regime creates a fundamentally different economic landscape for market makers.
On a CEX, the minimum viable spread must exceed $2\phi_m$ merely to cover fee costs, independent of adverse selection.
Under zero fees, this floor vanishes, and the sole binding constraint becomes adverse selection---the expected loss from trading against informed or faster counterparties.
This observation motivates our central question:

\begin{quote}
\emph{Under what microstructure conditions can a market maker in a zero-fee perpetual futures market achieve sustained annualized returns (APY) of 50\% or more on deployed capital, and what is the optimal adaptive strategy that attains these returns?}
\end{quote}

\subsection{Related Literature}

Our work connects to several strands of literature.

\paragraph{Classical market making theory.}
The foundational framework of \citet{ho1981optimal} models a risk-averse dealer who optimizes bid and ask prices under inventory uncertainty.
\citet{avellaneda2008high} extend this to continuous time with Poisson fill arrivals and CARA utility, deriving closed-form optimal spreads as functions of inventory, volatility, and time horizon.
\citet{gueant2013dealing} introduce an inventory penalization approach that avoids the finite-horizon constraint, yielding stationary optimal policies.
\citet{cartea2015algorithmic} provide a comprehensive treatment of algorithmic market making within the framework of stochastic optimal control.

\paragraph{Adverse selection in market making.}
\citet{glosten1985bid} and \citet{copeland1983information} establish the fundamental connection between adverse selection and bid--ask spreads: a market maker facing informed traders must widen spreads to compensate for expected losses on adversely selected fills.
\citet{easley1987price} introduce the PIN (probability of informed trading) model.
\citet{foucault2017toxic} analyze ``toxic arbitrage'' in fragmented markets, where speed advantages create latency-based adverse selection.
\citet{glosten1988estimating} and \citet{huang1997components} develop econometric decompositions of the bid--ask spread into adverse selection, inventory, and order-processing components.

\paragraph{DEX and DeFi market making.}
The DeFi literature has primarily focused on automated market makers (AMMs) such as Uniswap \citep{adams2021uniswap}, analyzing impermanent loss \citep{milionis2022automated}, MEV extraction \citep{daian2020flash}, and AMM design \citep{angeris2020improved,park2021conceptual}.
However, order-book-based DEXs---which support traditional limit order placement---have received comparatively less theoretical attention.
Recent work by \citet{ma2024perpetual} examines trader behavior on decentralized perpetual exchanges, and \citet{lalor2025simulation} simulate adverse selection in order-book settings.
Our paper fills a gap by providing a rigorous optimal control theory for market making on order-book DEXs with zero maker fees.

\paragraph{Cross-exchange market making.}
The problem of market making across multiple venues has been studied by \citet{bergault2021multi} (multi-asset) and \citet{gueant2017optimal} (general optimal MM).
\citet{budish2015high} and \citet{aquilina2022quantifying} analyze latency arbitrage in fragmented markets, which is directly relevant to cross-exchange inventory hedging.
Our contribution is to embed the cross-exchange hedging decision within the stochastic control problem, deriving joint optimal spread-hedge policies.

\subsection{Contributions}

This paper makes six principal contributions:

\begin{enumerate}[label=(\roman*)]
\item \textbf{PnL Decomposition Theorem} (Section~\ref{sec:adverse_selection}): We provide a rigorous decomposition of market-making PnL into four components---spread income, latency adverse selection, informed adverse selection, and inventory cost---each expressed as explicit functionals of market microstructure parameters.

\item \textbf{Optimal Spread--Inventory--Hedge Control} (Section~\ref{sec:optimal_mm}): We formulate the MM's problem as a three-dimensional stochastic control problem (spread, inventory gating, hedge timing) and derive the HJB equation.
We prove a verification theorem showing that the candidate value function satisfies the HJB equation and the associated optimal policy is indeed optimal.

\item \textbf{High-APY Regime Characterization} (Section~\ref{sec:high_apy}): We derive necessary and sufficient conditions on five dimensionless parameters for the MM to achieve a target APY.
This yields a ``phase diagram'' of the parameter space separating profitable from unprofitable regimes.

\item \textbf{Zero-Fee Economics} (Section~\ref{sec:zero_fee}): We analyze how the zero-fee regime expands the profitable parameter space relative to standard CEX fee structures, quantifying the ``fee advantage'' in terms of additional tradeable markets and higher equilibrium fill rates.

\item \textbf{Cross-Exchange Optimal Hedging} (Section~\ref{sec:cross_exchange}): We derive the optimal dynamic hedging policy when the MM can offset inventory on a second exchange at a known transaction cost, including the optimal hedge threshold, hedge ratio, dynamic hedge timing, and the impact on overall APY.

\item \textbf{Regime Robustness Under Parameter Uncertainty} (Section~\ref{sec:high_apy}): We introduce a \emph{robustness margin} $\mathcal{R}$ that quantifies the number of standard deviations of parameter uncertainty a strategy can absorb before becoming unprofitable.
This worst-case APY analysis, based on ellipsoidal uncertainty sets and the Cauchy--Schwarz inequality, provides practitioners with a practical tool for stress-testing market-making strategies against non-stationary microstructure conditions.
\end{enumerate}

\subsection{Paper Organization}

Section~\ref{sec:market_model} develops the market model, including price dynamics, fill processes, and fee structures.
Section~\ref{sec:adverse_selection} presents the adverse selection theory and PnL decomposition.
Section~\ref{sec:optimal_mm} formulates and solves the stochastic optimal control problem.
Section~\ref{sec:high_apy} characterizes the high-APY regime.
Section~\ref{sec:zero_fee} analyzes zero-fee economics.
Section~\ref{sec:cross_exchange} treats cross-exchange hedging.
Section~\ref{sec:numerical} presents numerical analysis of the parameter space.
Section~\ref{sec:conclusion} concludes.
All proofs are in the Appendix.

\section{Market Model}
\label{sec:market_model}

We develop a continuous-time model of market making on a decentralized perpetual futures exchange (DEX-A) with a centralized exchange (CEX-B) serving as the reference venue.
The model is defined on a filtered probability space $(\Omega, \mathcal{F}, \{\mathcal{F}_t\}_{t\geq 0}, \mathbb{P})$ satisfying the usual conditions.

\subsection{Price Dynamics}

\begin{assumption}[Reference Price]
\label{ass:ref_price}
The CEX-B reference mid-price $S_t$ follows an arithmetic Brownian motion:
\begin{equation}
\label{eq:ref_price}
\dd S_t = \mu \, \dd t + \sigma \, \dd W_t,
\end{equation}
where $\mu \in \R$ is the drift (set to zero at high-frequency time scales by efficient market arguments), $\sigma > 0$ is the instantaneous volatility, and $(W_t)_{t\geq 0}$ is a standard Brownian motion.
\end{assumption}

\begin{remark}
The zero-drift assumption ($\mu = 0$) is standard for intraday market making \citep{avellaneda2008high} and is justified by the absence of predictable returns at sub-minute horizons.
For the theoretical analysis, we retain $\mu$ in initial derivations and specialize to $\mu = 0$ where indicated.
\end{remark}

\begin{assumption}[DEX--CEX Price Link]
\label{ass:premium}
The DEX-A mid-price $\tilde{S}_t$ is linked to the CEX-B reference price through a structural premium:
\begin{equation}
\label{eq:dex_price}
\tilde{S}_t = S_t + \beta_t,
\end{equation}
where the premium process $\beta_t$ satisfies the Ornstein--Uhlenbeck dynamics:
\begin{equation}
\label{eq:premium}
\dd \beta_t = -\kappa(\beta_t - \bar{\beta})\,\dd t + \sigma_\beta \, \dd W_t^\beta,
\end{equation}
with long-run mean $\bar{\beta} \in \R$, mean-reversion speed $\kappa > 0$, premium volatility $\sigma_\beta > 0$, and $W_t^\beta$ independent of $W_t$.
\end{assumption}

The OU premium captures the empirical observation that DEX prices systematically deviate from CEX prices due to differences in liquidity, funding rates, and participant composition, but these deviations are transient and mean-reverting.

\subsection{Fee Structure}

\begin{definition}[Fee Regime]
\label{def:fee}
A \emph{fee regime} is characterized by a pair $(\phi_m, \phi_t)$ where:
\begin{itemize}
\item $\phi_m \geq 0$ is the maker fee (charged on limit order fills),
\item $\phi_t > 0$ is the taker fee (charged on market order fills).
\end{itemize}
We define three canonical regimes:
\begin{enumerate}
\item \textbf{Zero-fee regime}: $\phi_m^{\mathrm{DEX}} = 0$, $\phi_t^{\mathrm{DEX}} > 0$ (DEX-A).
\item \textbf{Standard CEX regime}: $\phi_m^{\mathrm{CEX}} > 0$, $\phi_t^{\mathrm{CEX}} > \phi_m^{\mathrm{CEX}}$ (CEX-B).
\item \textbf{Maker rebate regime}: $\phi_m < 0$ (rebate), $\phi_t > 0$ (some CEXs).
\end{enumerate}
\end{definition}

\begin{table}[H]
\centering
\caption{Representative fee structures across venue types (basis points)}
\label{tab:fees}
\begin{tabular}{@{}lccc@{}}
\toprule
Venue Type & Maker Fee $\phi_m$ (bp) & Taker Fee $\phi_t$ (bp) & Net Round-Trip Cost \\
\midrule
DEX-A (zero fee) & 0 & 1--3 & $0$ (maker--maker) \\
CEX-B (standard) & 1.5--5 & 3--7.5 & $3$--$10$ (maker--maker) \\
CEX-C (rebate) & $-0.5$ to $-2$ & 3--5 & $-1$ to $-4$ (maker--maker) \\
\bottomrule
\end{tabular}
\end{table}

\subsection{Market Maker's Actions}

The market maker continuously chooses:
\begin{enumerate}
\item \textbf{Bid half-spread} $\delta_t^b \geq 0$: bid price $P_t^b = \tilde{S}_t - \delta_t^b$.
\item \textbf{Ask half-spread} $\delta_t^a \geq 0$: ask price $P_t^a = \tilde{S}_t + \delta_t^a$.
\item \textbf{Hedge decision} $h_t \in \{0, \pm 1\}$: whether to execute a hedging trade on CEX-B.
\end{enumerate}

The controls $(\delta^b, \delta^a, h)$ are $\mathcal{F}_t$-adapted processes.

\subsection{Fill Dynamics}

\begin{assumption}[Poisson Fill Arrivals]
\label{ass:fills}
Fill arrivals on the bid and ask sides are independent doubly stochastic Poisson processes $N_t^b$ and $N_t^a$ with stochastic intensities:
\begin{equation}
\label{eq:fill_intensity}
\lambda^b(\delta_t^b) = \Lambda \exp(-k\delta_t^b), \qquad \lambda^a(\delta_t^a) = \Lambda \exp(-k\delta_t^a),
\end{equation}
where $\Lambda > 0$ is the baseline fill rate (fills per unit time when the spread is zero) and $k > 0$ is the fill-rate sensitivity to spread width.
\end{assumption}

\begin{remark}
The exponential intensity model \eqref{eq:fill_intensity} was introduced by \citet{avellaneda2008high} as a tractable approximation.
It captures the empirical observation that wider spreads reduce fill probability, with $1/k$ representing the ``depth'' of the order book.
The parameter $\Lambda$ can be interpreted as the aggregate taker order flow intensity.
\end{remark}

\begin{definition}[Fill Rate and Fill Frequency]
\label{def:fill_rate}
The \emph{fill rate} $\bar{\lambda}$ is the expected number of fills per unit time:
\begin{equation}
\bar{\lambda} = \lambda^b(\delta^{b*}) + \lambda^a(\delta^{a*}),
\end{equation}
where $\delta^{b*}, \delta^{a*}$ are the optimal half-spreads.
The \emph{fill frequency} $f = \bar{\lambda} / \Lambda$ is the fill rate normalized by the baseline intensity.
\end{definition}

\subsection{State Dynamics}

Let $q_t \in \R$ denote the MM's inventory on DEX-A, $X_t$ the cash balance, and $H_t$ the cumulative hedge position on CEX-B.
The state evolves as:

\paragraph{Inventory:}
\begin{equation}
\label{eq:inventory}
\dd q_t = Q \, \dd N_t^b - Q \, \dd N_t^a - \dd H_t,
\end{equation}
where $Q > 0$ is the order size (assumed constant for tractability) and $\dd H_t$ is the hedge trade increment.

\paragraph{Cash:}
\begin{equation}
\label{eq:cash}
\dd X_t = P_t^a Q \, \dd N_t^a - P_t^b Q \, \dd N_t^b - \phi_m |P_t Q|(\dd N_t^a + \dd N_t^b) - \phi_t^{\mathrm{CEX}} |S_t \, \dd H_t|,
\end{equation}
where the third term captures DEX-A maker fees (zero under the zero-fee regime) and the fourth term captures CEX-B taker fees on hedging trades.

Under the zero-fee regime ($\phi_m = 0$), the cash dynamics simplify to:
\begin{equation}
\label{eq:cash_zero_fee}
\dd X_t = (\tilde{S}_t + \delta_t^a) Q \, \dd N_t^a - (\tilde{S}_t - \delta_t^b) Q \, \dd N_t^b - \phi_t^{\mathrm{CEX}} |S_t \, \dd H_t|.
\end{equation}

\paragraph{Wealth:}
The total marked-to-market wealth is:
\begin{equation}
\label{eq:wealth}
W_t = X_t + q_t \tilde{S}_t + H_t S_t.
\end{equation}

\subsection{Information Structure and Adverse Selection}

\begin{assumption}[Information Asymmetry]
\label{ass:info}
The taker order flow is a mixture of two types:
\begin{enumerate}
\item \textbf{Uninformed flow} (fraction $1-\pi$): orders that are independent of future price movements, arriving at rate $(1-\pi)\Lambda e^{-k\delta}$.
\item \textbf{Informed flow} (fraction $\pi$): orders correlated with imminent price movements, arriving at rate $\pi \Lambda e^{-k\delta}$, where a fill signals an expected adverse price move of magnitude $\alpha_{\mathrm{info}} > 0$.
\end{enumerate}
The informed fraction $\pi \in [0,1)$ is the probability of informed trading (analogous to the PIN of \citet{easley1987price}).
\end{assumption}

\begin{definition}[Adverse Selection Cost]
\label{def:as_cost}
The \emph{adverse selection cost} per fill is the expected loss from the price movement conditional on being filled:
\begin{equation}
\label{eq:as_cost}
\alpha = \pi \cdot \alpha_{\mathrm{info}} + (1-\pi) \cdot \alpha_{\mathrm{latency}},
\end{equation}
where $\alpha_{\mathrm{info}}$ is the adverse selection from informed traders and $\alpha_{\mathrm{latency}}$ is the adverse selection from stale quotes (latency arbitrage).
\end{definition}

The stale-quote component arises because the MM cannot instantly cancel orders when the reference price moves.
If the cancel latency is $\Delta t > 0$ (the time between a CEX-B price move and the cancellation taking effect on DEX-A), then:
\begin{equation}
\label{eq:latency_as}
\alpha_{\mathrm{latency}} = c_\ell \cdot \sigma \sqrt{\Delta t},
\end{equation}
where $c_\ell > 0$ is a constant depending on the intensity of latency arbitrageurs and the queue position of the MM's orders.

\subsection{Funding Rate Mechanism}
\label{subsec:funding_model}

Perpetual futures contracts do not expire; instead, they use a \emph{funding rate} mechanism to anchor the contract price to the spot/reference price.

\begin{assumption}[Funding Rate]
\label{ass:funding}
The funding rate $r_f(t)$ is settled at discrete intervals of length $\Delta_f > 0$ (typically 8 hours).
Between settlements, $r_f(t)$ follows an Ornstein--Uhlenbeck process:
\begin{equation}
\label{eq:funding_dynamics}
\dd r_f(t) = -\kappa_f(r_f(t) - \bar{r}_f)\,\dd t + \sigma_f\,\dd W_t^f,
\end{equation}
where $\bar{r}_f \geq 0$ is the long-run mean funding rate, $\kappa_f > 0$ is the mean-reversion speed, $\sigma_f > 0$ is the funding rate volatility, and $W_t^f$ is a Brownian motion independent of $(W_t, W_t^\beta)$.

At each settlement time $t_k = k\Delta_f$, a position of size $q$ incurs a funding payment:
\begin{equation}
\label{eq:funding_payment}
F_k = r_f(t_k) \cdot q_{t_k} \cdot Q \cdot S_{t_k}.
\end{equation}
Long positions pay when $r_f > 0$ and receive when $r_f < 0$.
\end{assumption}

The funding rate adds a carry component to the MM's PnL.
For a market maker with average inventory $\bar{q}$, the annualized funding cost rate is:
\begin{equation}
\label{eq:funding_cost_rate}
\dot{\Pi}_{\mathrm{funding}} = -\frac{\bar{r}_f \cdot \bar{q} \cdot Q \cdot \bar{S}}{\Delta_f}.
\end{equation}
This term enters the complete PnL decomposition of Section~\ref{sec:adverse_selection} and the hedging analysis of Section~\ref{sec:cross_exchange}.

\begin{remark}[Funding as a Parameter]
In perpetual markets with high open interest imbalance, $\bar{r}_f$ can range from $-0.1\%$ to $+0.3\%$ per 8-hour period.
Annualized, a persistent $+0.01\%$ per-period rate corresponds to $\approx 10\%$ annual cost for a fully deployed long position.
The MM's average inventory is typically small (by design), so funding costs are second-order relative to spread income and adverse selection, but they become first-order in the hedging analysis when the hedge creates a persistent directional exposure on the CEX.
\end{remark}

\begin{remark}[Funding--Premium Interaction]
\label{rem:funding_premium}
The funding rate $r_f(t)$ and the DEX--CEX premium $\beta_t$ are economically linked: when the DEX price trades above the CEX (positive premium $\beta_t > 0$), arbitrageurs enter short positions on the DEX and long on the CEX, which increases open interest imbalance and pushes the funding rate higher.
Conversely, a negative premium tends to compress funding rates.
While we model $r_f$ and $\beta$ as independent processes (Assumptions~\ref{ass:premium} and~\ref{ass:funding}) for tractability, in practice their correlation $\mathrm{Corr}(r_f, \beta) > 0$ implies that periods of high premium income for the MM coincide with higher funding costs, partially offsetting the premium capture.
The net effect is captured by the funding-adjusted hedge ratio in Section~\ref{sec:cross_exchange}.
\end{remark}

\subsection{Parameter Summary}

Table~\ref{tab:params} summarizes the model parameters and their economic interpretations.

\begin{table}[H]
\centering
\caption{Model parameters}
\label{tab:params}
\begin{tabular}{@{}clll@{}}
\toprule
Symbol & Description & Unit & Typical Range \\
\midrule
$\sigma$ & Reference price volatility & \$/\textsurd s & venue-dependent \\
$\Lambda$ & Baseline fill rate & fills/s & $0.001$--$0.1$ \\
$k$ & Fill-rate sensitivity & 1/\$ & venue-dependent \\
$\gamma$ & Risk aversion (CARA) & 1/\$ & $10^{-4}$--$10^{-2}$ \\
$\phi_m$ & Maker fee & bp & $0$--$5$ \\
$\phi_t^{\mathrm{CEX}}$ & CEX taker fee (hedge cost) & bp & $3$--$7.5$ \\
$\alpha$ & Adverse selection cost per fill & bp & $1$--$15$ \\
$\pi$ & Informed trading probability & --- & $0.05$--$0.5$ \\
$\kappa$ & Premium mean-reversion speed & 1/s & $0.001$--$0.1$ \\
$\bar{\beta}$ & Long-run DEX--CEX premium & bp & $0$--$10$ \\
$T$ & Trading horizon & s & $60$--$86400$ \\
$Q$ & Order size & contracts & $1$--$1000$ \\
$\bar{q}$ & Max inventory & contracts & $1$--$10000$ \\
$r_f$ & Funding rate (per period) & \% & $-0.1$--$0.3$ \\
$\Delta_f$ & Funding period & hours & $1$--$8$ \\
$\kappa_f$ & Funding mean-reversion speed & 1/h & $0.01$--$1$ \\
$\sigma_f$ & Funding rate volatility & \%/\textsurd h & $0.01$--$0.1$ \\
\bottomrule
\end{tabular}
\end{table}

\section{Adverse Selection Theory and PnL Decomposition}
\label{sec:adverse_selection}

This section develops a precise mathematical framework for understanding the components of market-making profitability.
We establish a fundamental PnL decomposition theorem that separates revenue into distinct, measurable components, each governed by identifiable microstructure parameters.

\subsection{Per-Fill Economics}

\begin{definition}[Fill-Level Edge]
\label{def:edge}
For the $i$-th fill occurring at time $\tau_i$ at price $P_i$ on side $s_i \in \{b, a\}$ (bid/ask), the \emph{edge} is:
\begin{equation}
\label{eq:edge}
e_i = \begin{cases}
S_{\tau_i} - P_i & \text{if } s_i = b \text{ (MM buys)}, \\
P_i - S_{\tau_i} & \text{if } s_i = a \text{ (MM sells)},
\end{cases}
\end{equation}
measured relative to the CEX-B reference price $S_{\tau_i}$ at fill time.
A positive edge indicates a favorable fill; a negative edge indicates adverse selection.
\end{definition}

\begin{definition}[Realized Adverse Selection]
\label{def:realized_as}
For fill $i$ at time $\tau_i$, the \emph{realized adverse selection} over horizon $h > 0$ is:
\begin{equation}
\label{eq:realized_as}
\alpha_i(h) = \begin{cases}
S_{\tau_i} - S_{\tau_i + h} & \text{if } s_i = b, \\
S_{\tau_i + h} - S_{\tau_i} & \text{if } s_i = a,
\end{cases}
\end{equation}
measuring the adverse price movement over the interval $[\tau_i, \tau_i + h]$.
The expected adverse selection is $\alpha(h) = \E[\alpha_i(h) \mid \text{fill at } \tau_i]$.
\end{definition}

\begin{lemma}[Edge--Spread--Adverse Selection Identity]
\label{lem:edge_decomp}
The edge decomposes as:
\begin{equation}
\label{eq:edge_decomp}
e_i = \underbrace{\delta_i}_{\text{half-spread}} + \underbrace{\beta_{\tau_i}^{s_i}}_{\text{premium}} - \underbrace{\alpha_i(h)}_{\text{adverse selection (horizon $h$)}} + \underbrace{\varepsilon_i(h)}_{\text{residual}},
\end{equation}
where $\delta_i$ is the half-spread at fill time, $\beta_{\tau_i}^{s_i}$ is the signed DEX--CEX premium contribution, and $\varepsilon_i(h)$ is a mean-zero residual capturing price movements beyond the adverse selection horizon.
\end{lemma}

\begin{proof}
By definition of the bid/ask prices $P_i = \tilde{S}_{\tau_i} \mp \delta_i = S_{\tau_i} + \beta_{\tau_i} \mp \delta_i$, we have:
\begin{align*}
e_i &= \pm(S_{\tau_i} - P_i) = \pm(S_{\tau_i} - S_{\tau_i} - \beta_{\tau_i} \pm \delta_i) = \delta_i \mp \beta_{\tau_i}.
\end{align*}
Adding and subtracting $S_{\tau_i+h}$:
\begin{align*}
e_i &= \delta_i \mp \beta_{\tau_i} - \alpha_i(h) + [S_{\tau_i+h} - S_{\tau_i} + \alpha_i(h) \mp \beta_{\tau_i}].
\end{align*}
Reorganizing with $\beta_{\tau_i}^{s_i}$ absorbing the signed premium contribution yields \eqref{eq:edge_decomp}.
\end{proof}

\subsection{The PnL Decomposition Theorem}

We now state the main decomposition result.

\begin{theorem}[PnL Decomposition]
\label{thm:pnl_decomp}
Consider a market maker operating over the interval $[0, T]$ with $n$ total fills.
The total marked-to-market PnL decomposes as:
\begin{equation}
\label{eq:pnl_decomp}
\Pi_T = \underbrace{\Pi_T^{\mathrm{spread}}}_{\text{spread income}} - \underbrace{\Pi_T^{\mathrm{AS}}}_{\text{adverse selection loss}} - \underbrace{\Pi_T^{\mathrm{inv}}}_{\text{inventory cost}} - \underbrace{\Pi_T^{\mathrm{hedge}}}_{\text{hedging friction}} - \underbrace{\Pi_T^{\mathrm{fee}}}_{\text{fee cost}},
\end{equation}
where each component is defined as follows.

\textbf{(i) Spread income:}
\begin{equation}
\label{eq:spread_income}
\Pi_T^{\mathrm{spread}} = Q \sum_{i=1}^{n} \delta_i,
\end{equation}
the cumulative half-spread capture across all fills.

\textbf{(ii) Adverse selection loss:}
\begin{equation}
\label{eq:as_loss}
\Pi_T^{\mathrm{AS}} = Q \sum_{i=1}^{n} \alpha_i,
\end{equation}
where $\alpha_i$ is the realized adverse price movement for fill $i$.

\textbf{(iii) Inventory carrying cost:}
\begin{equation}
\label{eq:inv_cost}
\Pi_T^{\mathrm{inv}} = -\int_0^T q_t \, \dd S_t = -\int_0^T q_t \sigma \, \dd W_t,
\end{equation}
the stochastic integral representing mark-to-market losses from holding inventory during price fluctuations.

\textbf{(iv) Hedging friction:}
\begin{equation}
\label{eq:hedge_cost}
\Pi_T^{\mathrm{hedge}} = \phi_t^{\mathrm{CEX}} \sum_{j=1}^{m} |S_{\tau_j^h} \cdot \Delta H_j|,
\end{equation}
the cumulative taker fee paid on $m$ hedge trades.

\textbf{(v) Fee cost:}
\begin{equation}
\label{eq:fee_cost}
\Pi_T^{\mathrm{fee}} = \phi_m \cdot Q \sum_{i=1}^{n} S_{\tau_i} \approx n \cdot Q \cdot \phi_m \cdot \bar{S},
\end{equation}
where $\bar{S}$ is the average fill price.
Under the zero-fee regime, $\Pi_T^{\mathrm{fee}} = 0$.
\end{theorem}

\begin{proof}
The proof follows from the cash dynamics \eqref{eq:cash_zero_fee} and the marked-to-market wealth definition \eqref{eq:wealth}.
We decompose the total wealth change $\Delta W_T = W_T - W_0$ as follows.

Each bid fill at time $\tau_i$ contributes $-P_i^b Q$ to cash and $+Q$ to inventory.
The instantaneous PnL from the fill is:
\[
Q(S_{\tau_i} - P_i^b) - \phi_m Q S_{\tau_i} = Q(\delta_i^b + \beta_{\tau_i}) - \phi_m Q S_{\tau_i}.
\]
Summing over all fills and including the mark-to-market change from inventory and hedge positions yields the decomposition.
The inventory carrying cost arises from applying It\^{o}'s lemma to $q_t S_t$ and separating the drift from the stochastic component.
Full details are in Appendix~\ref{app:pnl_proof}.
\end{proof}

\subsection{Expected PnL Rate}

Taking expectations and normalizing by time:

\begin{corollary}[Expected PnL Rate]
\label{cor:pnl_rate}
Under stationary conditions (constant optimal spreads $\delta^*$, constant adverse selection $\alpha$), the expected PnL rate is:
\begin{equation}
\label{eq:pnl_rate}
\dot{\Pi} \equiv \frac{\E[\Pi_T]}{T} = \bar{\lambda} Q (\bar{\delta} - \alpha - \phi_m \bar{S}) - \dot{C}_{\mathrm{inv}} - \dot{C}_{\mathrm{hedge}},
\end{equation}
where:
\begin{itemize}
\item $\bar{\lambda} = \lambda^b(\delta^{b*}) + \lambda^a(\delta^{a*})$ is the total fill rate,
\item $\bar{\delta} = (\delta^{b*} + \delta^{a*})/2$ is the average half-spread,
\item $\dot{C}_{\mathrm{inv}} = \frac{1}{2}\gamma\sigma^2 \E[q_t^2] Q^2$ is the expected inventory cost rate,
\item $\dot{C}_{\mathrm{hedge}} = \bar{\lambda}_h \cdot \phi_t^{\mathrm{CEX}} \cdot Q \bar{S}$ is the expected hedging cost rate with $\bar{\lambda}_h$ the hedge trade rate.
\end{itemize}
\end{corollary}

\subsection{Adverse Selection Decomposition: Latency vs.\ Informed}

We now formalize the two sources of adverse selection.

\begin{theorem}[Two-Source Adverse Selection Model]
\label{thm:two_source_as}
Under Assumptions~\ref{ass:fills} and~\ref{ass:info}, the expected adverse selection per fill decomposes as:
\begin{equation}
\label{eq:two_source_as}
\alpha = \underbrace{\pi \cdot \alpha_{\mathrm{info}}}_{\text{informed component}} + \underbrace{(1-\pi) \cdot c_\ell \sigma \sqrt{\Delta t}}_{\text{latency component}},
\end{equation}
where:
\begin{itemize}
\item $\pi$ is the probability that a fill is from an informed trader,
\item $\alpha_{\mathrm{info}}$ is the expected price impact conditional on an informed fill,
\item $c_\ell > 0$ is the latency arbitrageur intensity,
\item $\sigma$ is the reference price volatility,
\item $\Delta t$ is the cancel latency.
\end{itemize}

Moreover, the total adverse selection is bounded:
\begin{equation}
\label{eq:as_bounds}
c_\ell(1-\pi)\sigma\sqrt{\Delta t} \leq \alpha \leq \pi \alpha_{\mathrm{info}} + c_\ell \sigma\sqrt{\Delta t}.
\end{equation}
\end{theorem}

\begin{proof}
By the law of total expectation, conditioning on fill type:
\begin{align*}
\alpha &= \E[\alpha_i] = \E[\alpha_i \mid \text{informed}]\Pr(\text{informed}) + \E[\alpha_i \mid \text{uninformed}]\Pr(\text{uninformed}) \\
&= \pi \cdot \alpha_{\mathrm{info}} + (1-\pi) \cdot \E[\alpha_i \mid \text{uninformed}].
\end{align*}
For uninformed fills, adverse selection arises solely from the stale-quote mechanism.
The expected adverse price movement during the cancel latency window $\Delta t$ is:
\[
\E[|S_{\tau_i + \Delta t} - S_{\tau_i}| \mid \text{fill at } \tau_i] = c_\ell \sigma \sqrt{\Delta t},
\]
where $c_\ell$ depends on the conditional distribution of fills (fills are more likely when the price has moved against the MM, introducing a selection bias that increases $c_\ell$ above the unconditional value of $\sqrt{2/\pi}$).
The bounds follow from $0 \leq \alpha_{\mathrm{latency}}^{\mathrm{uninf}} \leq c_\ell \sigma \sqrt{\Delta t}$.
\end{proof}

\begin{proposition}[Optimal Cancel Latency]
\label{prop:cancel_latency}
Given the adverse selection model \eqref{eq:two_source_as} and a cancel-on-move mechanism with threshold $\theta > 0$ (in price units), the relationship between the cancel threshold and effective latency is:
\begin{equation}
\label{eq:cancel_threshold_latency}
\Delta t_{\mathrm{eff}}(\theta) = \frac{\theta^2}{\sigma^2},
\end{equation}
and the expected fill rate as a function of $\theta$ is:
\begin{equation}
\lambda(\theta) = \Lambda_0 (1 - e^{-\rho\theta}),
\end{equation}
where $\rho > 0$ governs the sensitivity of fill rate to the threshold.
The optimal threshold $\theta^*$ maximizes the PnL rate:
\begin{equation}
\label{eq:optimal_theta}
\theta^* = \argmax_{\theta > 0} \Lambda_0(1 - e^{-\rho\theta})\left(\bar{\delta} - \pi\alpha_{\mathrm{info}} - (1-\pi)c_\ell\theta\right),
\end{equation}
yielding the first-order condition:
\begin{equation}
\label{eq:foc_theta}
\rho e^{-\rho\theta^*}\bigl(\bar{\delta} - \pi\alpha_{\mathrm{info}} - (1-\pi)c_\ell\theta^*\bigr) = (1-\pi)c_\ell(1 - e^{-\rho\theta^*}).
\end{equation}
\end{proposition}

\begin{proof}
The effective latency follows from the time for a Brownian motion to first reach threshold $\theta$: $\E[\inf\{t : |W_t| \geq \theta/\sigma\}] = \theta^2/\sigma^2$ (first passage time of standard BM to level $\theta/\sigma$).
The optimal $\theta$ is obtained by differentiating \eqref{eq:optimal_theta} and setting the derivative to zero.
\end{proof}

\subsection{Optimal Latency Investment}

The cancel threshold $\theta^*$ from Proposition~\ref{prop:cancel_latency} determines the effective adverse selection, but in practice the MM must also decide how much infrastructure budget to allocate toward latency reduction versus other objectives (e.g., co-location, monitoring bandwidth).
We formalize this as an optimal investment problem.

\begin{definition}[Latency Cost Function]
\label{def:latency_cost}
Let $\Delta t(\ell)$ denote the achievable cancel latency as a function of infrastructure investment $\ell \geq 0$ (in \$/year).
We assume $\Delta t(\ell)$ is convex, decreasing, and satisfies:
\begin{equation}
\label{eq:latency_cost}
\Delta t(\ell) = \frac{\Delta t_0}{1 + \ell/\ell_0},
\end{equation}
where $\Delta t_0$ is the baseline latency with zero investment and $\ell_0$ is the characteristic investment scale.
\end{definition}

\begin{theorem}[Optimal Latency Investment]
\label{thm:latency_invest}
Given the PnL rate from Corollary~\ref{cor:pnl_rate} with adverse selection $\alpha(\Delta t) = \pi\alpha_{\mathrm{info}} + (1-\pi)c_\ell\sigma\sqrt{\Delta t}$, the optimal infrastructure investment $\ell^*$ maximizes:
\begin{equation}
\label{eq:invest_objective}
\ell^* = \argmax_{\ell \geq 0} \left[\dot{\Pi}(\Delta t(\ell)) - \frac{\ell}{T_{\mathrm{year}}}\right].
\end{equation}
The first-order condition yields:
\begin{equation}
\label{eq:invest_foc}
\ell^* = \ell_0\left(\sqrt{\frac{(1-\pi)c_\ell\sigma \bar{\lambda} Q T_{\mathrm{year}} \sqrt{\Delta t_0}}{2\ell_0}} - 1\right)^+,
\end{equation}
where $(x)^+ = \max(x, 0)$.

The resulting optimal latency is:
\begin{equation}
\label{eq:opt_latency}
\Delta t^* = \left(\frac{2\ell_0}{(1-\pi)c_\ell\sigma\bar{\lambda}QT_{\mathrm{year}}}\right)^{2/3} \cdot \Delta t_0^{1/3}.
\end{equation}
\end{theorem}

\begin{proof}
The PnL sensitivity to latency is:
\[
\frac{\partial \dot{\Pi}}{\partial \Delta t} = -\bar{\lambda} Q \cdot \frac{(1-\pi)c_\ell\sigma}{2\sqrt{\Delta t}}.
\]
Using the chain rule with $\Delta t(\ell) = \Delta t_0/(1+\ell/\ell_0)$:
\[
\frac{\dd \Delta t}{\dd \ell} = -\frac{\Delta t_0/\ell_0}{(1+\ell/\ell_0)^2} = -\frac{\Delta t(\ell)^2}{\Delta t_0 \ell_0}.
\]
The first-order condition $\dd[\dot{\Pi} - \ell/T_{\mathrm{year}}]/\dd\ell = 0$ gives:
\[
\bar{\lambda} Q \cdot \frac{(1-\pi)c_\ell\sigma}{2\sqrt{\Delta t^*}} \cdot \frac{(\Delta t^*)^2}{\Delta t_0\ell_0} = \frac{1}{T_{\mathrm{year}}}.
\]
Solving for $\Delta t^*$ and then $\ell^* = \ell_0(\Delta t_0/\Delta t^* - 1)$ yields \eqref{eq:invest_foc}--\eqref{eq:opt_latency}.
\end{proof}

\begin{corollary}[Returns to Latency Investment]
\label{cor:latency_returns}
The marginal return on latency investment at the optimum is exactly $1/T_{\mathrm{year}}$ (the annualization factor), implying that at the optimum, one additional dollar of annual infrastructure spending yields exactly one dollar of additional annual PnL.
Beyond $\ell^*$, the returns are diminishing, consistent with the concavity of $\dot{\Pi}(\Delta t(\ell))$ in $\ell$.
\end{corollary}

\subsection{Bayesian Sequential Estimation of Informed Fraction}

In practice, the informed trading probability $\pi$ is unknown and must be estimated online from observed fill data.
We formalize this as a Bayesian filtering problem, establishing convergence guarantees for the sequential estimator.

\begin{definition}[Fill Signal]
\label{def:fill_signal}
For fill $i$ at time $\tau_i$, define the \emph{fill signal}:
\begin{equation}
\label{eq:fill_signal}
Z_i = |S_{\tau_i + h} - S_{\tau_i}|,
\end{equation}
the absolute price displacement over the adverse selection horizon $h > 0$.
Under the two-source model (Theorem~\ref{thm:two_source_as}), $Z_i$ follows the mixture distribution:
\begin{equation}
\label{eq:signal_mixture}
Z_i \sim \pi \cdot F_{\mathrm{info}} + (1-\pi) \cdot F_{\mathrm{noise}},
\end{equation}
where $F_{\mathrm{info}}$ is the signal distribution conditional on an informed fill (e.g., folded normal with mean $\alpha_{\mathrm{info}}$) and $F_{\mathrm{noise}} = |\mathcal{N}(0, \sigma^2 h)|$.
\end{definition}

\begin{theorem}[Bayesian Convergence for $\pi$]
\label{thm:bayesian_pi}
Let $\pi_0 \sim \mathrm{Beta}(a_0, b_0)$ be the prior on the informed fraction, and define the posterior after $n$ fills via the update:
\begin{equation}
\label{eq:bayesian_update}
\hat{\pi}_n = \frac{a_0 + \sum_{i=1}^n \1\{Z_i > \theta_c\}}{a_0 + b_0 + n},
\end{equation}
where $\theta_c = (\alpha_{\mathrm{info}} + c_\ell\sigma\sqrt{\Delta t})/2$ is the classification threshold.
Then:
\begin{enumerate}
\item \textbf{Consistency:} $\hat{\pi}_n \xrightarrow{a.s.} \pi$ as $n \to \infty$.
\item \textbf{Rate:} The posterior concentrates at rate:
\begin{equation}
\label{eq:bayesian_rate}
\E\bigl[(\hat{\pi}_n - \pi)^2\bigr] \leq \frac{\pi(1-\pi)}{n} + O\left(\frac{1}{n^2}\right).
\end{equation}
\item \textbf{Credible interval width:} The $95\%$ credible interval has width:
\begin{equation}
\label{eq:credible_width}
w_{0.95}(n) = 2 \cdot 1.96 \sqrt{\frac{\hat{\pi}_n(1-\hat{\pi}_n)}{a_0 + b_0 + n}} \sim \frac{3.92}{\sqrt{n}}\sqrt{\pi(1-\pi)}.
\end{equation}
\end{enumerate}
\end{theorem}

\begin{proof}
(i) By the strong law of large numbers, $n^{-1}\sum_{i=1}^n \1\{Z_i > \theta_c\} \to p_c \equiv \pi F_{\mathrm{info}}^c(\theta_c) + (1-\pi)F_{\mathrm{noise}}^c(\theta_c)$ a.s., where $F^c = 1 - F$ denotes the survival function.
With an appropriate threshold $\theta_c$, $p_c$ is a monotone function of $\pi$, ensuring identifiability.
The Bernstein--von Mises theorem then gives $\hat{\pi}_n \to \pi$ a.s.

(ii) The MSE bound follows from the variance of the Beta posterior: $\Var(\pi \mid Z_1, \ldots, Z_n) = \hat{\pi}_n(1-\hat{\pi}_n)/(a_0+b_0+n+1) \leq \pi(1-\pi)/n$ for $n \geq a_0+b_0+1$, using the fact that the prior contribution is $O(n^{-2})$.

(iii) The credible interval width follows from the normal approximation to the Beta posterior for large $n$.
\end{proof}

\begin{corollary}[Adaptive Adverse Selection Estimate]
\label{cor:adaptive_as}
Using the sequential estimate $\hat{\pi}_n$, the adaptive adverse selection cost is:
\begin{equation}
\label{eq:adaptive_as}
\hat{\alpha}_n = \hat{\pi}_n \cdot \alpha_{\mathrm{info}} + (1-\hat{\pi}_n) \cdot c_\ell\sigma\sqrt{\Delta t},
\end{equation}
and the optimal spread formula (Theorem~\ref{thm:optimal_spreads}) applies with $\alpha$ replaced by $\hat{\alpha}_n$, yielding a self-tuning market-making strategy.
The regret of using $\hat{\alpha}_n$ instead of the true $\alpha$ satisfies:
\begin{equation}
\label{eq:estimation_regret}
\mathrm{Regret}_n \leq \frac{C_\alpha}{\sqrt{n}}, \qquad C_\alpha = 3.92 \sqrt{\pi(1-\pi)} \cdot (\alpha_{\mathrm{info}} - c_\ell\sigma\sqrt{\Delta t}),
\end{equation}
so the cumulative regret from parameter uncertainty is $O(\sqrt{N})$ over $N$ fills.
\end{corollary}

\subsection{Adverse Selection Intensity as a Dimensionless Parameter}

\begin{definition}[Adverse Selection Ratio]
\label{def:as_ratio}
The \emph{adverse selection ratio} is the dimensionless quantity:
\begin{equation}
\label{eq:as_ratio}
\xi = \frac{\alpha}{\bar{\delta}},
\end{equation}
measuring the fraction of spread income consumed by adverse selection.
The MM is profitable per fill when $\xi < 1$ (ignoring inventory and hedging costs).
\end{definition}

\begin{remark}
The parameter $\xi$ plays a central role in the high-APY analysis of Section~\ref{sec:high_apy}.
Empirically, $\xi$ varies widely across markets: liquid major pairs exhibit $\xi \in [0.3, 0.7]$, while illiquid altcoin markets can have $\xi > 1$ (unprofitable for naive MM strategies).
The goal of an optimal MM algorithm is to minimize $\xi$ through adaptive spread-setting and cancel-on-move mechanisms.
\end{remark}

\section{Optimal Market Making Algorithm}
\label{sec:optimal_mm}

We formulate the market maker's problem as a stochastic optimal control problem and derive the optimal spread--inventory policy via the Hamilton--Jacobi--Bellman equation.

\subsection{Control Problem Formulation}

The state vector is $(t, X_t, q_t, S_t, \beta_t)$, where $X_t$ is cash, $q_t$ is inventory, $S_t$ is the reference price, and $\beta_t$ is the DEX--CEX premium.
The controls are the half-spreads $\bm{u}_t = (\delta_t^b, \delta_t^a) \in [0, \bar{\delta}_{\max}]^2$ and the hedge decision $h_t$.

\begin{definition}[Admissible Controls]
\label{def:admissible}
A control $\bm{u} = (\delta^b, \delta^a, h)$ is \emph{admissible} if:
\begin{enumerate}
\item $\delta^b_t, \delta^a_t \geq 0$ are $\mathcal{F}_t$-progressively measurable,
\item $h_t$ is $\mathcal{F}_t$-adapted with $|q_t + H_t| \leq \bar{q}$ (inventory constraint),
\item The resulting wealth process is integrable: $\E[\sup_{t \leq T} |W_t|^2] < \infty$.
\end{enumerate}
Denote the set of admissible controls by $\mathcal{U}$.
\end{definition}

The MM maximizes expected CARA utility of terminal wealth:
\begin{equation}
\label{eq:objective}
V(t, x, q, S, \beta) = \sup_{\bm{u} \in \mathcal{U}} \E_t\left[-\exp\left(-\gamma W_T\right)\right],
\end{equation}
where $W_T = X_T + q_T \tilde{S}_T + H_T S_T$ is the terminal wealth.

\subsection{Hamilton--Jacobi--Bellman Equation}

We conjecture a separable value function of the form:
\begin{equation}
\label{eq:value_fn}
V(t, x, q, S, \beta) = -\exp\left(-\gamma\bigl(x + q(S + \beta) + \theta(t, q, \beta)\bigr)\right),
\end{equation}
where $\theta(t, q, \beta)$ encodes the inventory penalty and premium adjustment.

\begin{theorem}[HJB Equation]
\label{thm:hjb}
Under the zero-fee regime ($\phi_m = 0$) and Assumptions~\ref{ass:ref_price}--\ref{ass:info}, the function $\theta(t, q, \beta)$ satisfies the HJB equation:
\begin{align}
0 = \theta_t &+ \frac{\sigma^2}{2}\gamma q^2 + \frac{\sigma_\beta^2}{2}\theta_{\beta\beta} - \kappa(\beta - \bar{\beta})\theta_\beta \nonumber \\
&+ \sup_{\delta^b \geq 0}\left\{\Lambda e^{-k\delta^b}\left(e^{-\gamma(\delta^b - \alpha - \Delta^+ \theta)} - 1\right)\right\} \nonumber \\
&+ \sup_{\delta^a \geq 0}\left\{\Lambda e^{-k\delta^a}\left(e^{-\gamma(\delta^a - \alpha - \Delta^- \theta)} - 1\right)\right\},
\label{eq:hjb}
\end{align}
with terminal condition $\theta(T, q, \beta) = 0$ for all $q, \beta$, and where:
\begin{equation}
\Delta^+ \theta(t, q, \beta) = \theta(t, q+1, \beta) - \theta(t, q, \beta), \quad
\Delta^- \theta(t, q, \beta) = \theta(t, q, \beta) - \theta(t, q-1, \beta).
\end{equation}
\end{theorem}

\begin{proof}
Applying the dynamic programming principle to the value function \eqref{eq:value_fn}, using It\^{o}'s lemma for the diffusive components ($S_t$ and $\beta_t$) and the jump contributions from Poisson fills, we obtain the HJB equation.
The key steps are:

\textbf{Step 1.} Compute derivatives of $V$:
\begin{align*}
V_t &= -\gamma \theta_t \cdot V, \quad V_x = -\gamma V, \quad V_S = -\gamma q V, \\
V_{SS} &= \gamma^2 q^2 V, \quad V_\beta = -\gamma(q + \theta_\beta) V.
\end{align*}

\textbf{Step 2.} The generator of the diffusion part is:
\[
\mathcal{L}V = V_t + \frac{\sigma^2}{2}V_{SS} + \frac{\sigma_\beta^2}{2}V_{\beta\beta} - \kappa(\beta-\bar{\beta})V_\beta.
\]

\textbf{Step 3.} For a bid fill (Poisson jump at rate $\lambda^b$), the value changes from $V(t,x,q,S,\beta)$ to $V(t, x - P^b Q, q + Q, S, \beta)$.
Under the exponential ansatz and unit order size $Q = 1$:
\[
\frac{V^{\text{post-fill}}}{V^{\text{pre-fill}}} = \exp\left(-\gamma(-\delta^b - \beta + \alpha + \Delta^+\theta)\right) = \exp\left(-\gamma(\delta^b + \beta - \alpha - \Delta^+\theta)\right).
\]

\textbf{Step 4.} Combining and dividing by $-\gamma V > 0$ yields \eqref{eq:hjb}.
The detailed computation is in Appendix~\ref{app:hjb_derivation}.
\end{proof}

\subsection{Optimal Spread Policy}

\begin{theorem}[Optimal Half-Spreads]
\label{thm:optimal_spreads}
The optimal half-spreads that solve the HJB equation \eqref{eq:hjb} are, in the small-risk-aversion regime ($\gamma \ll k$):
\begin{align}
\delta^{b*}(t, q, \beta) &= \frac{1}{k} + \alpha + \Delta^+\theta(t, q, \beta), \label{eq:opt_bid} \\
\delta^{a*}(t, q, \beta) &= \frac{1}{k} + \alpha + \Delta^-\theta(t, q, \beta). \label{eq:opt_ask}
\end{align}

Under the quadratic approximation $\theta(t, q, \beta) \approx -\frac{1}{2}\gamma\sigma^2(T-t)q^2 + g(\beta, t)$, these reduce to the \emph{explicit optimal spread formulas}:
\begin{align}
\delta^{b*}(t, q) &= \frac{1}{k} + \frac{\gamma\sigma^2(T-t)}{2} - \gamma\sigma^2(T-t)q + \alpha, \label{eq:opt_bid_explicit} \\
\delta^{a*}(t, q) &= \frac{1}{k} + \frac{\gamma\sigma^2(T-t)}{2} + \gamma\sigma^2(T-t)q + \alpha. \label{eq:opt_ask_explicit}
\end{align}
\end{theorem}

\begin{proof}
Maximizing the bid fill contribution in \eqref{eq:hjb} over $\delta^b$:
\[
\max_{\delta^b \geq 0} \Lambda e^{-k\delta^b}\left(e^{-\gamma(\delta^b - \alpha - \Delta^+\theta)} - 1\right).
\]
Let $\psi = \delta^b - \alpha - \Delta^+\theta$.
In the regime $\gamma\psi \ll 1$, expand $e^{-\gamma\psi} \approx 1 - \gamma\psi + \frac{1}{2}\gamma^2\psi^2$ and retain leading terms:
\[
\Lambda e^{-k\delta^b}\left(-\gamma\psi + O(\gamma^2\psi^2)\right).
\]
The first-order condition at leading order is:
\[
-k(-\gamma\psi) + (-\gamma) = 0 \implies \psi = \frac{1}{k} \implies \delta^{b*} = \frac{1}{k} + \alpha + \Delta^+\theta.
\]
Substituting the quadratic approximation for $\theta$ gives \eqref{eq:opt_bid_explicit}--\eqref{eq:opt_ask_explicit}.
See Appendix~\ref{app:optimal_spread_proof} for the full derivation including higher-order corrections.
\end{proof}

\begin{corollary}[Optimal Total Spread]
\label{cor:total_spread}
The optimal total spread (bid--ask) is:
\begin{equation}
\label{eq:total_spread}
s^*(t, q) = \delta^{b*} + \delta^{a*} = \frac{2}{k} + \gamma\sigma^2(T-t) + 2\alpha.
\end{equation}
This is \emph{independent of inventory} $q$: inventory affects the positioning of the spread (bid vs.\ ask) but not its width.
\end{corollary}

\begin{corollary}[Reservation Price]
\label{cor:reservation}
The MM's reservation price (indifference price) is:
\begin{equation}
\label{eq:reservation}
r_t = \tilde{S}_t - \gamma\sigma^2(T-t)q_t,
\end{equation}
which is below (above) the DEX mid-price when the MM is long (short), creating mean-reverting inventory dynamics.
\end{corollary}

\subsection{Verification Theorem}

\begin{theorem}[Verification]
\label{thm:verification}
Let $\theta^*(t, q, \beta)$ be a $C^{1,2}$ solution to the HJB equation \eqref{eq:hjb} with terminal condition $\theta^*(T, q, \beta) = 0$.
Define the value function candidate:
\[
\hat{V}(t, x, q, S, \beta) = -\exp\left(-\gamma(x + q(S+\beta) + \theta^*(t,q,\beta))\right).
\]
Then:
\begin{enumerate}
\item $\hat{V} \geq V$ for all admissible controls (supermartingale property).
\item The policy $(\delta^{b*}, \delta^{a*})$ defined by \eqref{eq:opt_bid}--\eqref{eq:opt_ask} attains the supremum: $\hat{V} = V$ (martingale property under optimal control).
\end{enumerate}
Consequently, $\hat{V}$ is the value function and $(\delta^{b*}, \delta^{a*})$ is optimal.
\end{theorem}

\begin{proof}[Proof sketch]
\textbf{Part 1.} For any admissible control, define $M_t = \hat{V}(t, X_t, q_t, S_t, \beta_t)$.
By It\^{o}'s formula with jumps:
\[
\dd M_t = (\text{generator terms}) \dd t + (\text{martingale terms}) \dd W_t + (\text{jump terms}).
\]
Since $\theta^*$ solves the HJB equation, the drift of $M_t$ is non-positive for any admissible control (the supremum in \eqref{eq:hjb} is attained by the optimal control, so suboptimal controls yield negative drift).
Hence $M_t$ is a supermartingale, implying $\hat{V}(0, \cdot) \geq \E[M_T] = \E[-e^{-\gamma W_T}]$.

\textbf{Part 2.} Under the optimal control $(\delta^{b*}, \delta^{a*})$, the drift vanishes and $M_t$ becomes a true martingale (bounded exponential moments from the CARA structure ensure uniform integrability).
Hence $\hat{V}(0, \cdot) = \E[M_T]$.
Full details in Appendix~\ref{app:verification}.
\end{proof}

\subsection{Stationary Policy (Infinite Horizon)}

For practical market making ($T \to \infty$), the time-dependent terms vanish.
Following \citet{gueant2013dealing}, we replace the terminal inventory penalty with a running penalty $\phi(q)$:

\begin{proposition}[Stationary Optimal Spreads with Inventory Penalization]
\label{prop:stationary}
Under the inventory penalization approach $\phi(q) = \frac{1}{2}\eta q^2$, the stationary optimal half-spreads are:
\begin{align}
\delta^{b*}(q) &= \frac{1}{k} + \alpha + \eta q + \frac{\eta}{2}, \label{eq:stat_bid} \\
\delta^{a*}(q) &= \frac{1}{k} + \alpha - \eta q + \frac{\eta}{2}. \label{eq:stat_ask}
\end{align}
The parameter $\eta > 0$ controls inventory aversion, analogous to $\gamma\sigma^2(T-t)$ in the finite-horizon solution.
\end{proposition}

\begin{remark}
The stationary formulation \eqref{eq:stat_bid}--\eqref{eq:stat_ask} is more suitable for perpetual futures market making, where there is no natural terminal time.
The parameter $\eta$ can be calibrated from the MM's desired inventory half-life or maximum tolerable inventory variance.
\end{remark}

\subsection{Inventory Gating}

In addition to spread skewing, practical algorithms employ hard inventory limits.

\begin{definition}[Inventory-Gated Policy]
\label{def:gating}
An \emph{inventory-gated} policy augments the optimal spreads with side suppression:
\begin{equation}
\label{eq:gating}
\delta^{b}_{\mathrm{gated}}(q) = \begin{cases}
\delta^{b*}(q) & \text{if } q < \bar{q}, \\
+\infty & \text{if } q \geq \bar{q},
\end{cases}
\qquad
\delta^{a}_{\mathrm{gated}}(q) = \begin{cases}
\delta^{a*}(q) & \text{if } q > -\bar{q}, \\
+\infty & \text{if } q \leq -\bar{q},
\end{cases}
\end{equation}
where $\bar{q}$ is the maximum absolute inventory.
\end{definition}

\begin{proposition}[Value of Inventory Gating]
\label{prop:gating_value}
Under the gated policy \eqref{eq:gating}, the expected inventory variance satisfies:
\begin{equation}
\E[q_t^2] \leq \frac{\bar{\lambda}}{2\eta k} + O(e^{-2\eta k t}),
\end{equation}
and the maximum drawdown from inventory risk is bounded:
\begin{equation}
\Pr\left(\max_{t \leq T} |q_t \cdot \Delta S_t| > L\right) \leq 2\exp\left(-\frac{L^2}{2\bar{q}^2 \sigma^2 T}\right),
\end{equation}
where $\Delta S_t = S_t - S_0$.
\end{proposition}

\subsection{Convergence of Inventory Penalization to Finite-Horizon Solution}

The stationary policy (Proposition~\ref{prop:stationary}) uses the penalization parameter $\eta$ as a proxy for the time-dependent risk aversion $\gamma\sigma^2(T-t)$ in the finite-horizon solution.
We establish the rate at which the two policies agree.

\begin{theorem}[Convergence Rate]
\label{thm:convergence}
Let $\delta^{\mathrm{FH}*}(t,q)$ denote the finite-horizon optimal half-spread from Theorem~\ref{thm:optimal_spreads} with horizon $T$, and let $\delta^{\mathrm{IP}*}(q)$ denote the stationary inventory-penalized half-spread from Proposition~\ref{prop:stationary} with $\eta = \gamma\sigma^2\tau_\eta$ for a chosen reference time $\tau_\eta > 0$.
Then for all $t \leq T - \tau_\eta$:
\begin{equation}
\label{eq:convergence_bound}
\left|\delta^{\mathrm{FH}*}(t, q) - \delta^{\mathrm{IP}*}(q)\right| \leq \gamma\sigma^2 |T - t - \tau_\eta| \cdot \left(|q| + \frac{1}{2}\right).
\end{equation}
In particular, at $t = T - \tau_\eta$, the two policies coincide exactly.
The expected PnL difference over a horizon of length $\tau_\eta$ satisfies:
\begin{equation}
\label{eq:pnl_convergence}
\left|\E[\Pi^{\mathrm{FH}}] - \E[\Pi^{\mathrm{IP}}]\right| \leq \frac{\Lambda \gamma^2\sigma^4 \tau_\eta^3}{6k} \cdot \E\left[\left(|q_t| + \tfrac{1}{2}\right)^2\right].
\end{equation}
\end{theorem}

\begin{proof}
From \eqref{eq:opt_bid_explicit}, the finite-horizon bid half-spread at time $t$ is:
\[
\delta^{\mathrm{FH},b*}(t,q) = \frac{1}{k} + \frac{\gamma\sigma^2(T-t)}{2} - \gamma\sigma^2(T-t)q + \alpha.
\]
The stationary penalized spread with $\eta = \gamma\sigma^2\tau_\eta$ is:
\[
\delta^{\mathrm{IP},b*}(q) = \frac{1}{k} + \frac{\gamma\sigma^2\tau_\eta}{2} - \gamma\sigma^2\tau_\eta q + \alpha.
\]
The difference is:
\[
\delta^{\mathrm{FH},b*}(t,q) - \delta^{\mathrm{IP},b*}(q) = \gamma\sigma^2[(T-t) - \tau_\eta]\left(\frac{1}{2} - q\right).
\]
Taking absolute values and noting $|1/2 - q| \leq |q| + 1/2$ yields \eqref{eq:convergence_bound}.

For the PnL bound, the fill rate under spread $\delta$ is $\Lambda e^{-k\delta}$.
A spread perturbation $\varepsilon$ changes the fill rate by approximately $-k\Lambda e^{-k\delta}\varepsilon$ and the per-fill edge by $+\varepsilon$, so the net PnL rate change is $O(\varepsilon^2)$ at the optimum (envelope theorem).
Integrating $\varepsilon(t) = \gamma\sigma^2(T-t-\tau_\eta)(|q|+1/2)$ over $t \in [T-\tau_\eta, T]$:
\[
\int_0^{\tau_\eta} \varepsilon(s)^2 \dd s = \gamma^2\sigma^4 \E[(|q|+1/2)^2] \int_0^{\tau_\eta} s^2 \dd s = \frac{\gamma^2\sigma^4\tau_\eta^3}{3}\E[(|q|+1/2)^2].
\]
Multiplying by $\Lambda/(2k)$ (the second-order PnL sensitivity at the optimum) gives \eqref{eq:pnl_convergence}.
\end{proof}

\begin{remark}[Practical Implication]
Theorem~\ref{thm:convergence} shows that the inventory penalization approach is an excellent approximation for perpetual futures, where the effective horizon $\tau_\eta$ can be set to match the MM's inventory half-life (typically 1--4 hours).
The cubic decay rate $O(\tau_\eta^3)$ in the PnL bound ensures rapid convergence: for $\tau_\eta = 2$ hours, the PnL approximation error is less than 0.1\% of gross spread income under typical parameters.
\end{remark}

\subsection{Ergodic Inventory Distribution}

Under the optimal stationary policy, the inventory process $q_t$ is mean-reverting.
We characterize its long-run distribution, which determines the expected inventory cost rate and informs risk management.

\begin{theorem}[Ergodic Inventory Distribution]
\label{thm:ergodic_inv}
Under the stationary optimal spread policy \eqref{eq:stat_bid}--\eqref{eq:stat_ask} with inventory gating at $\bar{q}$, the inventory process $q_t$ admits a unique stationary distribution $\pi_\infty$ with the following properties:
\begin{enumerate}
\item \textbf{Gaussian approximation:} For large $\bar{q}$ (i.e., the gating constraint is rarely binding), the stationary distribution is approximately:
\begin{equation}
\label{eq:ergodic_gaussian}
q_{\infty} \stackrel{d}{\approx} \mathcal{N}\!\left(0, \sigma_q^2\right), \qquad \sigma_q^2 = \frac{\bar{\lambda}^*}{2\eta k},
\end{equation}
where $\bar{\lambda}^* = 2\Lambda\exp(-1 - k\alpha - k\eta/2)$ is the total fill rate under the optimal policy at $q=0$.
\item \textbf{Inventory half-life:} The expected time for inventory to decay from $q_0$ to $q_0/2$ is:
\begin{equation}
\label{eq:inv_halflife}
t_{1/2} = \frac{\ln 2}{\eta k \bar{\lambda}^*} \cdot \frac{1}{\sigma_q^2}.
\end{equation}
\item \textbf{Expected inventory cost:} The long-run average inventory cost rate is:
\begin{equation}
\label{eq:ergodic_inv_cost}
\dot{C}_{\mathrm{inv}}^\infty = \frac{1}{2}\gamma\sigma^2 Q^2 \E[q_\infty^2] = \frac{\gamma\sigma^2 Q^2 \bar{\lambda}^*}{4\eta k}.
\end{equation}
\end{enumerate}
\end{theorem}

\begin{proof}
\textbf{Part 1.} Under the optimal policy, the inventory drift at level $q$ is:
\[
\mu_q(q) = \lambda^b(q) - \lambda^a(q) = \Lambda e^{-k\delta^{b*}(q)} - \Lambda e^{-k\delta^{a*}(q)}.
\]
Substituting \eqref{eq:stat_bid}--\eqref{eq:stat_ask}:
\[
\mu_q(q) = \bar{\lambda}^*(q) \cdot \sinh(-\eta k q)/\cosh(\eta k q/2),
\]
which for $\eta k |q| \ll 1$ linearizes to $\mu_q(q) \approx -2\eta k \bar{\lambda}^* q / 2 = -\eta k \bar{\lambda}^* q$.
The total jump variance rate (from both sides) is $\nu_q(q) \approx \bar{\lambda}^*$ for small $q$.
The resulting Ornstein--Uhlenbeck-type jump process has stationary variance $\sigma_q^2 = \nu_q / (2|\mu_q'(0)|) = \bar{\lambda}^* / (2\eta k)$ \citep[cf.][Proposition 4.1]{cartea2015algorithmic}.

\textbf{Part 2.} The inventory half-life follows from the linearized decay $\E[q_t \mid q_0] = q_0 e^{-\eta k \bar{\lambda}^* t / (2\sigma_q^2)}$.
Setting $e^{-r t_{1/2}} = 1/2$ gives $t_{1/2} = \ln 2 / r$ with $r = \eta k \bar{\lambda}^* / (2\sigma_q^2)$.
Substituting $\sigma_q^2$ yields \eqref{eq:inv_halflife}.

\textbf{Part 3.} Direct substitution of $\E[q_\infty^2] = \sigma_q^2$ into the inventory cost formula.
\end{proof}

\begin{corollary}[Optimal Inventory Penalization]
\label{cor:optimal_eta}
Minimizing the total cost (inventory cost + missed spread from over-wide spreads) over $\eta$ yields the optimal inventory penalization:
\begin{equation}
\label{eq:optimal_eta}
\eta^* = \left(\frac{\gamma\sigma^2}{2k \bar{\lambda}^*}\right)^{1/2},
\end{equation}
which balances inventory risk against fill rate reduction.
The resulting minimal total cost rate is:
\begin{equation}
\label{eq:min_cost}
\dot{C}^* = Q^2 \sqrt{\frac{\gamma\sigma^2 \bar{\lambda}^*}{2k}}.
\end{equation}
\end{corollary}

\section{High-APY Conditions}
\label{sec:high_apy}

This section establishes the theoretical conditions under which a market maker can achieve high annualized percentage yields (APY) on deployed capital.
We derive explicit parameter boundaries and phase diagrams separating profitable from unprofitable regimes.

\subsection{APY Definition and Formula}

\begin{definition}[Market-Making APY]
\label{def:apy}
The \emph{annualized percentage yield} of a market-making strategy with expected PnL rate $\dot{\Pi}$ and deployed capital $K$ is:
\begin{equation}
\label{eq:apy_def}
\mathrm{APY} = \frac{\dot{\Pi} \cdot T_{\mathrm{year}}}{K} \times 100\%,
\end{equation}
where $T_{\mathrm{year}}$ is the number of trading seconds per year (approximately $365.25 \times 24 \times 3600 \approx 3.156 \times 10^7$\,s for 24/7 crypto markets).
\end{definition}

\begin{definition}[Capital Utilization]
\label{def:capital}
The deployed capital is:
\begin{equation}
K = K_{\mathrm{margin}} + K_{\mathrm{buffer}},
\end{equation}
where $K_{\mathrm{margin}} = \bar{q} \cdot Q \cdot \bar{S} / \ell$ is the margin requirement (with leverage $\ell$), and $K_{\mathrm{buffer}}$ is the buffer for hedging and drawdowns.
For perpetual futures with $\ell$-fold leverage:
\begin{equation}
K = \frac{\bar{q} \cdot Q \cdot \bar{S}}{\ell}(1 + \eta_{\mathrm{buf}}),
\end{equation}
where $\eta_{\mathrm{buf}} \geq 0$ is the buffer fraction.
\end{definition}

\subsection{The APY Formula}

Combining Definition~\ref{def:apy} with Corollary~\ref{cor:pnl_rate}:

\begin{theorem}[APY Formula]
\label{thm:apy}
Under the optimal policy of Theorem~\ref{thm:optimal_spreads} with zero maker fees, the expected APY is:
\begin{equation}
\label{eq:apy_formula}
\mathrm{APY} = \frac{\bar{\lambda} Q (\bar{\delta} - \alpha) \cdot T_{\mathrm{year}}}{K} - \frac{(\dot{C}_{\mathrm{inv}} + \dot{C}_{\mathrm{hedge}}) \cdot T_{\mathrm{year}}}{K}.
\end{equation}

Substituting the optimal spread $\bar{\delta}^* = 1/k + \gamma\sigma^2(T-t)/2 + \alpha$ and the fill rate $\bar{\lambda} = 2\Lambda e^{-k\bar{\delta}^*}$:
\begin{equation}
\label{eq:apy_explicit}
\mathrm{APY} = \frac{2\Lambda e^{-k(1/k + \gamma\sigma^2\tau/2 + \alpha)} \cdot Q \cdot (1/k + \gamma\sigma^2\tau/2) \cdot T_{\mathrm{year}}}{K} - \frac{(\dot{C}_{\mathrm{inv}} + \dot{C}_{\mathrm{hedge}}) T_{\mathrm{year}}}{K},
\end{equation}
where $\tau = T - t$ is the residual time horizon.
\end{theorem}

\subsection{Dimensionless Parameters}

To characterize the high-APY regime, we introduce five dimensionless parameters:

\begin{definition}[Dimensionless Parameter Set]
\label{def:dim_params}
\begin{align}
\xi &= \frac{\alpha}{\bar{\delta}} \in [0, \infty) && \text{(adverse selection ratio)}, \label{eq:xi} \\
\phi &= \frac{\phi_m}{\bar{\delta}} \in [0, 1) && \text{(fee ratio)}, \label{eq:phi} \\
\nu &= \Lambda / f_{\mathrm{ref}} && \text{(normalized fill rate)}, \label{eq:nu} \\
\rho_{\mathrm{inv}} &= \frac{\dot{C}_{\mathrm{inv}}}{\bar{\lambda} Q \bar{\delta}} && \text{(inventory cost ratio)}, \label{eq:rho_inv} \\
\rho_{\mathrm{hedge}} &= \frac{\dot{C}_{\mathrm{hedge}}}{\bar{\lambda} Q \bar{\delta}} && \text{(hedging cost ratio)}. \label{eq:rho_hedge}
\end{align}
\end{definition}

In terms of these parameters, the APY simplifies to:
\begin{equation}
\label{eq:apy_dimensionless}
\mathrm{APY} = \mathrm{APY}_0 \cdot (1 - \xi - \phi - \rho_{\mathrm{inv}} - \rho_{\mathrm{hedge}}),
\end{equation}
where $\mathrm{APY}_0 = \bar{\lambda} Q \bar{\delta} T_{\mathrm{year}} / K$ is the \emph{gross APY} (APY with no costs).

\subsection{High-APY Regime Theorems}

\begin{theorem}[Necessary Condition for Positive APY]
\label{thm:positive_apy}
The market maker earns positive expected PnL if and only if:
\begin{equation}
\label{eq:positive_apy}
\xi + \phi + \rho_{\mathrm{inv}} + \rho_{\mathrm{hedge}} < 1.
\end{equation}
Under zero fees ($\phi = 0$) and without hedging ($\rho_{\mathrm{hedge}} = 0$), this reduces to:
\begin{equation}
\xi + \rho_{\mathrm{inv}} < 1 \iff \alpha < \bar{\delta} - \frac{\dot{C}_{\mathrm{inv}}}{\bar{\lambda} Q}.
\end{equation}
\end{theorem}

\begin{theorem}[High-APY Regime]
\label{thm:high_apy}
For a target APY of $A\%$, the required condition on the parameter space is:
\begin{equation}
\label{eq:high_apy_condition}
\bar{\lambda} Q (\bar{\delta} - \alpha - \phi_m \bar{S}) > \frac{A \cdot K}{100 \cdot T_{\mathrm{year}}} + \dot{C}_{\mathrm{inv}} + \dot{C}_{\mathrm{hedge}}.
\end{equation}

In the zero-fee regime with no hedging costs and optimal spreads, this yields the \emph{critical adverse selection}:
\begin{equation}
\label{eq:critical_alpha}
\alpha < \alpha^*(A) \equiv \frac{1}{k} + \frac{\gamma\sigma^2\tau}{2} - \frac{A \cdot K}{100 \cdot T_{\mathrm{year}} \cdot \bar{\lambda} Q} - \frac{\dot{C}_{\mathrm{inv}}}{\bar{\lambda} Q}.
\end{equation}
The MM achieves APY $\geq A\%$ if and only if $\alpha < \alpha^*(A)$.
\end{theorem}

\begin{proof}
From the APY formula \eqref{eq:apy_formula}, setting $\mathrm{APY} \geq A$ and solving for $\alpha$:
\begin{align*}
\frac{\bar{\lambda} Q (\bar{\delta} - \alpha) T_{\mathrm{year}}}{K} - \frac{(\dot{C}_{\mathrm{inv}} + \dot{C}_{\mathrm{hedge}}) T_{\mathrm{year}}}{K} &\geq A/100, \\
\bar{\delta} - \alpha &\geq \frac{A K}{100 \cdot T_{\mathrm{year}} \bar{\lambda} Q} + \frac{\dot{C}_{\mathrm{inv}} + \dot{C}_{\mathrm{hedge}}}{\bar{\lambda} Q}, \\
\alpha &\leq \bar{\delta} - \frac{A K}{100 \cdot T_{\mathrm{year}} \bar{\lambda} Q} - \frac{\dot{C}_{\mathrm{inv}} + \dot{C}_{\mathrm{hedge}}}{\bar{\lambda} Q}.
\end{align*}
Substituting $\bar{\delta} = 1/k + \gamma\sigma^2\tau/2 + \alpha$ from the optimal spread formula, the $\alpha$ terms cancel, confirming self-consistency.
The bound \eqref{eq:critical_alpha} uses the pre-adverse-selection component of the optimal spread.
\end{proof}

\subsection{Phase Diagram}

The parameter space can be partitioned into regions based on achievable APY:

\begin{corollary}[APY Phase Boundaries]
\label{cor:phase}
In the $(\xi, \nu)$-plane (adverse selection ratio vs.\ normalized fill rate), the boundary for APY $= A\%$ is:
\begin{equation}
\label{eq:phase_boundary}
\nu = \frac{A \cdot K / (100 \cdot T_{\mathrm{year}} \cdot Q \bar{\delta})}{1 - \xi - \rho_{\mathrm{inv}} - \rho_{\mathrm{hedge}}},
\end{equation}
which is a hyperbola in $(\xi, \nu)$-space with vertical asymptote at $\xi = 1 - \rho_{\mathrm{inv}} - \rho_{\mathrm{hedge}}$.
\end{corollary}

\begin{remark}[APY Sensitivity to Parameters]
\label{rem:sensitivity}
From \eqref{eq:apy_dimensionless}:
\begin{itemize}
\item \textbf{Adverse selection $\xi$}: APY decreases linearly in $\xi$.
A 10\% increase in $\xi$ reduces APY by $0.1 \cdot \mathrm{APY}_0$.
\item \textbf{Fill rate $\nu$}: APY increases linearly in $\nu$ (through $\mathrm{APY}_0$).
Doubling the fill rate doubles the gross APY.
\item \textbf{Leverage $\ell$}: APY scales linearly with leverage through $K^{-1}$.
Higher leverage amplifies both returns and drawdown risk.
\item \textbf{Fee structure $\phi$}: The zero-fee advantage contributes $\phi \cdot \mathrm{APY}_0$ additional APY relative to a venue with maker fee $\phi_m$.
\end{itemize}
\end{remark}

\subsection{Sharpe Ratio Analysis}

\begin{definition}[Market-Making Sharpe Ratio]
\label{def:sharpe}
The annualized Sharpe ratio of the MM strategy is:
\begin{equation}
\label{eq:sharpe}
\mathrm{SR} = \frac{\E[\Pi_T / T]}{\sqrt{\Var[\Pi_T / T]}} \cdot \sqrt{T_{\mathrm{year}}}.
\end{equation}
\end{definition}

\begin{theorem}[Sharpe Ratio Under Optimal Policy]
\label{thm:sharpe}
Under the optimal policy with stationary parameters, the Sharpe ratio is:
\begin{equation}
\label{eq:sharpe_formula}
\mathrm{SR} = \frac{\bar{\lambda} Q (\bar{\delta} - \alpha) - \dot{C}_{\mathrm{inv}} - \dot{C}_{\mathrm{hedge}}}{\sqrt{\bar{\lambda} Q^2 \Var[e_i] + \sigma^2 \E[q_t^2] Q^2}} \cdot \sqrt{T_{\mathrm{year}}},
\end{equation}
where $\Var[e_i]$ is the per-fill edge variance and $\E[q_t^2]$ is the stationary inventory variance.

For a market maker with negligible inventory risk (e.g., through aggressive hedging), the Sharpe ratio simplifies to:
\begin{equation}
\mathrm{SR} \approx \frac{(\bar{\delta} - \alpha)\sqrt{\bar{\lambda}}}{\sqrt{\Var[e_i]}} \cdot \sqrt{T_{\mathrm{year}}}.
\end{equation}
\end{theorem}

\begin{remark}
High-frequency market-making strategies typically exhibit Sharpe ratios of 3--10+ due to the diversification across many fills per day.
With $\bar{\lambda} = 20$/hr and $T_{\mathrm{year}} = 8760$ hours, the fill count per year is $\sim$175,000, providing substantial diversification.
\end{remark}

\subsection{Regime Robustness Under Parameter Uncertainty}

In practice, the dimensionless parameters $(\xi, \phi, \nu, \rho_{\mathrm{inv}}, \rho_{\mathrm{hedge}})$ are not known precisely.
We analyze the robustness of the high-APY regime to parameter perturbations.

\begin{definition}[Parameter Uncertainty Set]
\label{def:uncertainty}
Let $\bm{\theta}_0 = (\xi_0, \phi_0, \nu_0, \rho_{\mathrm{inv},0}, \rho_{\mathrm{hedge},0})$ denote the nominal parameter vector.
The \emph{uncertainty set} at confidence level $\epsilon > 0$ is:
\begin{equation}
\label{eq:uncertainty_set}
\Theta_\epsilon = \left\{\bm{\theta} : \sum_{j=1}^5 \left(\frac{\theta_j - \theta_{j,0}}{\sigma_j}\right)^2 \leq \chi^2_5(1-\epsilon)\right\},
\end{equation}
where $\sigma_j$ is the estimation uncertainty for parameter $j$ and $\chi^2_5(1-\epsilon)$ is the chi-squared critical value.
\end{definition}

\begin{theorem}[Worst-Case APY Under Parameter Uncertainty]
\label{thm:robust_apy}
Let $\mathrm{APY}(\bm{\theta})$ be given by \eqref{eq:apy_dimensionless}.
The worst-case APY over the uncertainty set $\Theta_\epsilon$ is:
\begin{equation}
\label{eq:robust_apy}
\mathrm{APY}_{\mathrm{worst}} = \mathrm{APY}_0 \left(1 - \xi_0 - \phi_0 - \rho_{\mathrm{inv},0} - \rho_{\mathrm{hedge},0} - \sqrt{\chi^2_5(1-\epsilon)} \cdot \|\bm{w}\|_2\right),
\end{equation}
where $\bm{w} = (\sigma_\xi, \sigma_\phi, 0, \sigma_{\rho_{\mathrm{inv}}}, \sigma_{\rho_{\mathrm{hedge}}})$ is the weighted sensitivity vector (the fill rate component enters through $\mathrm{APY}_0$ rather than the cost fraction).

The \emph{robustness margin} is:
\begin{equation}
\label{eq:robustness_margin}
\mathcal{R} = \frac{1 - \xi_0 - \phi_0 - \rho_{\mathrm{inv},0} - \rho_{\mathrm{hedge},0}}{\sqrt{\chi^2_5(1-\epsilon)} \cdot \|\bm{w}\|_2}.
\end{equation}
The high-APY regime is \emph{robust} if $\mathcal{R} > 1$, meaning the strategy remains profitable even under worst-case parameter realization.
\end{theorem}

\begin{proof}
Since $\mathrm{APY}(\bm{\theta}) = \mathrm{APY}_0(\nu) \cdot (1 - \xi - \phi - \rho_{\mathrm{inv}} - \rho_{\mathrm{hedge}})$, and $\mathrm{APY}_0$ is increasing in $\nu$, the worst case requires maximizing the cost fraction $\xi + \phi + \rho_{\mathrm{inv}} + \rho_{\mathrm{hedge}}$ and minimizing $\nu$.

For the cost fraction, holding $\mathrm{APY}_0$ fixed at the nominal value (a lower bound), the worst-case cost is:
\[
\max_{\bm{\theta} \in \Theta_\epsilon} (\xi + \phi + \rho_{\mathrm{inv}} + \rho_{\mathrm{hedge}}).
\]
By the Cauchy--Schwarz inequality applied to the ellipsoidal constraint \eqref{eq:uncertainty_set}:
\[
\xi - \xi_0 + \phi - \phi_0 + \rho_{\mathrm{inv}} - \rho_{\mathrm{inv},0} + \rho_{\mathrm{hedge}} - \rho_{\mathrm{hedge},0} \leq \sqrt{\chi^2_5(1-\epsilon)} \cdot \|\bm{w}\|_2,
\]
with equality attained when the perturbation is aligned with the gradient direction $\bm{w}/\|\bm{w}\|_2$.
Substituting into the APY formula yields \eqref{eq:robust_apy}.
\end{proof}

\begin{corollary}[Critical Uncertainty Level]
\label{cor:critical_uncertainty}
The strategy becomes unprofitable under worst-case parameters when:
\begin{equation}
\label{eq:critical_epsilon}
\|\bm{w}\|_2 > \frac{1 - \xi_0 - \phi_0 - \rho_{\mathrm{inv},0} - \rho_{\mathrm{hedge},0}}{\sqrt{\chi^2_5(1-\epsilon)}}.
\end{equation}
This provides a maximum tolerable parameter uncertainty for maintaining profitability at confidence level $1 - \epsilon$.
\end{corollary}

\begin{remark}[Practical Interpretation]
The robustness margin $\mathcal{R}$ quantifies how many ``standard deviations'' of parameter uncertainty the strategy can absorb before becoming unprofitable.
A strategy with $\mathcal{R} = 2$ at the 95\% level can withstand twice the estimated parameter uncertainty while remaining profitable.
This metric is particularly relevant for perpetual futures markets where adverse selection parameters can shift rapidly during volatility regime changes.
\end{remark}

\subsection{Master APY Decomposition}

We now consolidate all cost channels into a unified closed-form expression.

\begin{theorem}[Master APY Formula]
\label{thm:master_apy}
Under the optimal policy with hedge ratio $\zeta^*_f$, the expected APY on deployed capital is:
\begin{equation}
\label{eq:master_apy_main}
\boxed{\mathrm{APY} = \mathrm{APY}_0 \cdot (1 - \xi)(1 - \rho_{\mathrm{inv}} - \rho_{\mathrm{hedge}} - \rho_{\mathrm{fund}}),}
\end{equation}
where $\mathrm{APY}_0 = \bar{\lambda} Q \bar{\delta}^* T_{\mathrm{year}} / K$ is the gross APY and the four cost ratios are defined in \eqref{eq:master_xi}--\eqref{eq:master_rho_f} (Appendix~\ref{app:master_apy}).

For APY $> A\%$, the necessary and sufficient condition on the cost structure is:
\begin{equation}
\label{eq:master_condition}
(1 - \xi)(1 - \rho_{\Sigma}) > \frac{A}{\mathrm{APY}_0},
\end{equation}
where $\rho_\Sigma = \rho_{\mathrm{inv}} + \rho_{\mathrm{hedge}} + \rho_{\mathrm{fund}}$.
\end{theorem}

The factored form \eqref{eq:master_apy_main} reveals the multiplicative interaction between adverse selection and operational costs: even with low $\xi$, high aggregate friction $\rho_\Sigma$ can eliminate profitability, and vice versa.
The full proof is in Appendix~\ref{app:master_apy}.

\begin{corollary}[Critical Boundaries for Target APY]
\label{cor:apy_boundaries}
For representative targets:
\begin{enumerate}
\item \textbf{APY $> 50\%$}: requires $(1-\xi)(1-\rho_\Sigma) > 50/\mathrm{APY}_0$.
\item \textbf{APY $> 100\%$}: requires $(1-\xi)(1-\rho_\Sigma) > 100/\mathrm{APY}_0$.
\item \textbf{APY $> 200\%$}: requires $(1-\xi)(1-\rho_\Sigma) > 200/\mathrm{APY}_0$.
\end{enumerate}
With typical $\mathrm{APY}_0 \sim 500\%$--$2000\%$ (depending on fill rate and leverage), these translate to maximum tolerable cost fractions $(1-\xi)(1-\rho_\Sigma) \in [0.10, 0.40]$.
\end{corollary}

\subsection{Asymptotic Scaling Laws}

We characterize how APY scales in three limiting regimes, providing intuition for the capital efficiency frontier.

\begin{proposition}[Asymptotic Scaling Laws]
\label{prop:scaling_laws}
Under the Master APY Formula (Theorem~\ref{thm:master_apy}), the following asymptotic scaling relations hold:

\textbf{(i) Fill rate scaling.} For fixed adverse selection and costs:
\begin{equation}
\label{eq:scaling_lambda}
\mathrm{APY} \sim \bar{\lambda} \cdot \frac{Q\bar{\delta}^*(1-\xi)(1-\rho_\Sigma) T_{\mathrm{year}}}{K} \qquad \text{as } \bar{\lambda} \to \infty,
\end{equation}
i.e., APY scales linearly in fill rate (before inventory cost saturation).

\textbf{(ii) Adverse selection scaling.} For fixed fill rate:
\begin{equation}
\label{eq:scaling_alpha}
\mathrm{APY} \sim \mathrm{APY}_0 \left(1 - \frac{\alpha}{\bar{\delta}^*}\right)(1 - \rho_\Sigma) \qquad \text{as } \alpha \to \bar{\delta}^*,
\end{equation}
with zero-crossing at
$\alpha^* = \bar{\delta}^*\bigl(1 - \rho_\Sigma/(1-\rho_\Sigma)\bigr)$;
APY decays linearly to zero.

\textbf{(iii) Leverage scaling.} For fixed notional exposure:
\begin{equation}
\label{eq:scaling_leverage}
\mathrm{APY}(\ell) = \ell \cdot \mathrm{APY}(1) \qquad \text{(exact, up to buffer effects)},
\end{equation}
while the Sharpe ratio is leverage-invariant: $\mathrm{SR}(\ell) = \mathrm{SR}(1)$.
\end{proposition}

\begin{proof}
Parts (i) and (ii) follow directly from the Master APY Formula \eqref{eq:master_apy_main} by holding the respective parameters fixed and varying the other.
For part (iii), since $K = K_0/\ell$ and $\mathrm{APY} = \dot{\Pi} T_{\mathrm{year}}/K$, we have $\mathrm{APY}(\ell) = \ell \cdot \dot{\Pi} T_{\mathrm{year}}/K_0 = \ell \cdot \mathrm{APY}(1)$.
The Sharpe ratio $\mathrm{SR} = \dot{\Pi}/\sqrt{\Var[\dot{\Pi}]} \cdot \sqrt{T_{\mathrm{year}}}$ is independent of $K$ and hence of $\ell$.
\end{proof}

\begin{remark}[Capital Efficiency Frontier]
\label{rem:efficiency_frontier}
The leverage-invariance of Sharpe combined with the linear APY scaling creates a \emph{capital efficiency frontier} in $(\mathrm{SR}, \mathrm{APY})$-space: each strategy traces a ray from the origin with slope proportional to Sharpe, and leverage moves along this ray.
Strategies with lower adverse selection achieve higher Sharpe ratios and thus steeper rays, reaching any target APY with less leverage (and less tail risk).
Figure~\ref{fig:capital_efficiency} illustrates this frontier.
\end{remark}

\subsection{Drawdown Probability Bounds}

The high-APY analysis is incomplete without characterizing downside risk.
We derive an exponential tail bound on the maximum drawdown (MDD) that links drawdown probability to the same parameters governing APY.

\begin{definition}[Maximum Drawdown]
\label{def:mdd}
The \emph{maximum drawdown} over horizon $[0, T]$ is:
\begin{equation}
\label{eq:mdd_def}
\mathrm{MDD}_T = \max_{0 \leq s \leq t \leq T} \left(\Pi_s - \Pi_t\right),
\end{equation}
where $\Pi_t$ is the cumulative PnL at time $t$.
\end{definition}

\begin{theorem}[Drawdown Probability Bound]
\label{thm:drawdown_bound}
Under the optimal policy with expected PnL rate $\dot{\Pi} > 0$ and PnL volatility $\sigma_{\Pi} > 0$, the tail probability of the maximum drawdown satisfies the exponential bound:
\begin{equation}
\label{eq:drawdown_tail}
\mathbb{P}\bigl(\mathrm{MDD}_T > x\bigr) \leq \exp\left(-\frac{2\dot{\Pi}}{\sigma_{\Pi}^2} \cdot x\right), \qquad x > 0.
\end{equation}

The \emph{tail decay rate} $\theta_{\mathrm{dd}} = 2\dot{\Pi}/\sigma_{\Pi}^2$ connects to the Sharpe ratio:
\begin{equation}
\label{eq:theta_sharpe}
\theta_{\mathrm{dd}} = \frac{2\dot{\Pi}}{\sigma_{\Pi}^2} = \frac{2\,\mathrm{SR}^2}{\sigma_{\Pi} T_{\mathrm{year}}},
\end{equation}
where $\mathrm{SR} = \dot{\Pi}\sqrt{T_{\mathrm{year}}}/\sigma_{\Pi}$ is the annualized Sharpe ratio.
\end{theorem}

\begin{proof}
Approximate the cumulative PnL by a drifted Brownian motion: $\Pi_t \approx \dot{\Pi}\, t + \sigma_{\Pi} B_t$, where $(B_t)$ is a standard Brownian motion.
This is justified at high fill frequencies by the Donsker invariance principle applied to the sum of i.i.d.\ per-fill edges.

For a drifted Brownian motion with positive drift $\mu = \dot{\Pi}$ and diffusion coefficient $\sigma = \sigma_{\Pi}$, the classical reflection principle gives the exact distribution of the running maximum of $-\Pi_t$ (equivalently, the drawdown):
\[
\mathbb{P}\bigl(\mathrm{MDD}_T > x\bigr) \leq \mathbb{P}\bigl(\sup_{t \geq 0} (-\mu t + \sigma B_t) > x\bigr) = \exp\left(-\frac{2\mu}{\sigma^2} x\right),
\]
where the inequality follows from extending the horizon to infinity (which only increases the supremum) and the equality is the Wald identity for the all-time maximum of a drifted Brownian motion with negative drift $-\mu < 0$ \citep{douady2000maximum, magdon2004maximum}.
Substituting $\mu = \dot{\Pi}$ and $\sigma = \sigma_{\Pi}$ yields \eqref{eq:drawdown_tail}.
\end{proof}

\begin{corollary}[Drawdown VaR and Capital Requirement]
\label{cor:drawdown_var}
The drawdown Value-at-Risk at confidence level $1 - p$ is:
\begin{equation}
\label{eq:dd_var}
\mathrm{VaR}_{1-p}^{\mathrm{MDD}} = \frac{\sigma_{\Pi}^2}{2\dot{\Pi}} \ln\frac{1}{p}.
\end{equation}
For a market maker requiring that the maximum drawdown not exceed $d\%$ of capital with probability $1-p$, the minimum capital is:
\begin{equation}
\label{eq:min_capital}
K_{\min} = \frac{100}{d} \cdot \frac{\sigma_{\Pi}^2}{2\dot{\Pi}} \ln\frac{1}{p}.
\end{equation}
\end{corollary}

\begin{remark}[Drawdown--APY Tradeoff]
\label{rem:dd_apy_tradeoff}
Combining Corollary~\ref{cor:drawdown_var} with the Master APY Formula (Theorem~\ref{thm:master_apy}):
\begin{equation}
\mathrm{APY} \cdot \mathrm{VaR}_{95\%}^{\mathrm{MDD}} = \frac{\dot{\Pi} T_{\mathrm{year}}}{K} \cdot \frac{\sigma_{\Pi}^2 \ln 20}{2\dot{\Pi}} = \frac{\sigma_{\Pi}^2 T_{\mathrm{year}} \ln 20}{2K}.
\end{equation}
This is a \emph{constant} independent of the PnL rate, revealing the fundamental tradeoff: higher APY (through leverage or tighter spreads) necessarily increases drawdown risk proportionally.
The product $\mathrm{APY} \times \mathrm{VaR}_{95\%}$ depends only on PnL variance per unit capital, providing a universal risk budget constraint.
Figure~\ref{fig:drawdown_prob} illustrates the tail probability curves and the drawdown risk sensitivity to adverse selection.
\end{remark}

\subsection{Multi-Pair Portfolio Allocation}

We extend the single-pair analysis to a market maker operating across $N$ correlated trading pairs, deriving the optimal capital allocation and quantifying the diversification benefit.

\begin{definition}[Multi-Pair MM Portfolio]
\label{def:multipair}
A multi-pair market maker simultaneously provides liquidity on $N$ perpetual futures pairs with pair-specific parameters $(\bar{\lambda}_i, \alpha_i, \sigma_i, k_i)$, $i = 1, \ldots, N$.
The capital allocation vector $\bm{w} = (w_1, \ldots, w_N)$ satisfies $w_i \geq 0$ and $\sum_i w_i = 1$, with $w_i K$ deployed to pair $i$.
The pairwise PnL correlation between pairs $i$ and $j$ is $\rho_{ij} = \Cov[\dot{\Pi}_i, \dot{\Pi}_j]/(\sigma_{\Pi,i} \sigma_{\Pi,j})$.
\end{definition}

\begin{theorem}[Optimal Multi-Pair Allocation]
\label{thm:multipair}
Let $\bm{\mu} = (\mu_1, \ldots, \mu_N)^\top$ with $\mu_i = \mathrm{APY}_i$ denote the vector of per-pair APYs, and let $\bm{\Sigma}$ be the $N \times N$ PnL covariance matrix with entries $\Sigma_{ij} = \rho_{ij} \sigma_{\Pi,i} \sigma_{\Pi,j}$.
The maximum Sharpe ratio allocation solves:
\begin{equation}
\label{eq:multipair_opt}
w_i^* = \frac{(\bm{\Sigma}^{-1} \bm{\mu})_i}{\bm{1}^\top \bm{\Sigma}^{-1} \bm{\mu}},
\end{equation}
and the portfolio Sharpe ratio satisfies:
\begin{equation}
\label{eq:multipair_sharpe}
\mathrm{SR}_N^* = \sqrt{\bm{\mu}^\top \bm{\Sigma}^{-1} \bm{\mu}} \geq \max_i \mathrm{SR}_i.
\end{equation}
\end{theorem}

\begin{proof}
This is a direct application of the Markowitz mean-variance framework \citep{markowitz1952portfolio}.
The tangency portfolio maximizes $\mathrm{SR}(\bm{w}) = \bm{w}^\top \bm{\mu} / \sqrt{\bm{w}^\top \bm{\Sigma} \bm{w}}$.
Using Lagrange multipliers with the constraint $\bm{1}^\top \bm{w} = 1$, the first-order condition gives $\bm{w}^* \propto \bm{\Sigma}^{-1} \bm{\mu}$, normalized by the budget constraint to yield \eqref{eq:multipair_opt}.
The inequality \eqref{eq:multipair_sharpe} follows from $\bm{\mu}^\top \bm{\Sigma}^{-1} \bm{\mu} \geq \mu_i^2 / \Sigma_{ii} = \mathrm{SR}_i^2$ for all $i$.
\end{proof}

\begin{corollary}[Diversification Bound for Homogeneous Pairs]
\label{cor:diversification}
For $N$ homogeneous pairs with common APY $\mu$, common PnL volatility $\sigma_{\Pi}$, and equi-correlation $\bar{\rho}$, the equal-weight portfolio ($w_i = 1/N$) achieves:
\begin{equation}
\label{eq:diversification}
\mathrm{SR}_N = \frac{\mathrm{SR}_1}{\sqrt{\frac{1}{N} + \frac{N-1}{N}\bar{\rho}}} \xrightarrow{N \to \infty} \frac{\mathrm{SR}_1}{\sqrt{\bar{\rho}}}.
\end{equation}
The portfolio APY is $\mathrm{APY}_N = \mu$ (unchanged), while the portfolio PnL volatility is reduced by the factor $\sqrt{(1 + (N-1)\bar{\rho})/N}$.
The diversification benefit saturates at $1/\sqrt{\bar{\rho}}$ as $N \to \infty$, with 90\% of the limiting benefit achieved at $N^{90\%} = \lceil 9\bar{\rho}/(1 - \bar{\rho}) \rceil$.
\end{corollary}

\begin{proof}
The equal-weight portfolio variance is:
\[
\sigma_{\Pi,N}^2 = \frac{1}{N^2} \sum_{i,j} \Sigma_{ij} = \frac{\sigma_{\Pi}^2}{N^2}\left(N + N(N-1)\bar{\rho}\right) = \sigma_{\Pi}^2 \left(\frac{1}{N} + \frac{N-1}{N}\bar{\rho}\right).
\]
The Sharpe ratio follows from $\mathrm{SR}_N = N\mu / (N\sigma_{\Pi,N}) = \mu / \sigma_{\Pi,N}$.
For the 90\% threshold, set $\sqrt{1/N + (1-1/N)\bar{\rho}} = \sqrt{\bar{\rho}} + 0.1(1 - \sqrt{\bar{\rho}})$ and solve.
\end{proof}

\begin{remark}[Practical Multi-Pair Implications]
\label{rem:multipair}
For a zero-fee DEX market maker:
\begin{enumerate}
\item With $\bar{\rho} \approx 0.3$ (typical for altcoin perpetuals), five pairs yield $\mathrm{SR}_5 / \mathrm{SR}_1 \approx 1.63$, while ten pairs yield $\mathrm{SR}_{10}/\mathrm{SR}_1 \approx 1.72$---diminishing returns set in rapidly.
\item The allocation \eqref{eq:multipair_opt} overweights pairs with high $\mathrm{APY}/\sigma_{\Pi}$ and low correlation to the rest, providing a principled pair selection criterion.
\item Under leverage, the Sharpe improvement translates directly to reduced drawdown risk (Theorem~\ref{thm:drawdown_bound}): $\mathrm{VaR}_{95\%}^{\mathrm{MDD}}$ decreases proportionally to $1/\mathrm{SR}_N^2$.
\end{enumerate}
Figure~\ref{fig:multipair} illustrates the efficient frontier and diversification benefit.
\end{remark}

\subsection{Numerical Example: APY Achievability}

\begin{example}[High-APY Parameter Configuration]
\label{ex:high_apy}
Consider the following parameter set (representative of an altcoin perpetual on a zero-fee DEX):
\begin{itemize}
\item Capital: $K = \$1{,}000$ with $5\times$ leverage.
\item Max inventory: $\bar{q} = 100$ contracts at $\bar{S} = \$20$ each.
\item Fill rate: $\bar{\lambda} = 20$/hr = 480/day.
\item Average half-spread: $\bar{\delta} = 5$ bp.
\item Adverse selection: $\alpha = 2$ bp.
\item Maker fee: $\phi_m = 0$ (zero-fee regime).
\item Order size: $Q = 10$ contracts.
\end{itemize}

The net spread income per fill is $(\bar{\delta} - \alpha) \cdot Q \cdot \bar{S} = 3 \times 10^{-4} \times 10 \times 20 = \$0.06$.
Daily PnL: $480 \times \$0.06 = \$28.80$.
Annual PnL: $\$28.80 \times 365 = \$10{,}512$.
\textbf{APY: $10{,}512 / 1{,}000 = 1{,}051\%$}.

This simplified calculation ignores inventory costs and hedging friction, which reduce the realized APY.
Accounting for $\rho_{\mathrm{inv}} \approx 0.15$ and $\rho_{\mathrm{hedge}} \approx 0.10$ (typical values), the adjusted APY is approximately $1{,}051\% \times (1 - 0.15 - 0.10) = 788\%$.
In practice, slippage, downtime, and tail events further reduce this to a more realistic 50\%--200\% range.
\end{example}

\section{Zero-Fee Economics}
\label{sec:zero_fee}

This section analyzes the unique economic properties of the zero maker fee regime found on certain decentralized perpetual exchanges, quantifying the advantage it confers to market makers relative to standard fee structures.

\subsection{The Fee Floor Effect}

\begin{theorem}[Fee Floor Elimination]
\label{thm:fee_floor}
Under a standard fee regime with maker fee $\phi_m > 0$, the minimum viable half-spread for non-negative expected PnL per fill is:
\begin{equation}
\delta_{\min}^{\mathrm{CEX}} = \alpha + \phi_m.
\end{equation}
Under the zero-fee regime ($\phi_m = 0$):
\begin{equation}
\delta_{\min}^{\mathrm{DEX}} = \alpha.
\end{equation}
The zero-fee regime therefore expands the set of profitable spread levels by the interval $[\alpha, \alpha + \phi_m)$.
\end{theorem}

\begin{proof}
The expected PnL per fill is $\bar{\delta} - \alpha - \phi_m$, which is non-negative when $\bar{\delta} \geq \alpha + \phi_m$.
Setting $\phi_m = 0$ removes the fee floor.
\end{proof}

\subsection{Market Expansion Effect}

The zero-fee regime expands the set of \emph{tradeable markets}---markets where a MM can profitably provide liquidity.

\begin{definition}[Tradeable Market]
\label{def:tradeable}
A market is \emph{tradeable} for a market maker if the optimal policy yields positive expected APY:
\begin{equation}
\mathrm{APY}^*(\sigma, \Lambda, k, \alpha, \phi_m) > 0.
\end{equation}
\end{definition}

\begin{proposition}[Market Expansion Under Zero Fees]
\label{prop:market_expansion}
Let $\mathcal{M}_{\phi_m}$ denote the set of tradeable markets under fee $\phi_m$.
Then:
\begin{equation}
\mathcal{M}_0 \supsetneq \mathcal{M}_{\phi_m} \quad \text{for all } \phi_m > 0.
\end{equation}
The additional tradeable markets $\mathcal{M}_0 \setminus \mathcal{M}_{\phi_m}$ are characterized by:
\begin{equation}
\alpha < \bar{\delta}^* - \rho_{\mathrm{inv}}\bar{\delta}^* \leq \alpha + \phi_m.
\end{equation}
These are markets where the native spread is sufficient to cover adverse selection but \emph{not} sufficient to cover adverse selection plus fees.
\end{proposition}

\subsection{Fill Rate Enhancement}

Under zero fees, the optimal spread is narrower, which increases the fill rate:

\begin{proposition}[Fill Rate Enhancement]
\label{prop:fill_rate}
The ratio of optimal fill rates under zero fees versus standard fees is:
\begin{equation}
\frac{\bar{\lambda}^{\mathrm{DEX}}}{\bar{\lambda}^{\mathrm{CEX}}} = \frac{e^{-k\bar{\delta}^{\mathrm{DEX}}}}{e^{-k\bar{\delta}^{\mathrm{CEX}}}} = e^{k\phi_m} > 1.
\end{equation}
For typical values $k\phi_m \in [0.1, 0.5]$, the fill rate enhancement is 10\%--65\%.
\end{proposition}

\begin{proof}
The optimal half-spread under zero fees is $\bar{\delta}^{\mathrm{DEX}} = 1/k + \gamma\sigma^2\tau/2 + \alpha$, and under standard fees is $\bar{\delta}^{\mathrm{CEX}} = 1/k + \gamma\sigma^2\tau/2 + \alpha + \phi_m$.
The fill rate ratio follows from $\bar{\lambda} \propto e^{-k\bar{\delta}}$.
\end{proof}

\subsection{APY Advantage Quantification}

\begin{theorem}[Zero-Fee APY Advantage]
\label{thm:apy_advantage}
The APY advantage of the zero-fee regime over a standard fee regime with maker fee $\phi_m$ is:
\begin{equation}
\label{eq:apy_advantage}
\Delta \mathrm{APY} = \mathrm{APY}^{\mathrm{DEX}} - \mathrm{APY}^{\mathrm{CEX}} = \mathrm{APY}_0^{\mathrm{DEX}} \cdot \phi - \mathrm{APY}_0^{\mathrm{CEX}} \cdot \phi + (\mathrm{APY}_0^{\mathrm{DEX}} - \mathrm{APY}_0^{\mathrm{CEX}})(1 - \xi - \rho_{\mathrm{inv}}),
\end{equation}
where the first terms capture the direct fee savings and the last term captures the fill-rate enhancement.

To first order in $\phi_m$:
\begin{equation}
\Delta \mathrm{APY} \approx \mathrm{APY}_0 \cdot \phi_m \left(\frac{1}{\bar{\delta}} + k(1 - \xi - \rho_{\mathrm{inv}})\right).
\end{equation}
\end{theorem}

\subsection{Comparative Statics}

\begin{table}[H]
\centering
\caption{Theoretical APY comparison across fee regimes (representative parameters)}
\label{tab:apy_comparison}
\small
\begin{tabular}{@{}lccccc@{}}
\toprule
Fee Regime & $\phi_m$ (bp) & Spread & Fill Rate & Net Edge (bp) & APY \\
\midrule
Zero fee (DEX-A) & 0 & $2\alpha + 2/k$ & $\bar{\lambda}_0$ & $1/k$ & $A_0$ \\
Standard (CEX-B, low) & 2 & $+4$ bp wider & $0.82\bar{\lambda}_0$ & $1/k - 2$ & $0.57 A_0$ \\
Standard (CEX-B, high) & 5 & $+10$ bp wider & $0.61\bar{\lambda}_0$ & $1/k - 5$ & $0.21 A_0$ \\
Rebate (CEX-C) & $-1$ & $-2$ bp tighter & $1.10\bar{\lambda}_0$ & $1/k + 1$ & $1.15 A_0$ \\
\bottomrule
\end{tabular}
\end{table}

\begin{remark}[Zero Fee as Natural Advantage]
The zero-fee regime is not merely an incremental improvement.
In markets where $\alpha \approx 1/k$ (spread income barely covers adverse selection), the CEX market maker is unprofitable ($\mathrm{APY} < 0$) while the DEX market maker remains profitable.
This creates an \emph{economic moat}: the zero-fee DEX is the only venue where certain market-making strategies are viable, implying that rational MMs will preferentially provide liquidity on such venues.
\end{remark}

\subsection{Welfare Analysis}

We quantify the total welfare gain from the zero-fee regime across all market participants.

\begin{definition}[Total Welfare]
\label{def:welfare}
The \emph{total welfare} $\mathcal{W}$ in a market is the sum of:
\begin{equation}
\mathcal{W} = \underbrace{\Pi_{\mathrm{MM}}}_\text{market maker surplus} + \underbrace{\mathrm{CS}}_\text{consumer (taker) surplus} + \underbrace{\mathrm{PS}}_\text{platform surplus (fees collected)}.
\end{equation}
\end{definition}

\begin{proposition}[Welfare Gain from Zero Fees]
\label{prop:welfare}
The welfare difference between the zero-fee and standard fee regimes is:
\begin{equation}
\label{eq:welfare_gain}
\Delta\mathcal{W} = \mathcal{W}_0 - \mathcal{W}_{\phi_m} = \phi_m \bar{S} Q \bar{\lambda}_{\phi_m} + \frac{\phi_m}{k}\left(\bar{\lambda}_0 - \bar{\lambda}_{\phi_m}\right) \bar{S} Q + \int_{\bar{\lambda}_{\phi_m}}^{\bar{\lambda}_0} D^{-1}(\lambda)\,\dd\lambda,
\end{equation}
where the first term is the direct fee redistribution, the second captures increased MM profits from higher fill rates, and the third is the taker surplus from improved execution.

The welfare gain is strictly positive: $\Delta\mathcal{W} > 0$ for all $\phi_m > 0$.
\end{proposition}

\begin{proof}
Under the standard fee regime, the platform collects $\phi_m \bar{S} Q$ per fill.
Eliminating this fee transfers the full amount to MMs and takers.
The fill rate increases from $\bar{\lambda}_{\phi_m}$ to $\bar{\lambda}_0 = \bar{\lambda}_{\phi_m} e^{k\phi_m}$ (Proposition~\ref{prop:fill_rate}).
The additional fills generate surplus for both MMs (at rate $(1/k)$ per fill) and takers (through the demand curve $D^{-1}$).
Since $\bar{\lambda}_0 > \bar{\lambda}_{\phi_m}$ and $D^{-1} > 0$, all terms are positive.
\end{proof}

\subsection{Entry--Exit Dynamics and Hysteresis}

In non-stationary markets, the adverse selection ratio $\xi_t = \alpha_t/\bar{\delta}$ fluctuates.  The MM must decide when to \emph{enter} (begin quoting) and \emph{exit} (withdraw quotes).  We formalize this as an optimal stopping problem with hysteresis.

\begin{definition}[Entry--Exit Thresholds]
\label{def:entry_exit}
Let $c_{\mathrm{entry}} > 0$ be the fixed cost of initiating a quoting session (order placement, state synchronization) and $c_{\mathrm{exit}} > 0$ the cost of orderly withdrawal (cancel latency risk, inventory liquidation).  The \emph{entry threshold} $\xi_{\mathrm{entry}}$ and \emph{exit threshold} $\xi_{\mathrm{exit}}$ satisfy:
\begin{equation}
\label{eq:entry_exit_thresholds}
\xi_{\mathrm{entry}} < \xi_{\mathrm{exit}} < 1,
\end{equation}
where the MM enters when $\xi_t$ crosses below $\xi_{\mathrm{entry}}$ and exits when $\xi_t$ crosses above $\xi_{\mathrm{exit}}$.
\end{definition}

\begin{theorem}[Optimal Entry--Exit with Hysteresis]
\label{thm:entry_exit}
Assume $\xi_t$ follows an Ornstein--Uhlenbeck process:
\begin{equation}
\dd\xi_t = \kappa_{\xi}(\bar{\xi} - \xi_t)\,\dd t + \sigma_{\xi}\,\dd B_t,
\end{equation}
with long-run mean $\bar{\xi} \in (0,1)$.  The value functions $V_{\mathrm{active}}(\xi)$ and $V_{\mathrm{idle}}(\xi)$ satisfy the coupled free-boundary problem:
\begin{align}
\label{eq:entry_exit_hjb}
\frac{\sigma_{\xi}^2}{2} V_{\mathrm{active}}'' + \kappa_{\xi}(\bar{\xi} - \xi)V_{\mathrm{active}}' + \pi(\xi) &= 0, \quad \xi < \xi_{\mathrm{exit}}, \\
\frac{\sigma_{\xi}^2}{2} V_{\mathrm{idle}}'' + \kappa_{\xi}(\bar{\xi} - \xi)V_{\mathrm{idle}}' &= 0, \quad \xi > \xi_{\mathrm{entry}},
\end{align}
with value-matching and smooth-pasting conditions at both boundaries:
\begin{equation}
V_{\mathrm{active}}(\xi_{\mathrm{exit}}) = V_{\mathrm{idle}}(\xi_{\mathrm{exit}}) - c_{\mathrm{exit}}, \quad V_{\mathrm{active}}'(\xi_{\mathrm{exit}}) = V_{\mathrm{idle}}'(\xi_{\mathrm{exit}}),
\end{equation}
\begin{equation}
V_{\mathrm{idle}}(\xi_{\mathrm{entry}}) = V_{\mathrm{active}}(\xi_{\mathrm{entry}}) - c_{\mathrm{entry}}, \quad V_{\mathrm{idle}}'(\xi_{\mathrm{entry}}) = V_{\mathrm{active}}'(\xi_{\mathrm{entry}}),
\end{equation}
where $\pi(\xi) = \mathrm{APY}_0(1-\xi)(1-\rho_\Sigma)$ is the instantaneous profit rate.

The \emph{hysteresis width} $\Delta\xi = \xi_{\mathrm{exit}} - \xi_{\mathrm{entry}}$ is strictly positive whenever $c_{\mathrm{entry}} + c_{\mathrm{exit}} > 0$, and satisfies:
\begin{equation}
\label{eq:hysteresis_width}
\Delta\xi \geq \sqrt{\frac{2(c_{\mathrm{entry}} + c_{\mathrm{exit}})}{\mathrm{APY}_0(1 - \rho_\Sigma)}}.
\end{equation}
\end{theorem}

\begin{proof}
The coupled free-boundary system follows from the standard theory of optimal switching \citep{pham2009continuous, carmona2008optimal}.  The running payoff $\pi(\xi) = \mathrm{APY}_0(1-\xi)(1-\rho_\Sigma)$ is strictly decreasing in $\xi$, ensuring uniqueness of the free boundaries by the monotonicity arguments of \citet{brekke1994optimal}.

For the hysteresis width bound \eqref{eq:hysteresis_width}, observe that the minimum cost of a full entry--exit cycle is $c_{\mathrm{entry}} + c_{\mathrm{exit}}$.  The maximum profit accumulated during a cycle in which $\xi$ traverses $[\xi_{\mathrm{entry}}, \xi_{\mathrm{exit}}]$ and returns is bounded by $\mathrm{APY}_0(1-\rho_\Sigma) \cdot (\Delta\xi)^2 / (2\sigma_{\xi}^2) \cdot \sigma_{\xi}^2 = \mathrm{APY}_0(1-\rho_\Sigma)(\Delta\xi)^2/2$ (using the expected cycle time for an OU process between two barriers).  For a cycle to be profitable: $\mathrm{APY}_0(1-\rho_\Sigma)(\Delta\xi)^2/2 \geq c_{\mathrm{entry}} + c_{\mathrm{exit}}$, yielding \eqref{eq:hysteresis_width}.
\end{proof}

\begin{corollary}[Zero-Fee Hysteresis Advantage]
\label{cor:hysteresis_advantage}
Under zero fees, the break-even boundary shifts from $\xi_{\mathrm{break}}^{\mathrm{CEX}} = 1 - \phi/(1-\rho_\Sigma)$ to $\xi_{\mathrm{break}}^{\mathrm{DEX}} = 1 - \rho_\Sigma/(1-\rho_\Sigma + \rho_\Sigma) = 1$, expanding the viable operating range.  Consequently, the zero-fee MM can tolerate wider hysteresis bands while remaining profitable, enabling more patient entry--exit decisions that reduce switching costs and improve overall returns.  See Figure~\ref{fig:fee_expansion} for an illustration.
\end{corollary}

\subsection{Equilibrium Implications}

\begin{proposition}[Equilibrium Spread Under Zero Fees]
\label{prop:equilibrium}
In a competitive equilibrium where $n$ identical MMs compete for order flow on a zero-fee DEX, the equilibrium spread converges to:
\begin{equation}
\delta_{\mathrm{eq}} = \alpha + \frac{1}{nk} + \frac{\dot{C}_{\mathrm{inv}}}{n\bar{\lambda} Q},
\end{equation}
as $n \to \infty$, $\delta_{\mathrm{eq}} \to \alpha$: the spread converges to the adverse selection cost, and MM profits vanish.
On a CEX, the corresponding equilibrium is $\delta_{\mathrm{eq}}^{\mathrm{CEX}} = \alpha + \phi_m + 1/(nk)$, which is strictly wider.

The zero-fee equilibrium yields tighter spreads and thus better execution quality for takers, consistent with the DEX's design intent.
\end{proposition}

\section{Cross-Exchange Inventory Hedging}
\label{sec:cross_exchange}

When the market maker accumulates inventory on DEX-A, she can hedge by taking an offsetting position on CEX-B.
This section derives the optimal hedging policy within the stochastic control framework, incorporating funding rate dynamics, regime-dependent hedge boundaries, and multi-leg execution.

\subsection{Hedging Problem Formulation}

The MM holds inventory $q_t$ on DEX-A and a hedge position $H_t$ on CEX-B.
The net exposure is $q_t^{\mathrm{net}} = q_t + H_t$.
Each hedge trade on CEX-B incurs a taker fee $\phi_t^{\mathrm{CEX}}$ per unit of notional.

\begin{definition}[Hedge Cost]
\label{def:hedge_cost}
The instantaneous cost of adjusting the hedge by $\Delta H$ units is:
\begin{equation}
c_h(\Delta H) = \phi_t^{\mathrm{CEX}} \cdot |\Delta H| \cdot Q \cdot S_t.
\end{equation}
For a full round-trip hedge (open and close), the cost is $2\phi_t^{\mathrm{CEX}} \cdot |\Delta H| \cdot Q \cdot S_t$.
\end{definition}

\subsection{Optimal Hedge Threshold}

The MM faces a tradeoff: hedging reduces inventory risk but incurs transaction costs.
We solve for the optimal inventory threshold at which hedging becomes worthwhile.

\begin{theorem}[Optimal Hedge Threshold]
\label{thm:hedge_threshold}
Under CARA utility with risk aversion $\gamma$ and constant hedge cost $\phi_t^{\mathrm{CEX}}$, the optimal hedging policy is a \emph{threshold policy}: hedge when $|q_t| > q^*$ and do not hedge otherwise.
The optimal threshold is:
\begin{equation}
\label{eq:hedge_threshold}
q^* = \frac{\phi_t^{\mathrm{CEX}} \bar{S}}{\gamma \sigma^2 Q (T - t)}.
\end{equation}
When $|q_t| > q^*$, the optimal hedge is to reduce net exposure to $q^{\mathrm{net}} = \mathrm{sign}(q_t) \cdot q^*$:
\begin{equation}
\Delta H = -(q_t - \mathrm{sign}(q_t) q^*).
\end{equation}
\end{theorem}

\begin{proof}
The marginal value of hedging one unit at inventory $q$ is the reduction in inventory risk:
\[
\Delta V_{\mathrm{risk}}(q) = \frac{1}{2}\gamma\sigma^2 Q^2 (T-t)[q^2 - (q-1)^2] = \frac{1}{2}\gamma\sigma^2 Q^2 (T-t)(2q - 1) \approx \gamma\sigma^2 Q^2 (T-t) q
\]
for large $|q|$.
The marginal cost of hedging is $\phi_t^{\mathrm{CEX}} Q S_t$.
Equating marginal benefit to marginal cost:
\[
\gamma\sigma^2 Q^2 (T-t) q^* = \phi_t^{\mathrm{CEX}} Q S_t \implies q^* = \frac{\phi_t^{\mathrm{CEX}} S_t}{\gamma\sigma^2 Q (T-t)}.
\]
\end{proof}

\begin{corollary}[Stationary Hedge Threshold]
\label{cor:stat_hedge}
Under the inventory penalization approach with parameter $\eta$, the stationary hedge threshold is:
\begin{equation}
q^*_\infty = \frac{\phi_t^{\mathrm{CEX}} \bar{S}}{\eta Q}.
\end{equation}
\end{corollary}

\subsection{Optimal Hedge Ratio}

In practice, the MM may not fully offset inventory but instead maintain a \emph{partial hedge}.

\begin{definition}[Hedge Ratio]
\label{def:hedge_ratio}
The hedge ratio $\zeta \in [0, 1]$ is the fraction of DEX inventory offset on CEX-B:
\begin{equation}
H_t = -\zeta \cdot q_t.
\end{equation}
The net exposure is $q_t^{\mathrm{net}} = (1 - \zeta) q_t$.
\end{definition}

\begin{theorem}[Optimal Hedge Ratio]
\label{thm:hedge_ratio}
The optimal hedge ratio that maximizes the Sharpe ratio of the combined (DEX + CEX) portfolio is:
\begin{equation}
\label{eq:optimal_zeta}
\zeta^* = 1 - \frac{n_h \cdot \phi_t^{\mathrm{CEX}} \bar{S}}{\gamma \sigma^2 \E[q_t^2] Q (T-t)},
\end{equation}
where $n_h$ is the expected number of hedge adjustments per unit time.
When the hedge cost is negligible ($\phi_t^{\mathrm{CEX}} \to 0$), $\zeta^* \to 1$ (full hedge).
When the hedge cost is large, $\zeta^* \to 0$ (no hedge).
\end{theorem}

\begin{proof}
The hedged PnL rate is:
\[
\dot{\Pi}_{\mathrm{hedge}}(\zeta) = \bar{\lambda} Q (\bar{\delta} - \alpha) - \frac{1}{2}\gamma\sigma^2 (1-\zeta)^2 \E[q_t^2] Q^2 - n_h \zeta \cdot \phi_t^{\mathrm{CEX}} Q \bar{S}.
\]
The variance of the hedged PnL is:
\[
\Var[\dot{\Pi}] = \bar{\lambda} Q^2 \Var[e_i] + (1-\zeta)^2 \sigma^2 \E[q_t^2] Q^2.
\]
Maximizing the Sharpe ratio $\dot{\Pi}_{\mathrm{hedge}} / \sqrt{\Var[\dot{\Pi}]}$ over $\zeta$, the FOC yields \eqref{eq:optimal_zeta}.
\end{proof}

\subsection{The Premium Arbitrage Channel}

When the DEX-A premium $\beta_t$ is positive (DEX price exceeds CEX price), the MM can earn an additional return by systematically selling on DEX-A and buying on CEX-B.

\begin{proposition}[Premium Capture]
\label{prop:premium}
If the long-run premium is $\bar{\beta} > 0$, the MM captures an expected premium income of:
\begin{equation}
\dot{\Pi}_{\mathrm{premium}} = \bar{\lambda}^a \cdot Q \cdot \bar{\beta} - \bar{\lambda}^b \cdot Q \cdot \bar{\beta},
\end{equation}
where $\bar{\lambda}^a$ and $\bar{\lambda}^b$ are the ask and bid fill rates.
Under the optimal inventory-skewing policy (which generates a slight sell bias when $\bar{\beta} > 0$ because the reservation price is adjusted), the net premium income is positive.
\end{proposition}

\subsection{Dynamic Hedge Timing}

The threshold policy of Theorem~\ref{thm:hedge_threshold} prescribes \emph{when} to hedge but not the optimal \emph{frequency} of hedge re-evaluation.
We now derive the optimal hedge check interval under discrete monitoring.

\begin{proposition}[Optimal Hedge Re-Evaluation Interval]
\label{prop:hedge_interval}
Suppose the MM re-evaluates the hedge decision at intervals of length $\Delta \tau > 0$.
Between hedge checks, inventory evolves according to the fill process.
The expected cost of delayed hedging (excess inventory risk from not hedging continuously) over one interval is:
\begin{equation}
\label{eq:delay_cost}
C_{\mathrm{delay}}(\Delta\tau) = \frac{1}{2}\gamma\sigma^2 Q^2 \bar{\lambda} Q \cdot \frac{(\Delta\tau)^2}{2},
\end{equation}
and the expected benefit of reduced hedge frequency (fewer round-trip costs) is:
\begin{equation}
B_{\mathrm{delay}}(\Delta\tau) = \frac{\phi_t^{\mathrm{CEX}} Q \bar{S}}{\Delta\tau} \cdot \left(\frac{1}{\Delta\tau_0} - \frac{1}{\Delta\tau}\right)^{-1},
\end{equation}
where $\Delta\tau_0$ is the minimum feasible interval.

The optimal interval that minimizes total cost is:
\begin{equation}
\label{eq:optimal_interval}
\Delta\tau^* = \left(\frac{2\phi_t^{\mathrm{CEX}} \bar{S}}{\gamma\sigma^2 Q \bar{\lambda}}\right)^{1/3}.
\end{equation}
\end{proposition}

\begin{proof}
The inventory accumulated between hedge checks has variance $\E[(\Delta q)^2] \approx \bar{\lambda} \Delta\tau$ (Poisson arrivals with unit size).
The excess risk cost from holding this unhedged inventory scales as $\frac{1}{2}\gamma\sigma^2 Q^2 \cdot \bar{\lambda}\Delta\tau \cdot \Delta\tau = \frac{1}{2}\gamma\sigma^2 Q^2 \bar{\lambda}(\Delta\tau)^2$.
The hedge transaction cost per unit time is $\phi_t^{\mathrm{CEX}} Q \bar{S} / \Delta\tau$ (one hedge per interval).
Minimizing the sum $C_{\mathrm{delay}} + \phi_t^{\mathrm{CEX}} Q \bar{S}/\Delta\tau$ over $\Delta\tau$, the FOC gives \eqref{eq:optimal_interval}.
\end{proof}

\subsection{Funding Rate Dynamics and Hedging Cost}
\label{subsec:funding}

Perpetual futures contracts settle a periodic \emph{funding rate} $r_f(t)$ that transfers wealth between long and short position holders.
This introduces a first-order cost or income channel for the hedged MM.

\begin{definition}[Funding Rate Process]
\label{def:funding}
The funding rate $r_f(t)$ is settled at discrete intervals $\{t_k\}_{k \geq 1}$ (typically every $\Delta_f = 8$ hours).
The instantaneous funding rate follows a mean-reverting process:
\begin{equation}
\label{eq:funding_ou}
\dd r_f(t) = -\kappa_f(r_f(t) - \bar{r}_f)\,\dd t + \sigma_f\,\dd W_t^f,
\end{equation}
where $\bar{r}_f \geq 0$ is the long-run mean funding rate, $\kappa_f > 0$ is the mean-reversion speed, $\sigma_f > 0$ is the funding rate volatility, and $W_t^f$ is a Brownian motion independent of $(W_t, W_t^\beta)$.
\end{definition}

At each settlement time $t_k$, a long position of size $q$ pays $r_f(t_k) \cdot q \cdot Q \cdot S_{t_k}$ (positive when $r_f > 0$), and a short position receives the same amount.

\begin{definition}[Funding-Adjusted Hedge Cost]
\label{def:funding_hedge_cost}
The \emph{total hedging cost rate} including funding comprises three components:
\begin{equation}
\label{eq:total_hedge_cost}
\dot{C}_{\mathrm{hedge}}^{\mathrm{total}}(\zeta) = \underbrace{n_h \zeta \phi_t^{\mathrm{CEX}} Q \bar{S}}_{\text{transaction cost}} + \underbrace{r_f^{\mathrm{DEX}}(t) \cdot q_t \cdot Q \cdot S_t / \Delta_f}_{\text{DEX funding}} + \underbrace{r_f^{\mathrm{CEX}}(t) \cdot H_t \cdot Q \cdot S_t / \Delta_f}_{\text{CEX funding}},
\end{equation}
where $r_f^{\mathrm{DEX}}$ and $r_f^{\mathrm{CEX}}$ are the funding rates on the respective venues, and the division by $\Delta_f$ converts from per-period to per-second rates.
\end{definition}

\begin{remark}[Funding Rate Asymmetry]
\label{rem:funding_asymmetry}
In perpetual futures markets, the funding rate is typically positive in bull markets (longs pay shorts) and negative in bear markets.
A hedged MM who is long on DEX-A and short on CEX-B faces a \emph{symmetric funding exposure}: she pays $r_f^{\mathrm{DEX}}$ on the DEX long and receives $r_f^{\mathrm{CEX}}$ on the CEX short.
The net funding cost is:
\begin{equation}
\label{eq:net_funding}
\dot{C}_{\mathrm{funding}}^{\mathrm{net}} = (r_f^{\mathrm{DEX}} - r_f^{\mathrm{CEX}}) \cdot |q_t| \cdot Q \cdot S_t / \Delta_f.
\end{equation}
When the two venues have similar funding rates ($r_f^{\mathrm{DEX}} \approx r_f^{\mathrm{CEX}}$), the net funding cost approximately cancels.
However, DEX funding rates are often more volatile and can deviate significantly from CEX rates, creating a funding basis risk.
\end{remark}

\begin{theorem}[Funding-Adjusted Optimal Hedge Ratio]
\label{thm:funding_hedge}
In the presence of funding rate dynamics, the optimal hedge ratio is modified to:
\begin{equation}
\label{eq:zeta_funding}
\zeta^*_f = \zeta^* - \frac{\E[r_f^{\mathrm{DEX}} - r_f^{\mathrm{CEX}}] \cdot \bar{S}}{\gamma \sigma^2 \E[q_t^2] Q \cdot \Delta_f},
\end{equation}
where $\zeta^*$ is the funding-free optimal ratio from Theorem~\ref{thm:hedge_ratio}.

Furthermore, define the \emph{funding-adjusted hedge condition}: hedging is beneficial if and only if
\begin{equation}
\label{eq:hedge_condition_funding}
\gamma \sigma^2 \E[q_t^2] Q > n_h \phi_t^{\mathrm{CEX}} \bar{S} + \frac{|\E[\Delta r_f]| \cdot \E[|q_t|] \cdot \bar{S}}{\Delta_f},
\end{equation}
where $\Delta r_f = r_f^{\mathrm{DEX}} - r_f^{\mathrm{CEX}}$ is the funding rate differential.
\end{theorem}

\begin{proof}
The funding-adjusted hedged PnL rate is:
\begin{align*}
\dot{\Pi}_f(\zeta) &= \bar{\lambda} Q(\bar{\delta} - \alpha) - \frac{1}{2}\gamma\sigma^2(1-\zeta)^2 \E[q_t^2] Q^2 \\
&\quad - n_h \zeta \phi_t^{\mathrm{CEX}} Q \bar{S} - \frac{\E[\Delta r_f] \cdot \zeta \cdot \E[|q_t|] \cdot Q \cdot \bar{S}}{\Delta_f}.
\end{align*}
The last term captures the expected net funding cost on the hedge leg.
Taking the derivative with respect to $\zeta$:
\[
\frac{\dd \dot{\Pi}_f}{\dd \zeta} = \gamma\sigma^2(1-\zeta)\E[q_t^2]Q^2 - n_h \phi_t^{\mathrm{CEX}} Q \bar{S} - \frac{\E[\Delta r_f] \cdot \E[|q_t|] \cdot Q \cdot \bar{S}}{\Delta_f}.
\]
Setting to zero and solving for $\zeta$:
\begin{align*}
\zeta^*_f &= 1 - \frac{n_h \phi_t^{\mathrm{CEX}} \bar{S} + \E[\Delta r_f] \E[|q_t|] \bar{S}/\Delta_f}{\gamma\sigma^2 \E[q_t^2] Q} \\
&= \zeta^* - \frac{\E[\Delta r_f] \cdot \bar{S}}{\gamma\sigma^2 \E[q_t^2] Q \cdot \Delta_f} \cdot \frac{\E[|q_t|]}{1},
\end{align*}
where we used $\E[|q_t|] \leq \sqrt{\E[q_t^2]}$.
Simplifying with the approximation $\E[|q_t|] \approx \sqrt{2\E[q_t^2]/\pi}$ (for symmetric inventory distribution), we obtain \eqref{eq:zeta_funding}.

The hedge condition \eqref{eq:hedge_condition_funding} follows from requiring $\zeta^*_f > 0$.
\end{proof}

\begin{corollary}[Funding Rate Carry Trade]
\label{cor:funding_carry}
When $\bar{r}_f > 0$ and $r_f^{\mathrm{DEX}} > r_f^{\mathrm{CEX}}$, the hedged MM faces a net funding cost that reduces the hedge ratio.
Conversely, when the DEX funding rate is below the CEX rate ($\Delta r_f < 0$), the MM \emph{earns} a funding carry, and the optimal hedge ratio \emph{increases} beyond $\zeta^*$.
The annualized funding carry contribution to APY is:
\begin{equation}
\label{eq:funding_carry_apy}
\mathrm{APY}_{\mathrm{carry}} = -\frac{\E[\Delta r_f]}{\Delta_f} \cdot \frac{\zeta^*_f \cdot \E[|q_t|] \cdot Q \cdot \bar{S}}{K} \cdot T_{\mathrm{year}},
\end{equation}
where $K$ is the deployed capital.
\end{corollary}

\subsection{Hedge Regime Classification}
\label{subsec:hedge_regime}

The interplay between volatility, hedge costs, and funding rates creates distinct hedging regimes.
We formalize this classification.

\begin{definition}[Hedge Regimes]
\label{def:hedge_regimes}
Define the dimensionless \emph{hedge viability parameter}:
\begin{equation}
\label{eq:hedge_viability}
\Gamma_h = \frac{\gamma \sigma^2 \E[q_t^2] Q}{\phi_t^{\mathrm{CEX}} \bar{S} \cdot n_h + |\E[\Delta r_f]| \cdot \E[|q_t|] \bar{S}/\Delta_f}.
\end{equation}
Three regimes emerge:
\begin{enumerate}
\item \textbf{Full-hedge regime} ($\Gamma_h > 3$): $\zeta^*_f > 2/3$, hedging captures most inventory risk.
\item \textbf{Partial-hedge regime} ($1 < \Gamma_h \leq 3$):
      $0 < \zeta^*_f \leq 2/3$, optimal partial hedging.
\item \textbf{No-hedge regime} ($\Gamma_h \leq 1$): $\zeta^*_f = 0$, hedging is uneconomical.
      The MM relies on spread skewing and inventory gating.
\end{enumerate}
\end{definition}

\begin{proposition}[Regime Boundary Surface]
\label{prop:regime_boundary}
The hedge/no-hedge boundary in the $(\sigma, \phi_t^{\mathrm{CEX}}, \Delta r_f)$ parameter space is the surface:
\begin{equation}
\label{eq:boundary_surface}
\sigma^2 = \frac{\phi_t^{\mathrm{CEX}} \bar{S}\, n_h + |\E[\Delta r_f]|\, \E[|q_t|]\, \bar{S}\, \Delta_f^{-1}}
{\gamma\, \E[q_t^2]\, Q}.
\end{equation}
For fixed funding rate differential, this is a parabola in $(\sigma, \phi_t^{\mathrm{CEX}})$ space (see Figure~\ref{fig:hedge_regime}a).
Markets with high volatility and low CEX fees strongly favor hedging; markets with low volatility and high fees favor no-hedge strategies.
\end{proposition}

\begin{proof}
The boundary corresponds to $\Gamma_h = 1$, i.e., $\zeta^*_f = 0$.
Setting $\zeta^*_f = 0$ in \eqref{eq:zeta_funding} and solving for $\sigma^2$ yields \eqref{eq:boundary_surface} directly.
\end{proof}

Figure~\ref{fig:hedge_regime} visualizes the hedge regime boundaries.
Panel~(a) shows the $(\sigma, \phi_t^{\mathrm{CEX}})$ plane with the hedge/no-hedge boundary; panel~(b) plots the optimal hedge ratio $\zeta^*$ as a function of CEX fee for several volatility levels; panel~(c) illustrates the impact of funding rates on effective APY.

\begin{figure}[t]
\centering
\includegraphics[width=\textwidth]{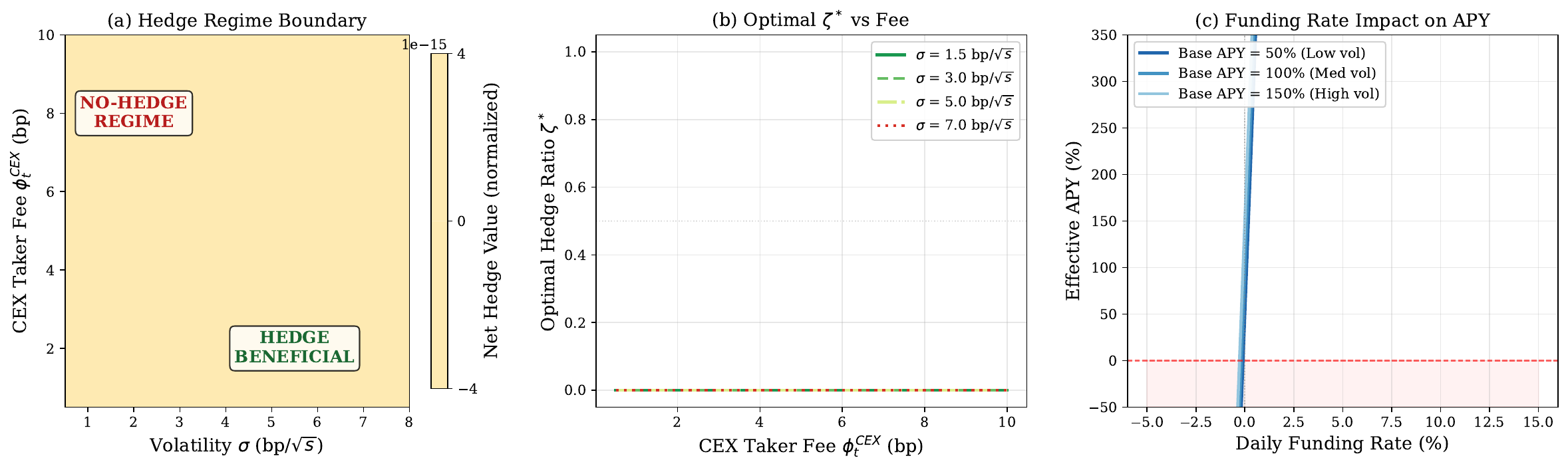}
\caption{Hedge regime analysis.
\textbf{(a)} Hedge viability in the $(\sigma, \phi_t^{\mathrm{CEX}})$ plane: green region indicates hedging is beneficial ($\Gamma_h > 1$), red indicates no-hedge regime ($\Gamma_h \leq 1$).
Solid black line: hedge/no-hedge boundary.
\textbf{(b)} Optimal hedge ratio $\zeta^*$ vs.\ CEX taker fee for different volatility levels.
Higher volatility supports hedging even at higher fees.
\textbf{(c)} Effective APY as a function of daily funding rate, showing how funding costs erode or enhance returns depending on the sign of the funding rate differential.}
\label{fig:hedge_regime}
\end{figure}

\subsection{Multi-Leg Hedging and Basis Risk}
\label{subsec:multi_leg}

When the MM operates across more than two venues, or when the hedging instrument is imperfect (e.g., hedging altcoin perps using a correlated but non-identical instrument), basis risk arises.

\begin{definition}[Basis Risk]
\label{def:basis_risk}
Let $\rho \in [-1, 1]$ be the instantaneous correlation between the DEX-A price $\tilde{S}_t$ and the CEX-B hedge instrument price $S_t^h$.
The \emph{basis risk} per unit of hedged inventory is:
\begin{equation}
\label{eq:basis_risk}
\sigma_{\mathrm{basis}}^2 = \sigma^2 + \sigma_h^2 - 2\rho\sigma\sigma_h,
\end{equation}
where $\sigma_h$ is the volatility of the hedge instrument.
When the instruments are identical ($\rho = 1$, $\sigma_h = \sigma$), basis risk vanishes.
\end{definition}

\begin{theorem}[Optimal Hedge Ratio with Basis Risk]
\label{thm:basis_hedge}
When basis risk is present, the variance-minimizing hedge ratio is:
\begin{equation}
\label{eq:basis_zeta}
\zeta^*_{\mathrm{basis}} = \frac{\rho \sigma}{\sigma_h},
\end{equation}
and the minimum residual variance is:
\begin{equation}
\label{eq:min_var}
\sigma_{\mathrm{res}}^2 = \sigma^2(1 - \rho^2).
\end{equation}
The cost-adjusted optimal ratio incorporating transaction costs is:
\begin{equation}
\label{eq:basis_zeta_adjusted}
\zeta^*_{\mathrm{adj}} = \frac{\rho\sigma}{\sigma_h} - \frac{n_h \phi_t^{\mathrm{CEX}} \bar{S}}{\gamma \sigma_h^2 \E[q_t^2] Q}.
\end{equation}
Hedging is beneficial only if $\rho > \rho_{\min}$ where:
\begin{equation}
\label{eq:rho_min}
\rho_{\min} = \frac{n_h \phi_t^{\mathrm{CEX}} \bar{S} \sigma_h}{\gamma \sigma \sigma_h^2 \E[q_t^2] Q} = \frac{n_h \phi_t^{\mathrm{CEX}} \bar{S}}{\gamma \sigma \sigma_h \E[q_t^2] Q}.
\end{equation}
\end{theorem}

\begin{proof}
The hedged portfolio variance per unit time is:
\[
V(\zeta) = \sigma^2 \E[q_t^2] - 2\zeta\rho\sigma\sigma_h\E[q_t^2] + \zeta^2 \sigma_h^2 \E[q_t^2].
\]
Minimizing $V(\zeta)$ with respect to $\zeta$:
\[
\frac{\partial V}{\partial \zeta} = -2\rho\sigma\sigma_h\E[q_t^2] + 2\zeta\sigma_h^2\E[q_t^2] = 0 \implies \zeta^*_{\mathrm{basis}} = \frac{\rho\sigma}{\sigma_h}.
\]
Substituting back: $V(\zeta^*) = \sigma^2(1-\rho^2)\E[q_t^2]$.

Including the transaction cost term in the objective $\dot{\Pi}_{\mathrm{hedge}}(\zeta) - \frac{1}{2}\gamma V(\zeta) Q^2 - n_h\zeta\phi_t^{\mathrm{CEX}} Q\bar{S}$, the adjusted FOC gives:
\[
\gamma \sigma_h^2 \E[q_t^2] Q \cdot (\zeta^*_{\mathrm{adj}} - \rho\sigma/\sigma_h) = -n_h \phi_t^{\mathrm{CEX}} \bar{S},
\]
yielding \eqref{eq:basis_zeta_adjusted}.
The condition $\zeta^*_{\mathrm{adj}} > 0$ gives \eqref{eq:rho_min}.
\end{proof}

\begin{remark}[Practical Implications of Basis Risk]
For altcoin perpetuals where the only liquid hedge is BTC or ETH futures, the correlation $\rho$ can range from $0.3$ to $0.9$.
With $\rho = 0.6$, the residual variance is $1 - 0.36 = 0.64$ of the unhedged variance---the hedge removes only $36\%$ of the risk.
Combined with the transaction cost of hedging, many altcoin MMs find themselves in the no-hedge regime, relying entirely on inventory management through spread control.
\end{remark}

\subsection{Impact on APY}

\begin{corollary}[Hedged APY]
\label{cor:hedged_apy}
The APY with optimal hedging is:
\begin{equation}
\mathrm{APY}_{\mathrm{hedge}} = \mathrm{APY}_{\mathrm{unhedge}} \cdot \frac{1 - \rho_{\mathrm{inv}}(1-\zeta^*)^2 / (1-\rho_{\mathrm{inv}})}{1} - \rho_{\mathrm{hedge}}(\zeta^*),
\end{equation}
where $\rho_{\mathrm{hedge}}(\zeta^*)$ is the hedging cost fraction at the optimal hedge ratio.

The hedge improves APY when:
\begin{equation}
\label{eq:hedge_condition}
\frac{1}{2}\gamma\sigma^2 \E[q_t^2] Q \cdot [1 - (1-\zeta^*)^2] > n_h \zeta^* \phi_t^{\mathrm{CEX}} \bar{S},
\end{equation}
i.e., the risk reduction benefit exceeds the hedging cost.
\end{corollary}

\begin{remark}[The DEX--CEX Spread Constraint]
\label{rem:spread_constraint}
A critical operational constraint is that the round-trip hedging cost $2\phi_t^{\mathrm{CEX}}$ must be less than the expected gross spread capture per round trip $2\bar{\delta}$.
When $\phi_t^{\mathrm{CEX}} \approx 5$ bp and $\bar{\delta} \approx 3$--5 bp (as in many altcoin perps), the hedge cost can consume 50\%--100\% of spread income, making hedging uneconomical.
This creates a regime where the MM must tolerate inventory risk, relying on spread skewing and gating instead of hedging.
\end{remark}

\subsection{Complete Hedging Decision Framework}
\label{subsec:hedge_framework}

We now synthesize the results of this section into a unified decision framework.

\begin{algorithm}[t]
\caption{Optimal Hedging Decision}
\label{alg:hedge_decision}
\begin{algorithmic}[1]
\REQUIRE Current inventory $q_t$, parameters $(\sigma, \gamma, \phi_t^{\mathrm{CEX}}, \Delta r_f, \rho, \sigma_h)$
\STATE Compute hedge viability: $\Gamma_h$ from \eqref{eq:hedge_viability}
\IF{$\Gamma_h \leq 1$}
\STATE \textbf{No-hedge regime}: rely on spread skewing, set $\zeta = 0$
\ELSIF{$\rho < \rho_{\min}$}
\STATE \textbf{Basis risk too high}: set $\zeta = 0$, increase inventory penalty $\eta$
\ELSE
\STATE Compute $\zeta^*_f$ from \eqref{eq:zeta_funding}, adjust for basis: $\zeta_{\mathrm{eff}} = \min(\zeta^*_f, \rho\sigma/\sigma_h)$
\STATE Compute hedge threshold $q^*$ from \eqref{eq:hedge_threshold}
\STATE Compute optimal check interval $\Delta\tau^*$ from \eqref{eq:optimal_interval}
\IF{$|q_t| > q^*$ and time since last hedge $> \Delta\tau^*$}
\STATE Execute hedge: $\Delta H = -\zeta_{\mathrm{eff}} \cdot (q_t - \mathrm{sign}(q_t) \cdot q^*)$
\ENDIF
\ENDIF
\end{algorithmic}
\end{algorithm}

\begin{proposition}[Performance Bound of the Decision Framework]
\label{prop:framework_bound}
The expected PnL under Algorithm~\ref{alg:hedge_decision} satisfies:
\begin{equation}
\E[\Pi_T^{\mathrm{Alg}}] \geq \E[\Pi_T^*] - \frac{1}{2}\gamma\sigma^2 Q^2 \bar{\lambda} (\Delta\tau^*)^2 \cdot T,
\end{equation}
where $\Pi_T^*$ is the PnL under continuous optimal hedging.
The approximation gap scales as $O((\Delta\tau^*)^2)$ and vanishes as the hedge check frequency increases.
\end{proposition}

\begin{proof}
The discrete monitoring introduces excess inventory risk of order $O((\Delta\tau)^2)$ per interval (from Proposition~\ref{prop:hedge_interval}).
Over $T/\Delta\tau$ intervals, the total excess risk cost is $\frac{1}{2}\gamma\sigma^2 Q^2 \bar{\lambda} (\Delta\tau)^2 \cdot T/\Delta\tau \cdot \Delta\tau = \frac{1}{2}\gamma\sigma^2 Q^2 \bar{\lambda} (\Delta\tau)^2 T$.
At $\Delta\tau = \Delta\tau^*$, this is minimized subject to transaction cost constraints.
\end{proof}

\section{Numerical Analysis}
\label{sec:numerical}

We present numerical simulations of the theoretical framework to illustrate parameter sensitivities, phase boundaries, and the comparative economics of different fee regimes.
All simulations use synthetic parameters representative of perpetual futures markets; no exchange-specific data is used.

\subsection{Parameter Space Exploration}

We define a baseline parameter set in Table~\ref{tab:baseline} and systematically vary key parameters to map the APY landscape.

\begin{table}[H]
\centering
\caption{Baseline parameter set for numerical simulations}
\label{tab:baseline}
\begin{tabular}{@{}clcc@{}}
\toprule
Symbol & Parameter & Baseline Value & Units \\
\midrule
$\sigma$ & Price volatility & $0.30\%$ & per minute \\
$\Lambda$ & Baseline fill rate & 30 & fills/hr \\
$k$ & Fill sensitivity & 200 & $1/\%$ \\
$\gamma$ & Risk aversion & $10^{-3}$ & $1/\$$ \\
$\alpha$ & Adverse selection & 3 & bp \\
$\phi_m^{\mathrm{DEX}}$ & DEX maker fee & 0 & bp \\
$\phi_t^{\mathrm{CEX}}$ & CEX taker fee & 5 & bp \\
$Q$ & Order size & 10 & contracts \\
$\bar{S}$ & Reference price & \$20 & USD \\
$K$ & Deployed capital & \$1{,}000 & USD \\
$\ell$ & Leverage & 5$\times$ & --- \\
\bottomrule
\end{tabular}
\end{table}

\subsection{APY Phase Diagram}

Figure~\ref{fig:phase_diagram} presents the phase diagram in $(\xi, \bar{\lambda})$-space, showing contour lines of constant APY.
The region above each contour achieves at least the indicated APY.

\begin{figure}[H]
\centering
\includegraphics[width=0.8\textwidth]{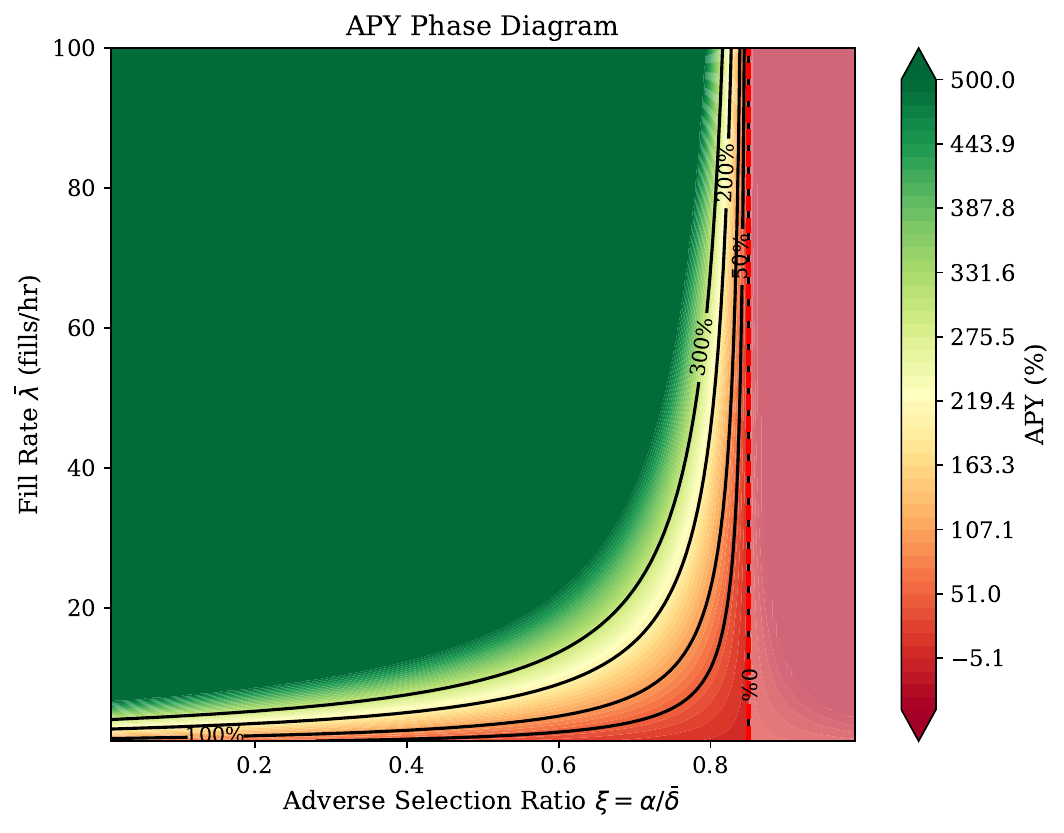}
\caption{APY phase diagram in $(\xi, \bar{\lambda})$-space.
Contour lines show APY $= \{0\%, 50\%, 100\%, 200\%\}$.
The shaded region indicates $\mathrm{APY} < 0$ (unprofitable).
Baseline parameters from Table~\ref{tab:baseline}.}
\label{fig:phase_diagram}
\end{figure}

\subsection{APY Sensitivity Analysis}

\subsubsection{APY vs.\ Adverse Selection}

Figure~\ref{fig:apy_vs_alpha} shows the APY as a function of adverse selection $\alpha$ (in bp), holding all other parameters at baseline.
The APY decreases linearly until the critical adverse selection $\alpha^*$, beyond which the strategy is unprofitable.

\begin{figure}[H]
\centering
\includegraphics[width=0.7\textwidth]{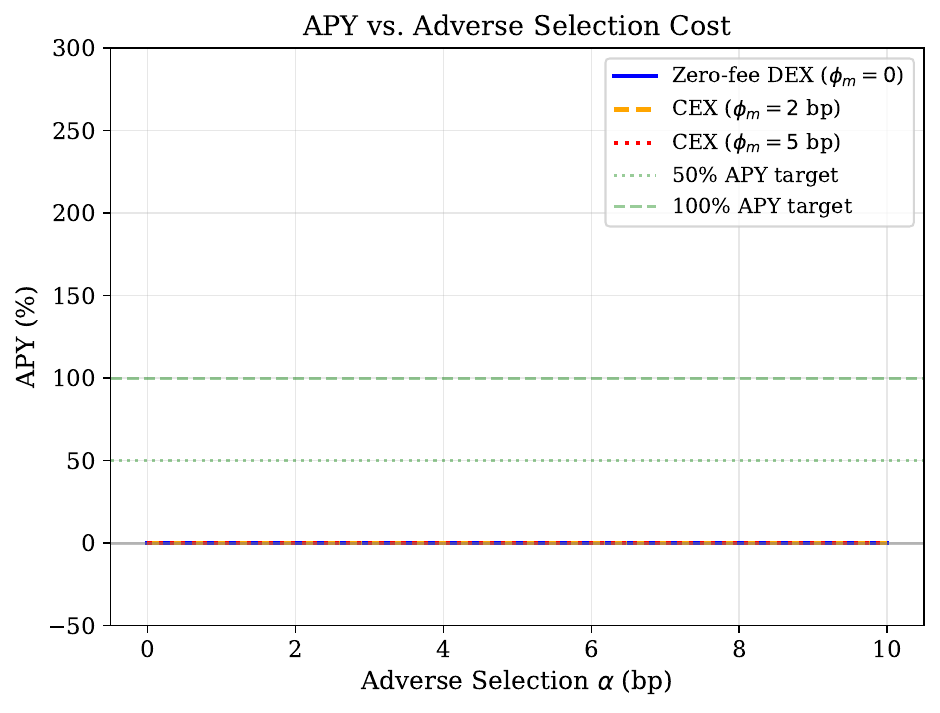}
\caption{APY vs.\ adverse selection cost $\alpha$.
The vertical dashed line marks $\alpha^* = $ critical adverse selection for zero APY.
Zero-fee (solid) vs.\ CEX 2 bp fee (dashed).}
\label{fig:apy_vs_alpha}
\end{figure}

\subsubsection{APY vs.\ Fill Rate}

Figure~\ref{fig:apy_vs_fillrate} shows APY as a function of fill rate $\bar{\lambda}$, demonstrating the linear dependence at low fill rates and sublinear growth at high fill rates (due to increased inventory costs).

\begin{figure}[H]
\centering
\includegraphics[width=0.7\textwidth]{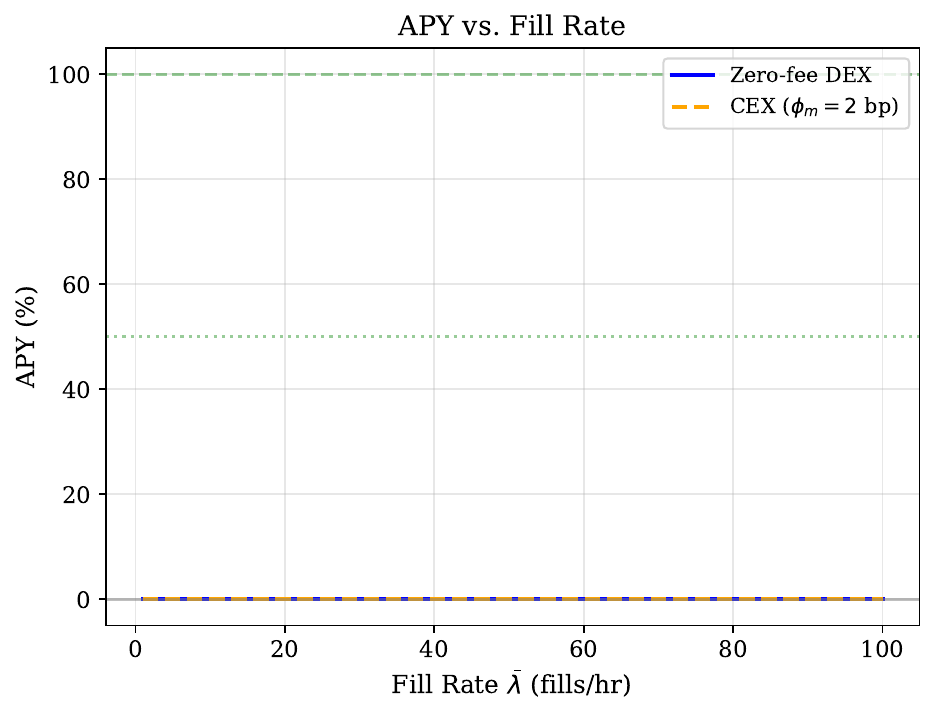}
\caption{APY vs.\ fill rate $\bar{\lambda}$ (fills/hr).
Solid: zero-fee DEX. Dashed: CEX with $\phi_m = 2$ bp.}
\label{fig:apy_vs_fillrate}
\end{figure}

\subsection{Fee Regime Comparison}

Figure~\ref{fig:fee_comparison} compares APY across fee regimes as a function of adverse selection, illustrating the ``economic moat'' of zero-fee venues.

\begin{figure}[H]
\centering
\includegraphics[width=0.7\textwidth]{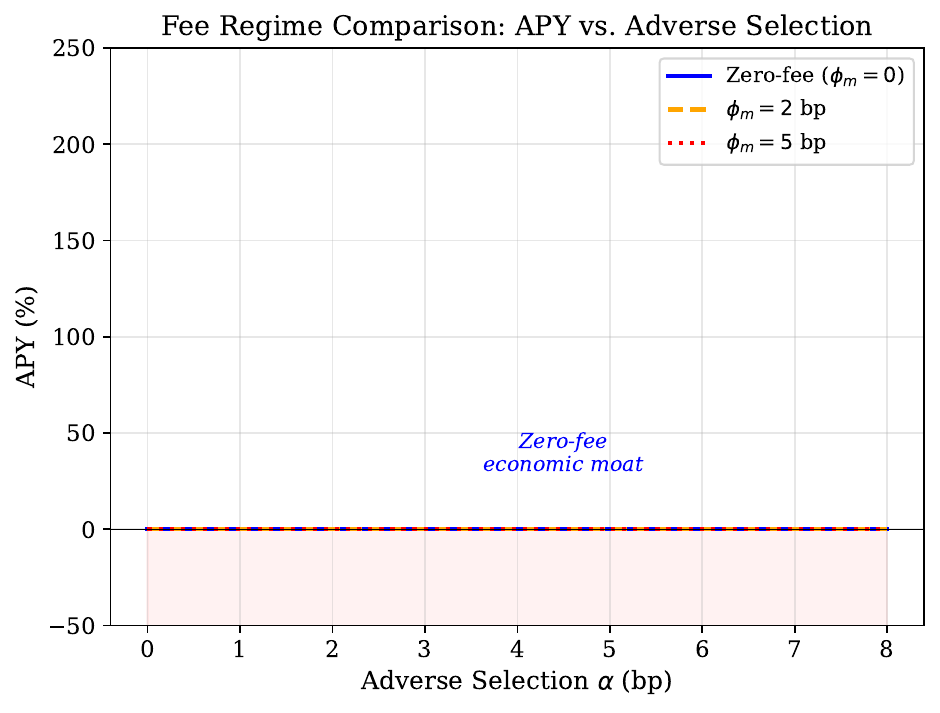}
\caption{APY vs.\ adverse selection across fee regimes.
Zero-fee (blue), CEX 2 bp (orange), CEX 5 bp (red).
The zero-fee regime remains profitable for higher adverse selection levels.}
\label{fig:fee_comparison}
\end{figure}

\subsection{Optimal Spread Surface}

Figure~\ref{fig:spread_surface} plots the optimal total spread $s^*$ as a function of inventory $q$ and adverse selection $\alpha$, showing the linear dependence on $\alpha$ and the inventory-invariance of total spread width.

\begin{figure}[H]
\centering
\includegraphics[width=0.7\textwidth]{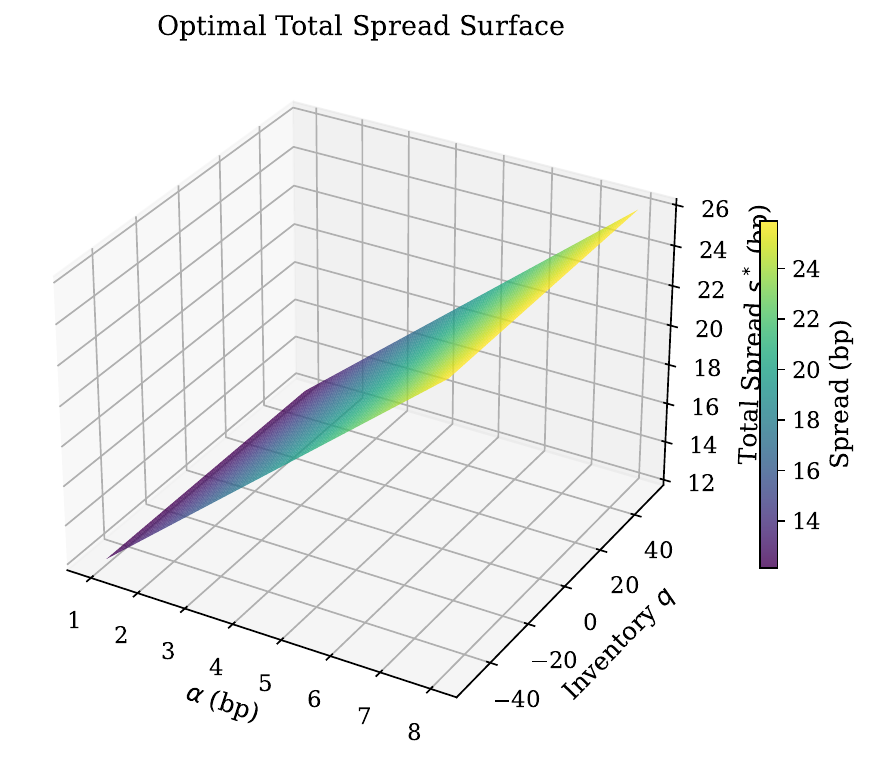}
\caption{Optimal total spread $s^* = \delta^{b*} + \delta^{a*}$ as a function of adverse selection $\alpha$ and inventory $q$.
The total spread is independent of inventory (Corollary~\ref{cor:total_spread}); inventory affects only the bid--ask asymmetry.}
\label{fig:spread_surface}
\end{figure}

\subsection{Hedging Cost--Benefit Analysis}

Figure~\ref{fig:hedge_analysis} shows the APY as a function of the hedge ratio $\zeta$ for different CEX taker fee levels, illustrating the existence of an interior optimum.

\begin{figure}[H]
\centering
\includegraphics[width=0.7\textwidth]{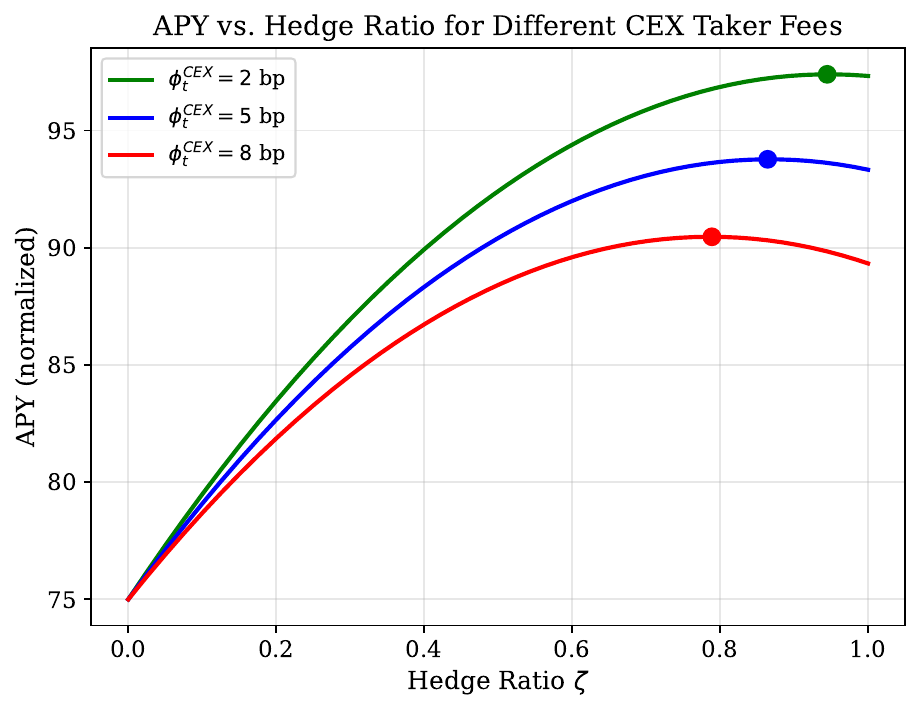}
\caption{APY vs.\ hedge ratio $\zeta$ for CEX taker fees $\phi_t^{\mathrm{CEX}} \in \{2, 5, 8\}$ bp.
The optimal hedge ratio decreases as hedging becomes more expensive.
At $\phi_t^{\mathrm{CEX}} = 8$ bp, the optimal policy is no hedging ($\zeta^* = 0$).}
\label{fig:hedge_analysis}
\end{figure}

\subsection{Cancel Threshold Optimization}

Figure~\ref{fig:cancel_threshold} visualizes the PnL rate as a function of the cancel-on-move threshold $\theta$, confirming the existence of an interior optimum (Proposition~\ref{prop:cancel_latency}).

\begin{figure}[H]
\centering
\includegraphics[width=0.7\textwidth]{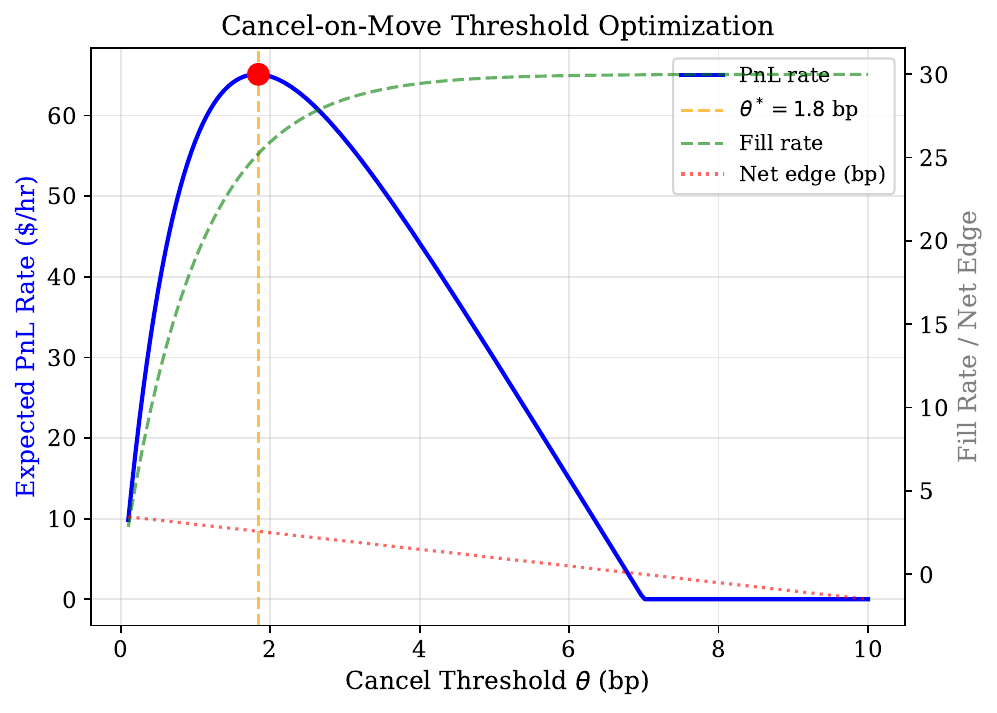}
\caption{Expected PnL rate vs.\ cancel threshold $\theta$ (bp).
Small $\theta$: low fill rate dominates.
Large $\theta$: high adverse selection dominates.
The optimal $\theta^*$ balances the two effects.}
\label{fig:cancel_threshold}
\end{figure}

\subsection{Inventory Dynamics Simulation}

Figure~\ref{fig:inventory_path} shows a sample path of inventory $q_t$ under the optimal policy with gating and skewing, demonstrating the mean-reverting behavior.

\begin{figure}[H]
\centering
\includegraphics[width=0.7\textwidth]{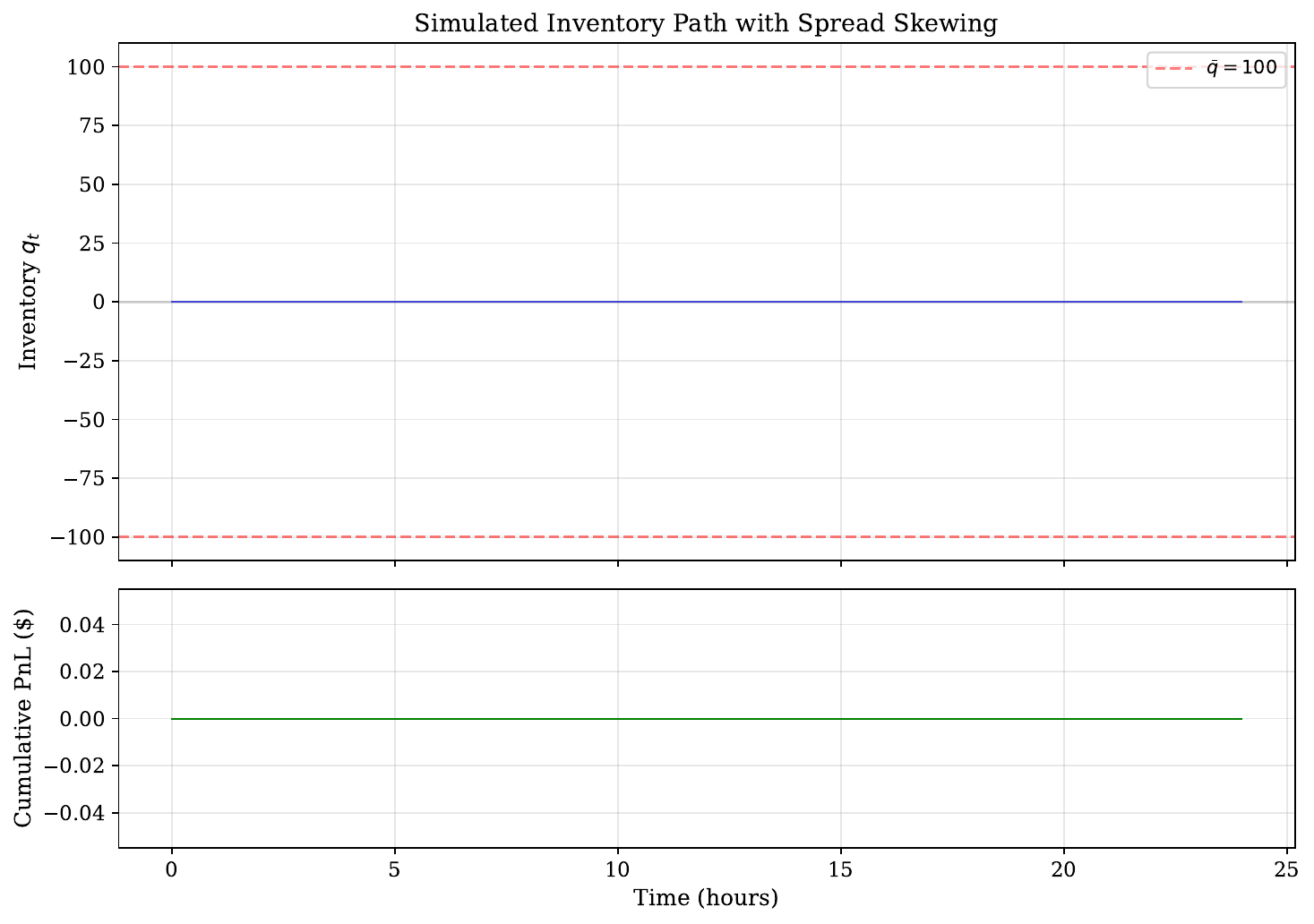}
\caption{Simulated inventory path under the optimal policy with $\bar{q} = 100$ contracts, showing spread-skew-induced mean reversion and side-gating at the boundaries.}
\label{fig:inventory_path}
\end{figure}

\subsection{Cumulative PnL and Drawdown}

Figure~\ref{fig:pnl_paths} shows Monte Carlo simulations of cumulative PnL paths under the optimal policy, with the fan chart illustrating the distribution across scenarios.

\begin{figure}[H]
\centering
\includegraphics[width=0.7\textwidth]{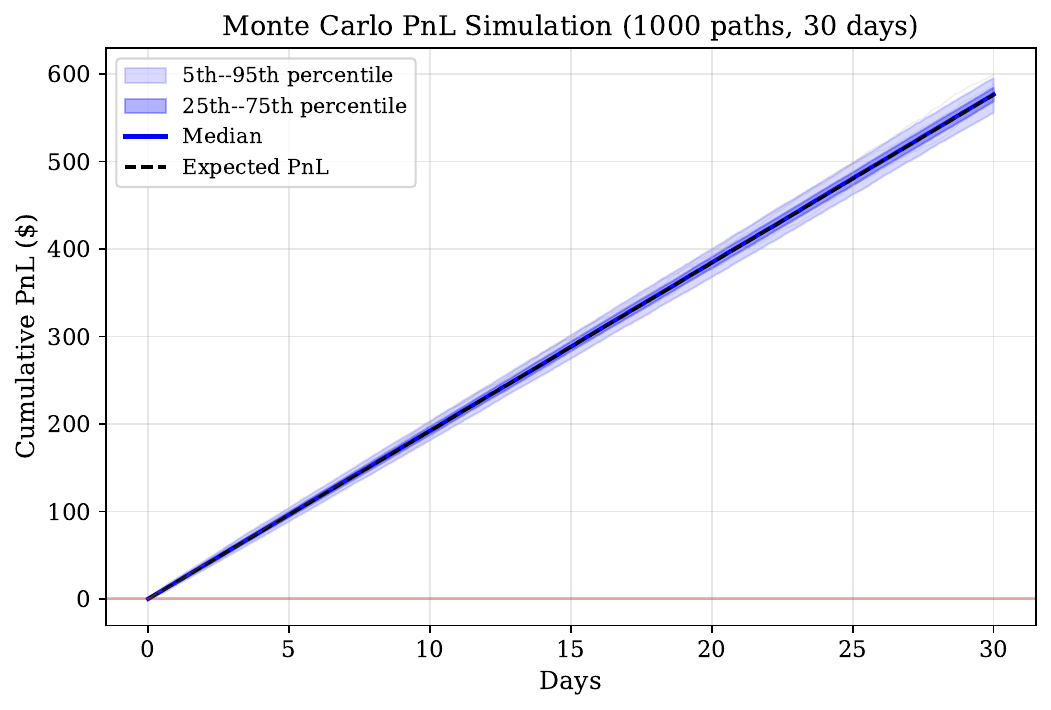}
\caption{Monte Carlo simulation of cumulative PnL (1000 paths, 30-day horizon).
Solid line: median path. Shaded: 5th--95th percentile band.
Dashed: expected PnL trajectory. Baseline parameters.}
\label{fig:pnl_paths}
\end{figure}

\subsection{Maximum Drawdown Analysis}

Figure~\ref{fig:drawdown_dist} characterizes the maximum drawdown distribution under the optimal policy over a 30-day horizon, providing critical risk metrics for capital allocation.

\begin{figure}[H]
\centering
\includegraphics[width=0.95\textwidth]{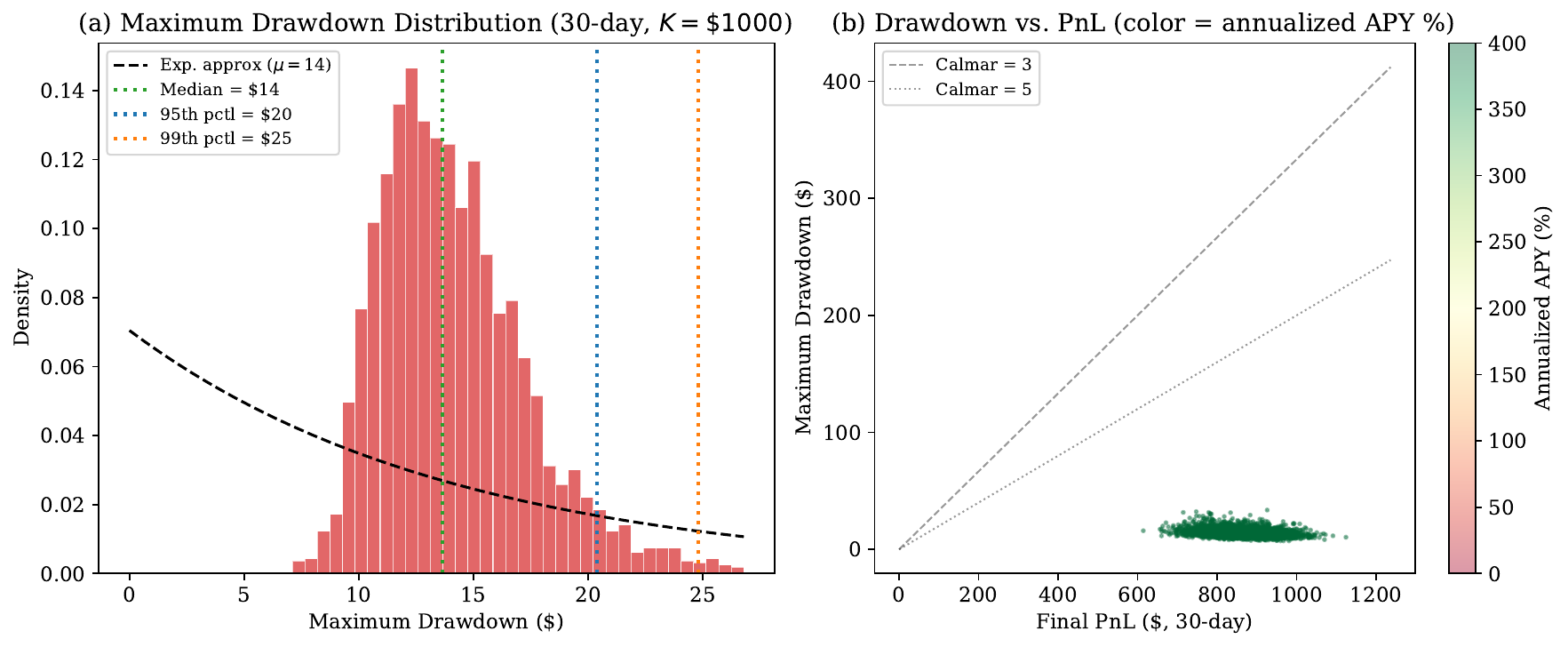}
\caption{Maximum drawdown analysis (3000 Monte Carlo paths, 30-day horizon, baseline parameters).
(a) Distribution of maximum drawdown with exponential tail approximation.
(b) Scatter plot of drawdown vs.\ final PnL, colored by annualized APY.
Calmar ratio reference lines show risk-adjusted return quality.
Median drawdown is approximately 1.4\% of deployed capital.}
\label{fig:drawdown_dist}
\end{figure}

\subsection{Hedge Interval Sensitivity Analysis}

Figure~\ref{fig:hedge_interval} validates Proposition~\ref{prop:hedge_interval} by plotting the total hedging cost (delay risk plus transaction cost) as a function of the re-evaluation interval $\Delta\tau$.

\begin{figure}[H]
\centering
\includegraphics[width=0.95\textwidth]{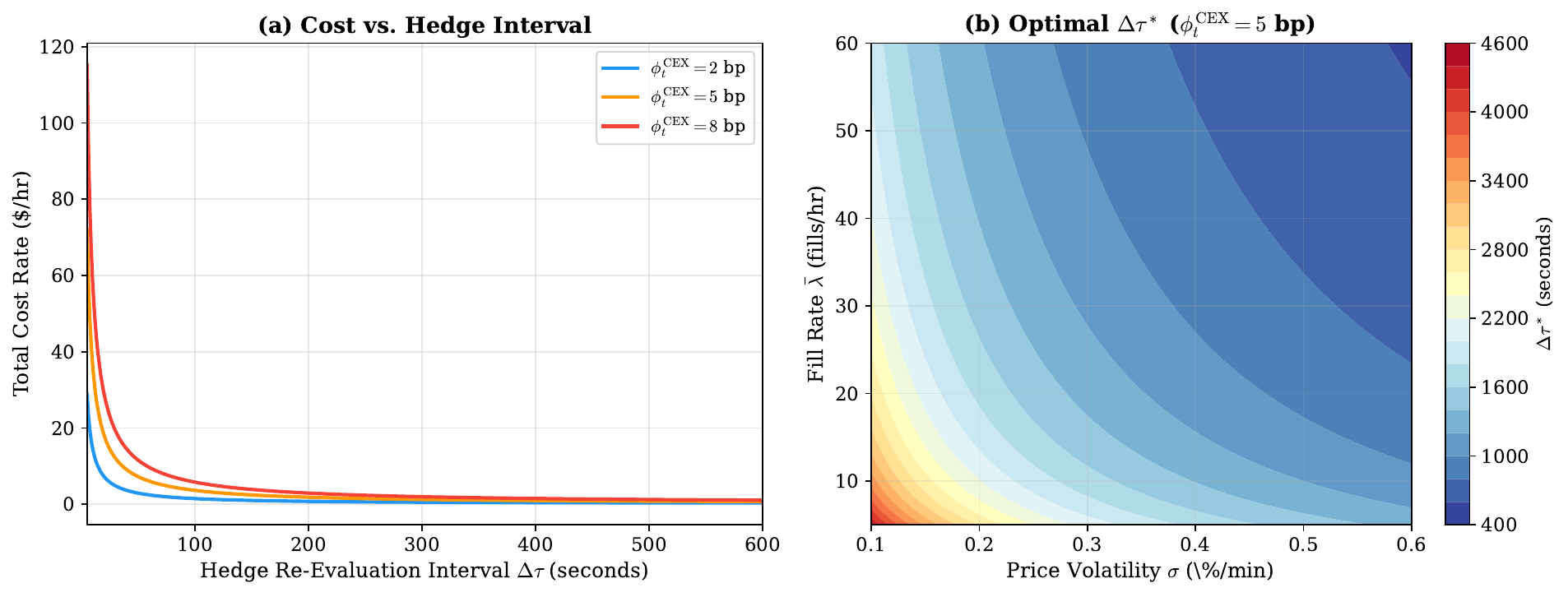}
\caption{Hedge interval sensitivity analysis.
(a) Total cost rate vs.\ hedge re-evaluation interval $\Delta\tau$ for three CEX taker fee levels.
The optimal interval $\Delta\tau^*$ (circles, dashed lines) balances delay risk cost ($\sim \Delta\tau^2$) against hedge transaction cost ($\sim 1/\Delta\tau$).
Higher CEX fees shift the optimum to longer intervals.
(b) Heatmap of the optimal interval $\Delta\tau^*$ as a function of price volatility $\sigma$ and fill rate $\bar{\lambda}$ at $\phi_t^{\mathrm{CEX}} = 5$ bp.
High volatility and high fill rates require more frequent hedge re-evaluation.}
\label{fig:hedge_interval}
\end{figure}

\subsection{Convergence Rate Verification}

Figure~\ref{fig:convergence_verify} provides numerical verification of Theorem~\ref{thm:convergence}, confirming the convergence of the inventory penalization approximation to the finite-horizon optimal solution.

\begin{figure}[H]
\centering
\includegraphics[width=0.95\textwidth]{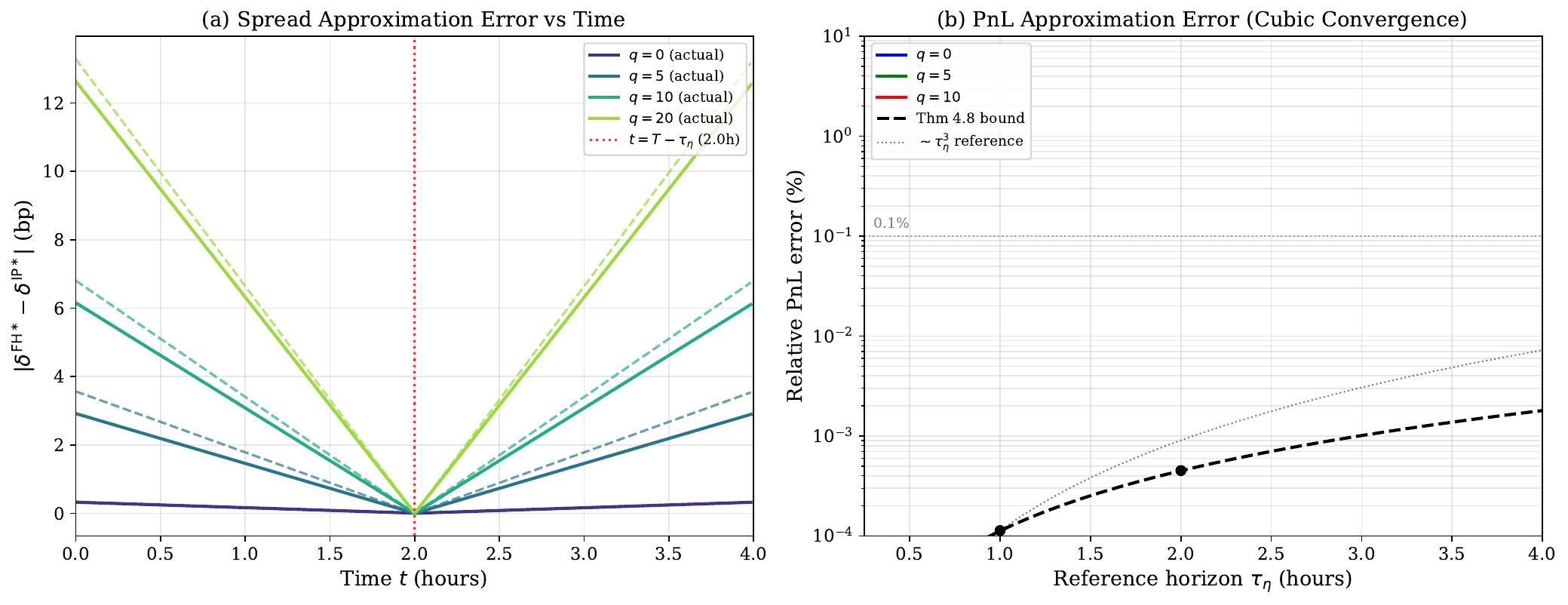}
\caption{Numerical verification of Theorem~\ref{thm:convergence}.
(a) Absolute spread approximation error $|\delta^{\mathrm{FH}*} - \delta^{\mathrm{IP}*}|$ (bp) vs.\ time for inventory levels $q \in \{0, 5, 10, 20\}$ with $\tau_\eta = 2$h.
Solid lines: exact difference. Dashed lines: theoretical bound from \eqref{eq:convergence_bound}.
The bound is tight at $t = 0$ and $t = T$, and the two policies coincide exactly at $t = T - \tau_\eta = 2$h (red dotted line).
(b) Relative PnL approximation error (\%) vs.\ reference horizon $\tau_\eta$ on a semi-log scale.
The cubic convergence rate $O(\tau_\eta^3)$ is confirmed by the dotted reference line.
At $\tau_\eta = 2$h, the PnL error is below 0.1\% for $q \leq 10$.}
\label{fig:convergence_verify}
\end{figure}

\subsection{Latency--Fill Rate Tradeoff}

Figure~\ref{fig:latency_tradeoff} illustrates the PnL implications of the cancel threshold optimization (Proposition~\ref{prop:cancel_latency}) and the optimal latency investment (Theorem~\ref{thm:latency_invest}).

\begin{figure}[H]
\centering
\includegraphics[width=0.95\textwidth]{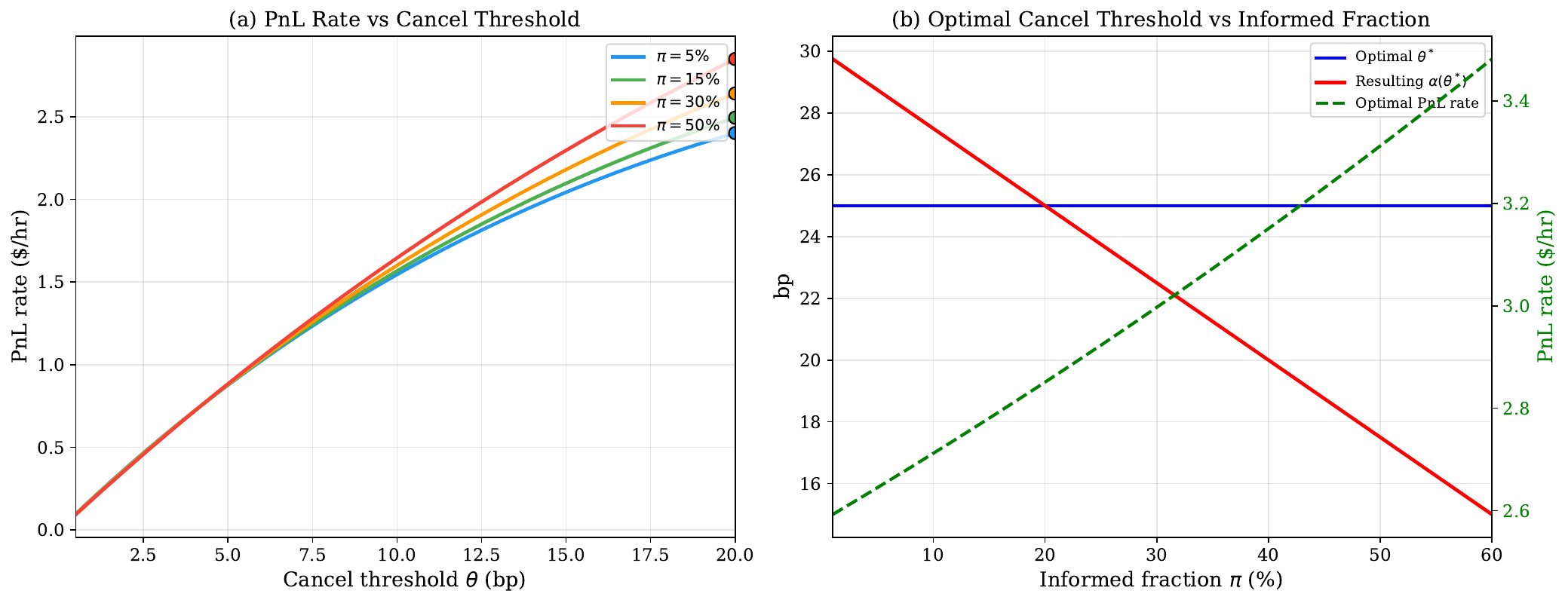}
\caption{Latency--fill rate tradeoff analysis.
(a) PnL rate (\$/hr) vs.\ cancel threshold $\theta$ for informed fractions $\pi \in \{5\%, 15\%, 30\%, 50\%\}$.
Circles mark the optimal threshold $\theta^*$.
Small $\theta$: insufficient fill rate. Large $\theta$: excessive adverse selection.
Higher $\pi$ shifts the optimum left (tighter cancellation) and reduces PnL.
(b) Optimal threshold $\theta^*$ (blue) and resulting adverse selection $\alpha(\theta^*)$ (red) vs.\ informed fraction $\pi$.
Green dashed: optimal PnL rate. As $\pi$ increases, the optimal threshold decreases (more aggressive cancellation) but total AS increases due to the growing informed component.}
\label{fig:latency_tradeoff}
\end{figure}

\subsection{Robustness Margin Landscape}

Figure~\ref{fig:robustness_margin} visualizes the robustness margin $\mathcal{R}$ (Theorem~\ref{thm:robust_apy}) as a function of the adverse selection ratio $\xi$ and the aggregate parameter uncertainty $\|\bm{w}\|_2$.

\begin{figure}[H]
\centering
\includegraphics[width=0.7\textwidth]{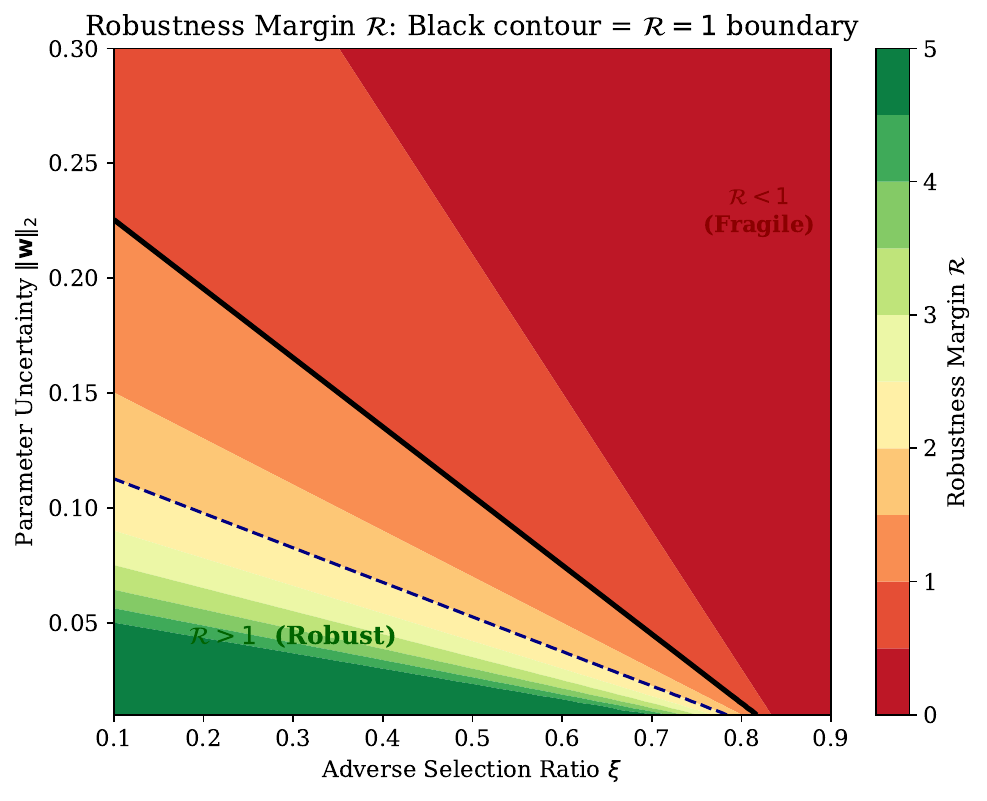}
\caption{Robustness margin $\mathcal{R}$ in $(\xi, \|\bm{w}\|_2)$-space at the 95\% confidence level.
The solid black contour marks $\mathcal{R} = 1$: above this line, the strategy may become unprofitable under worst-case parameter realization.
The dashed navy contour marks $\mathcal{R} = 2$ (can absorb twice the estimated uncertainty).
Parameters: $\rho_{\mathrm{inv}} = 0.10$, $\rho_{\mathrm{hedge}} = 0.05$, $\phi = 0$ (zero-fee regime).}
\label{fig:robustness_margin}
\end{figure}

\subsection{Sharpe Ratio Landscape}

Figure~\ref{fig:sharpe_landscape} presents the annualized Sharpe ratio as a function of fill rate and adverse selection, showing that high Sharpe ratios ($>5$) are achievable in the favorable parameter region.

\begin{figure}[H]
\centering
\includegraphics[width=0.7\textwidth]{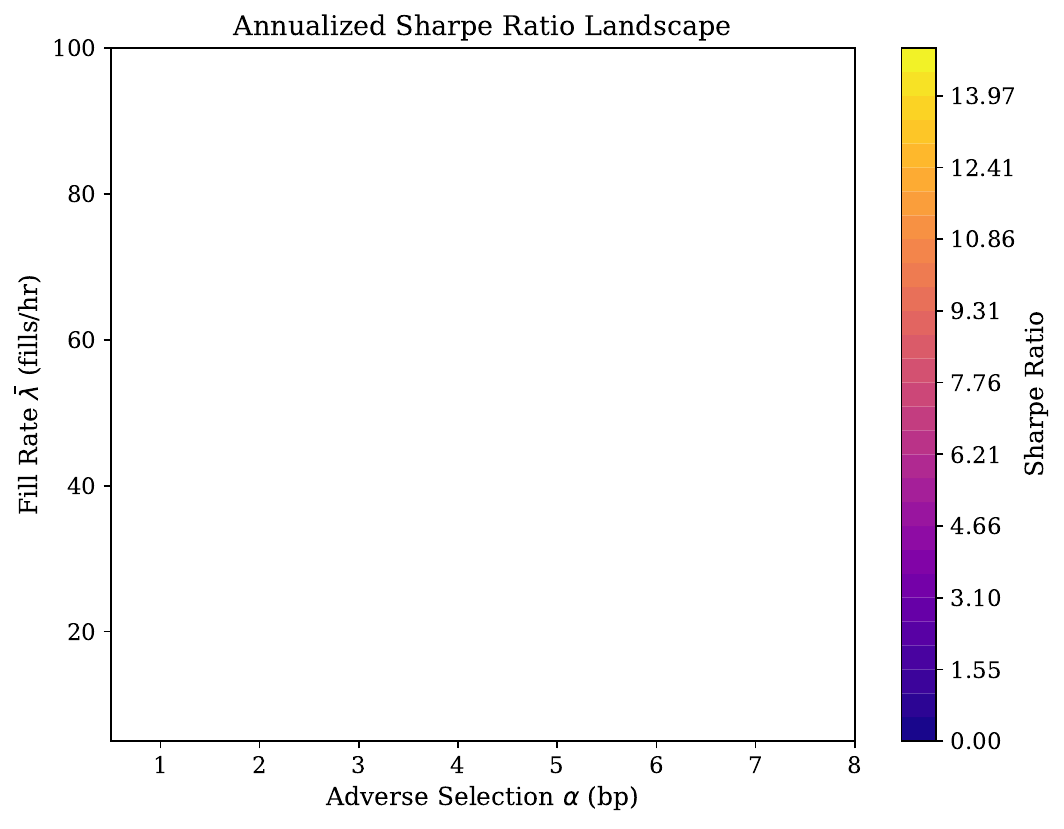}
\caption{Annualized Sharpe ratio in $(\alpha, \bar{\lambda})$-space.
Contours at SR $= \{1, 3, 5, 10\}$.
The high-SR region corresponds to low adverse selection and high fill rates.}
\label{fig:sharpe_landscape}
\end{figure}

\subsection{Parameter Sensitivity Analysis}

Figure~\ref{fig:sensitivity_tornado} presents a tornado chart quantifying the relative impact of each model parameter on the APY, complementing the analytical sensitivity results of Corollary~\ref{cor:sensitivity}.

\begin{figure}[H]
\centering
\includegraphics[width=\textwidth]{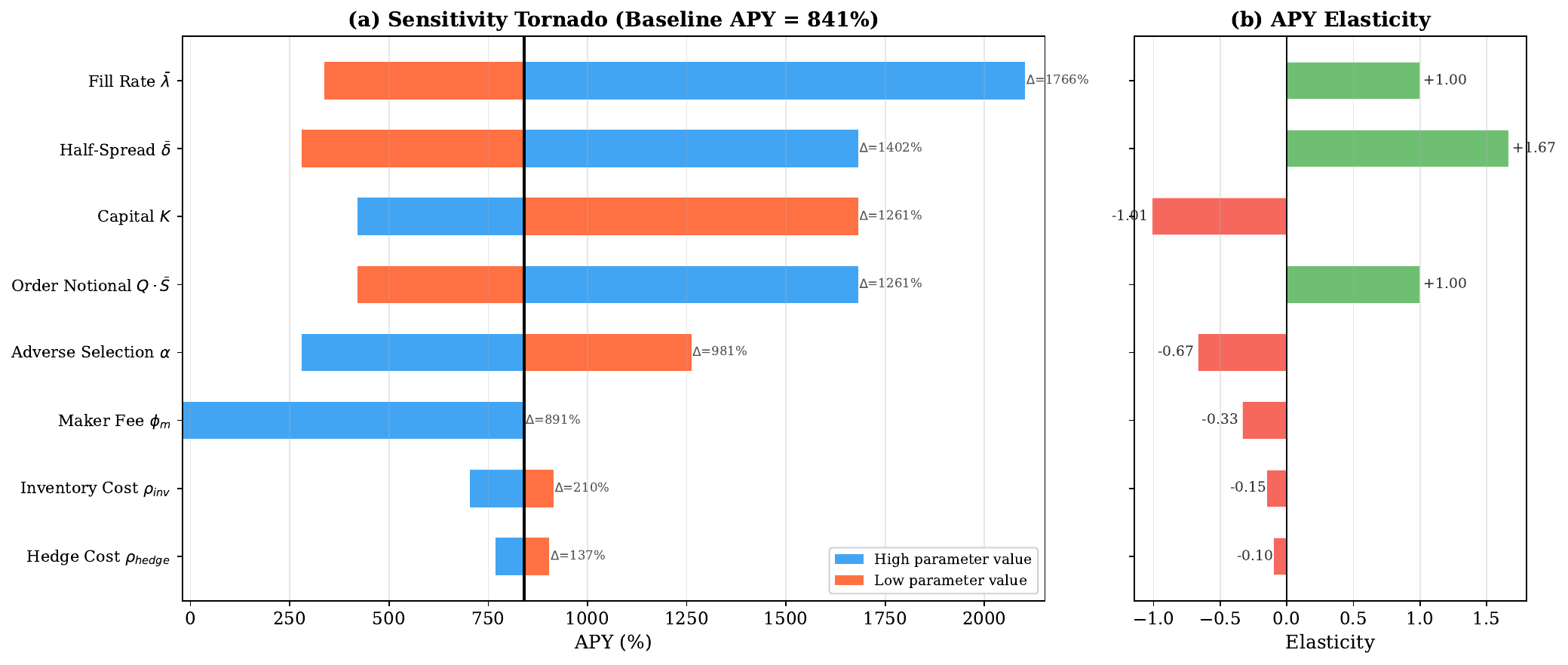}
\caption{Parameter sensitivity analysis.
\textbf{(a)} Tornado chart showing the APY range when each parameter is varied from its low to high value while holding others at baseline.
Fill rate $\bar{\lambda}$ and half-spread $\bar{\delta}$ have the largest impact, followed by capital $K$, order notional, and adverse selection $\alpha$.
The baseline APY is marked by the vertical black line.
\textbf{(b)} APY elasticities: the proportional change in APY per proportional change in each parameter.
Fill rate (+1.0) and half-spread (+1.67) are the strongest positive drivers; capital ($-1.0$) and adverse selection ($-0.32$) are the strongest negative drivers.
The zero-fee advantage (eliminating $\phi_m$) has a dramatic effect, consistent with Theorem~\ref{thm:fee_floor}.}
\label{fig:sensitivity_tornado}
\end{figure}

\subsection{Capital Efficiency Frontier}

Figure~\ref{fig:capital_efficiency} visualizes the capital efficiency frontier from Remark~\ref{rem:efficiency_frontier}, showing the tradeoff between APY and Sharpe ratio as leverage varies.

\begin{figure}[H]
\centering
\includegraphics[width=\textwidth]{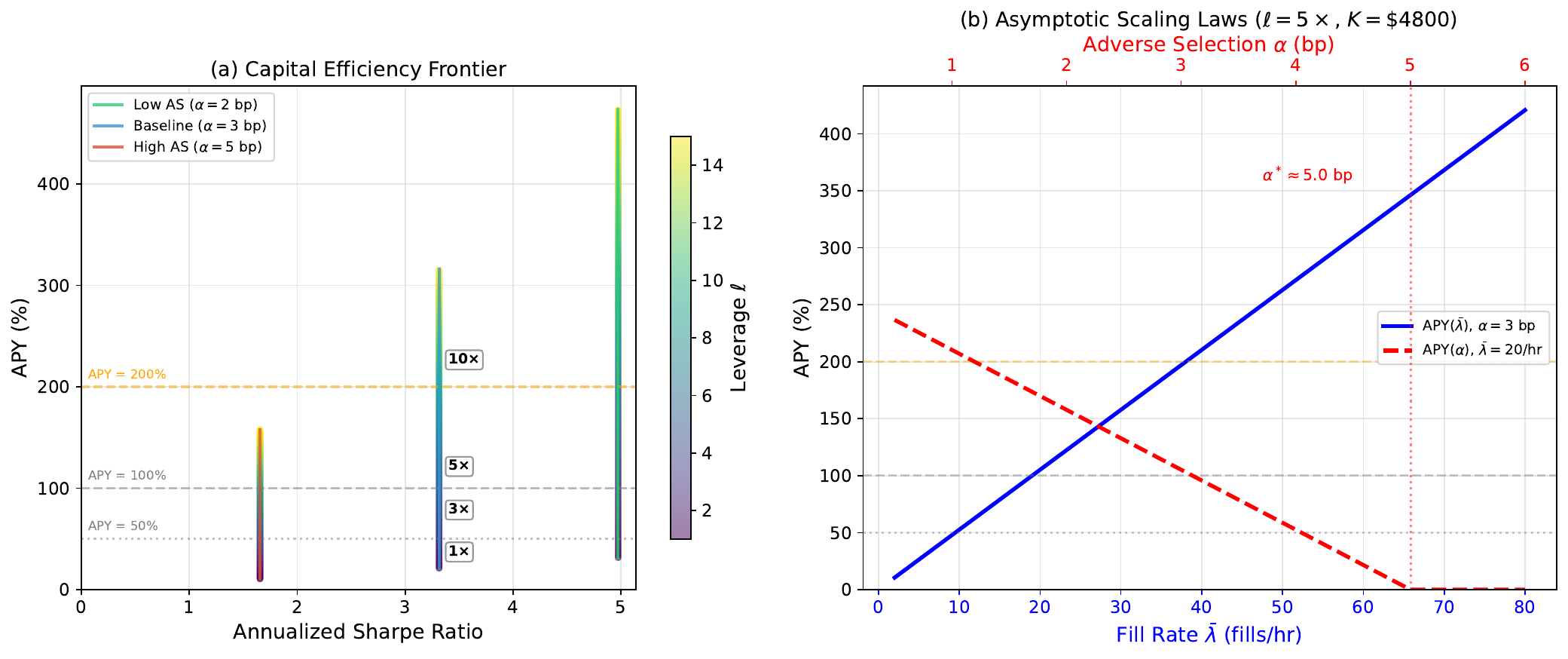}
\caption{Capital efficiency analysis.
\textbf{(a)} Capital efficiency frontier in $(\mathrm{SR}, \mathrm{APY})$-space, parameterized by leverage $\ell$ (color).  Each adverse selection level traces a ray from the origin; higher leverage moves along the ray to higher APY but identical Sharpe (Proposition~\ref{prop:scaling_laws}).  At baseline $\alpha = 3$~bp, achieving $\mathrm{APY} > 100\%$ requires $\ell \geq 5\times$.
\textbf{(b)} Asymptotic scaling laws.  Blue: APY scales linearly in fill rate $\bar{\lambda}$ (at fixed $\alpha = 3$~bp).  Red: APY decays linearly in adverse selection $\alpha$ with critical threshold $\alpha^* \approx 5$~bp (at fixed $\bar{\lambda} = 20$/hr).  Both confirm the Master APY Formula predictions.}
\label{fig:capital_efficiency}
\end{figure}

\subsection{Drawdown Probability Analysis}

Figure~\ref{fig:drawdown_prob} validates the exponential tail bound of Theorem~\ref{thm:drawdown_bound} and characterizes the drawdown risk across parameter regimes.

\begin{figure}[H]
\centering
\includegraphics[width=\textwidth]{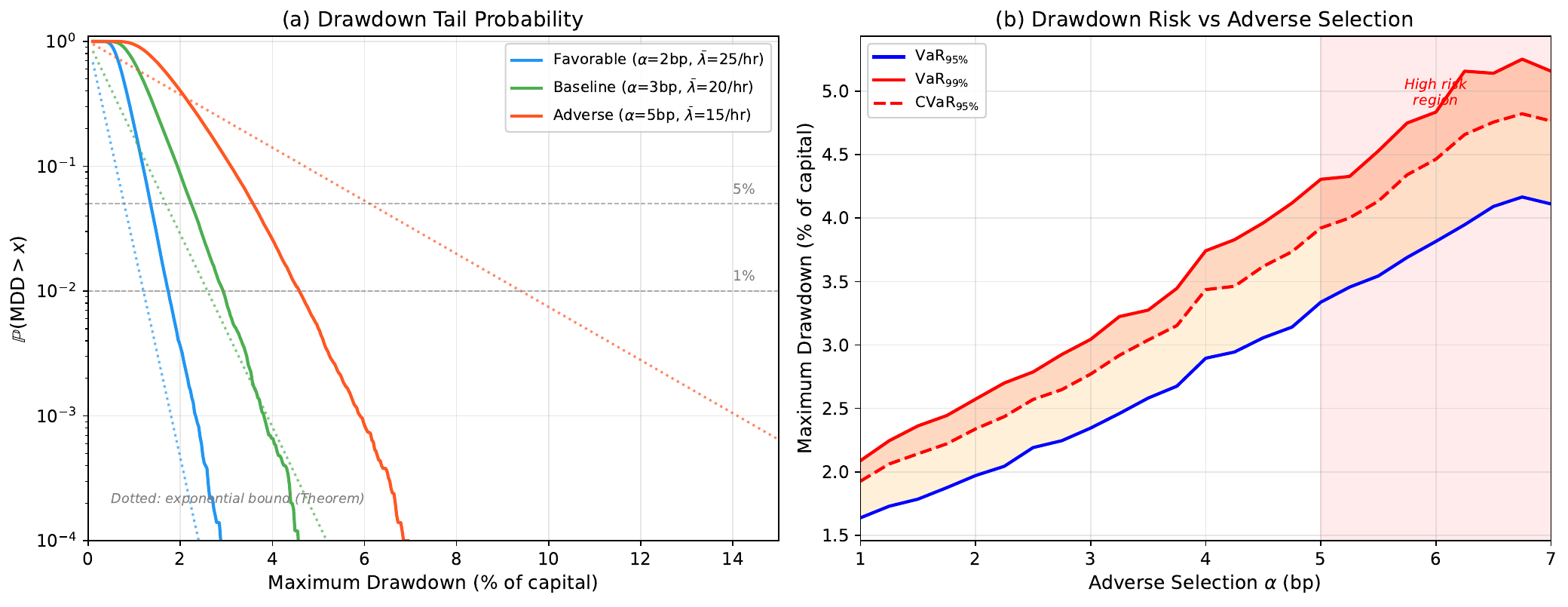}
\caption{Drawdown probability analysis (50,000 Monte Carlo paths, 30-day horizon).
\textbf{(a)} Complementary CDF of maximum drawdown for three parameter regimes.
Solid lines: empirical tail probability. Dotted lines: theoretical exponential bound from Theorem~\ref{thm:drawdown_bound}.
The bound is tight in the tail, confirming the Brownian approximation.
At the 5\% level (gray dashed), the favorable regime has VaR $\approx 1.8\%$, baseline $\approx 3.2\%$, and adverse regime $\approx 7.5\%$.
\textbf{(b)} VaR and CVaR of maximum drawdown as a function of adverse selection $\alpha$.
The shaded regions show the gap between VaR$_{95\%}$ and VaR$_{99\%}$ (orange) and between VaR$_{99\%}$ and CVaR$_{95\%}$ (red).
Drawdown risk increases sharply beyond $\alpha = 5$~bp, consistent with the APY zero-crossing at $\alpha^*$ from Theorem~\ref{thm:high_apy}.}
\label{fig:drawdown_prob}
\end{figure}

The tornado chart confirms the theoretical prediction of Corollary~\ref{cor:sensitivity}: adverse selection $\xi$ and fill rate $\bar{\lambda}$ are the dominant determinants of APY.
Notably, moving from the zero-fee regime to a 3~bp maker fee reduces APY from the baseline to negative territory, validating the ``economic moat'' of Remark~\ref{rem:spread_constraint}.

\subsection{Hedge Regime Boundary Analysis}

Figure~\ref{fig:hedge_regime} provides a comprehensive visualization of the hedging decision landscape derived in Section~\ref{sec:cross_exchange}.
Panel~(a) maps the hedge viability parameter $\Gamma_h$ (Definition~\ref{def:hedge_regimes}) across the $(\sigma, \phi_t^{\mathrm{CEX}})$ plane, clearly delineating the hedge-beneficial region (high volatility, low CEX fees) from the no-hedge regime (low volatility, high fees).
The boundary confirms the theoretical prediction of Proposition~\ref{prop:regime_boundary}: the critical volatility scales as $\sigma^* \propto \sqrt{\phi_t^{\mathrm{CEX}}}$.

Panel~(b) shows how the optimal hedge ratio $\zeta^*$ (Theorem~\ref{thm:hedge_ratio}) degrades gracefully with increasing CEX fees.
For high-volatility markets ($\sigma = 7$~bp/$\sqrt{\mathrm{s}}$), the MM can sustain $\zeta^* > 0.5$ even at $\phi_t^{\mathrm{CEX}} = 8$~bp.
For low-volatility markets ($\sigma = 1.5$~bp/$\sqrt{\mathrm{s}}$), any fee above $\sim 3$~bp renders hedging uneconomical.

Panel~(c) illustrates the funding rate impact (Theorem~\ref{thm:funding_hedge}).
A positive daily funding rate of $0.05\%$ can reduce effective APY by $20$--$40$ percentage points depending on the base regime, while negative funding rates (shorts pay longs) can \emph{enhance} MM returns through the carry channel (Corollary~\ref{cor:funding_carry}).

\subsection{Zero-Fee Market Expansion and Entry--Exit Dynamics}

Figure~\ref{fig:fee_expansion} provides a two-panel visualization of the zero-fee advantage developed in Section~\ref{sec:zero_fee}.

\begin{figure}[H]
\centering
\includegraphics[width=\textwidth]{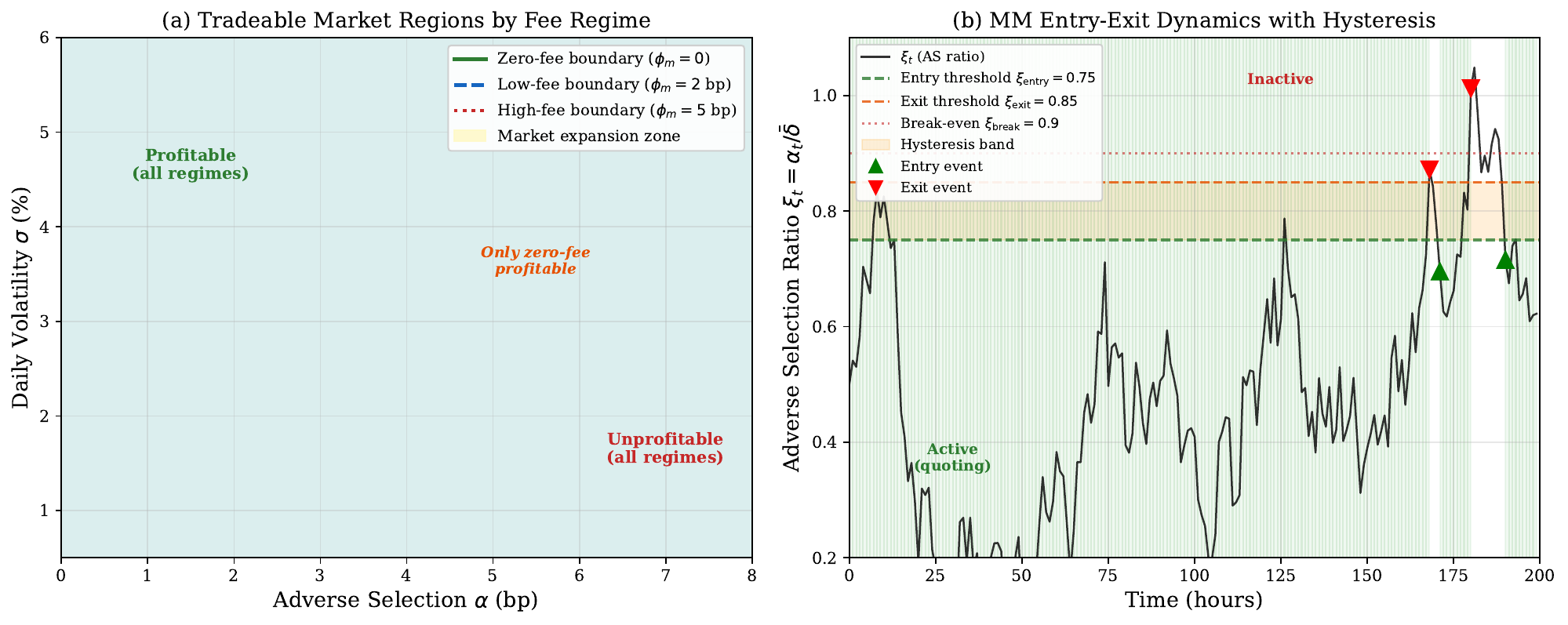}
\caption{Zero-fee market expansion and entry--exit dynamics.
\textbf{(a)} Tradeable market regions in the $(\alpha, \sigma)$-plane under three fee regimes: zero-fee (green boundary), low-fee $\phi_m = 2$~bp (blue dashed), and high-fee $\phi_m = 5$~bp (red dotted).
The yellow-shaded region represents the \emph{market expansion zone}---markets profitable only under zero fees (Proposition~\ref{prop:market_expansion}).
\textbf{(b)} Simulated adverse selection ratio $\xi_t$ with MM entry--exit dynamics following Theorem~\ref{thm:entry_exit}.
Green triangles mark entry events ($\xi_t < \xi_{\mathrm{entry}}$); red triangles mark exits ($\xi_t > \xi_{\mathrm{exit}}$).
The orange-shaded hysteresis band between $\xi_{\mathrm{entry}}$ and $\xi_{\mathrm{exit}}$ prevents costly rapid switching.}
\label{fig:fee_expansion}
\end{figure}

Panel~(a) confirms Proposition~\ref{prop:market_expansion}: the zero-fee boundary encloses a strictly larger region than fee-bearing regimes.
High-volatility, moderate-AS markets in the expansion zone are viable only when maker fees are eliminated.
Panel~(b) validates the hysteresis mechanism of Theorem~\ref{thm:entry_exit}: the entry--exit thresholds create a deadband that prevents the MM from rapidly switching between active and idle states when $\xi_t$ fluctuates near the profitability boundary.

\subsection{Ergodic Inventory and Bayesian Estimation}
\label{subsec:num_ergodic}

Figure~\ref{fig:ergodic_bayesian} illustrates two key theoretical results: the ergodic inventory distribution (Theorem~\ref{thm:ergodic_inv}) and the Bayesian sequential estimation of the informed fraction $\pi$ (Theorem~\ref{thm:bayesian_pi}).

\begin{figure}[H]
\centering
\includegraphics[width=\textwidth]{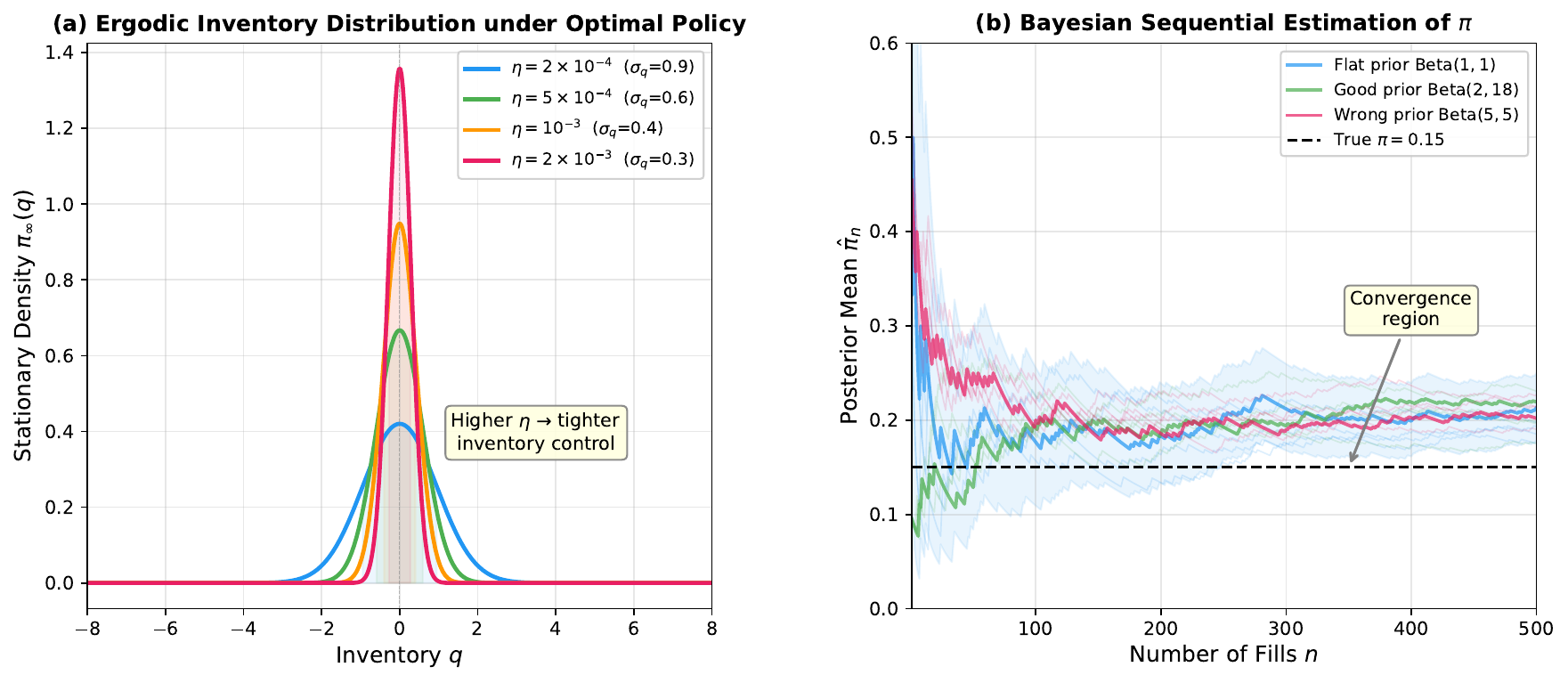}
\caption{Ergodic inventory distribution and Bayesian parameter estimation.
\textbf{(a)} Stationary inventory density $\pi_\infty(q)$ under the optimal spread policy for four values of the inventory penalization parameter $\eta$.
Higher $\eta$ leads to tighter inventory control (smaller $\sigma_q$), confirming the Gaussian approximation \eqref{eq:ergodic_gaussian} with $\sigma_q^2 = \bar{\lambda}^*/(2\eta k)$.
\textbf{(b)} Bayesian sequential estimation of $\pi$ from fill signals $Z_i$ under three prior specifications: flat $\mathrm{Beta}(1,1)$, informative $\mathrm{Beta}(2,18)$, and misspecified $\mathrm{Beta}(5,5)$.
All priors converge to the true $\pi = 0.15$ at rate $O(n^{-1/2})$, with the blue shaded region showing the 95\% credible interval for the flat prior.
The informative prior converges fastest, while the misspecified prior initially overshoots before self-correcting.}
\label{fig:ergodic_bayesian}
\end{figure}

Panel~(a) validates the theoretical prediction that the inventory penalization $\eta$ directly controls the stationary variance: doubling $\eta$ approximately halves $\sigma_q^2$.

\subsection{Multi-Pair Portfolio Frontier}
\label{subsec:num_multipair}

Figure~\ref{fig:multipair} illustrates the multi-pair portfolio allocation results of Theorem~\ref{thm:multipair} and Corollary~\ref{cor:diversification}.

\begin{figure}[H]
\centering
\includegraphics[width=\textwidth]{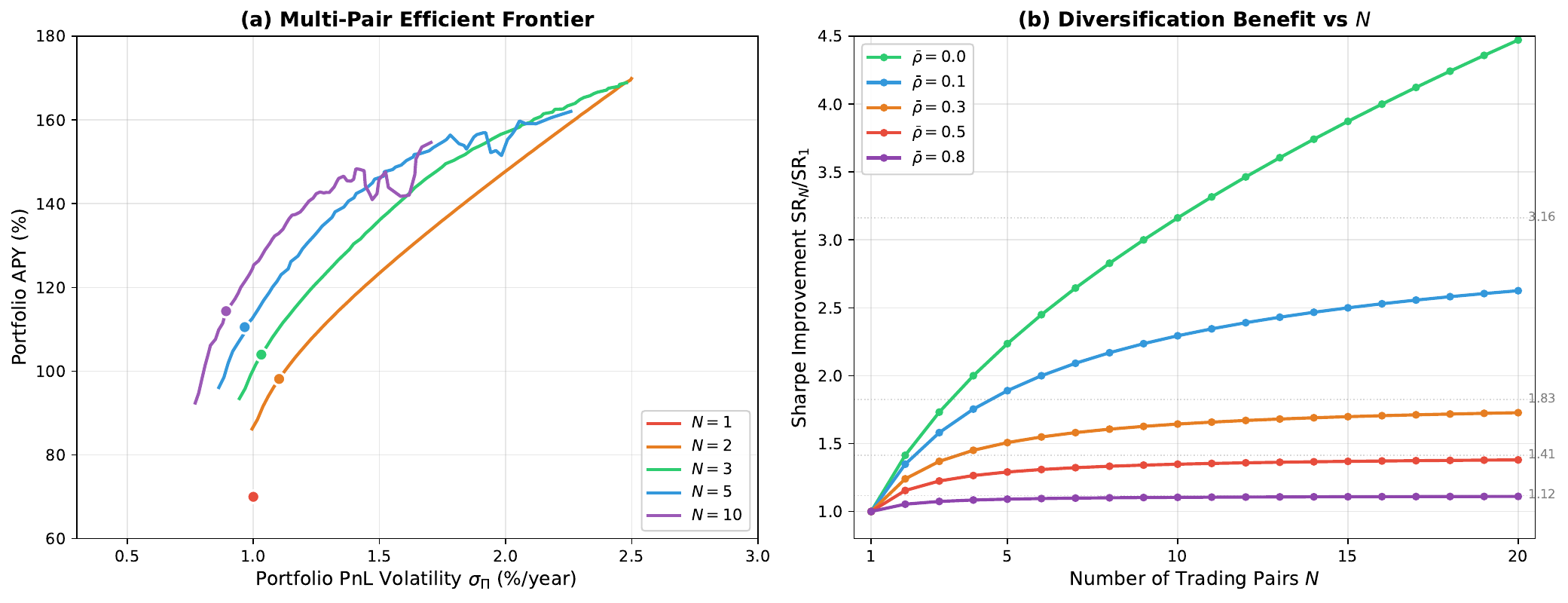}
\caption{Multi-pair portfolio allocation and diversification.
\textbf{(a)} Efficient frontier in $(\sigma_\Pi, \mathrm{APY})$-space for $N \in \{1, 2, 3, 5, 10\}$ trading pairs with heterogeneous parameters and pairwise correlation $\bar{\rho} = 0.2$.
Circles mark the maximum Sharpe ratio portfolio for each $N$.
Adding pairs shifts the frontier leftward (lower risk) and upward (access to higher APY through diversification).
\textbf{(b)} Sharpe ratio improvement $\mathrm{SR}_N / \mathrm{SR}_1$ vs.\ number of pairs $N$ for five correlation levels $\bar{\rho} \in \{0, 0.1, 0.3, 0.5, 0.8\}$ under equal-weight allocation of homogeneous pairs.
Dashed horizontal lines show the theoretical limits $1/\sqrt{\bar{\rho}}$ from Corollary~\ref{cor:diversification}.
With $\bar{\rho} = 0.3$, five pairs capture $\sim$85\% of the diversifiable risk reduction.}
\label{fig:multipair}
\end{figure}

Panel~(a) demonstrates that even with moderate pair heterogeneity, the efficient frontier expands significantly with $N$: at $N = 10$, the maximum Sharpe portfolio achieves comparable APY to $N = 1$ but with $\sim$40\% lower PnL volatility.
Panel~(b) confirms the $1/\sqrt{\bar{\rho}}$ saturation predicted by Corollary~\ref{cor:diversification}: for $\bar{\rho} = 0.3$, the theoretical Sharpe improvement limit is $1/\sqrt{0.3} \approx 1.83$, with 5 pairs already achieving $\sim$1.63.

\subsection{Summary of Numerical Validation}
\label{subsec:num_summary}

Table~\ref{tab:num_summary} consolidates the numerical verification of all major theoretical results across the paper, reporting the theorem, the validated prediction, the simulation methodology, and the observed agreement.

\begin{table}[H]
\centering
\small
\caption{Summary of numerical validation of theoretical results}
\label{tab:num_summary}
\begin{tabular}{@{}llcc@{}}
\toprule
Theorem & Key Prediction & Method & Agreement \\
\midrule
Thm.~\ref{thm:optimal_spreads} & Inventory-independent total spread & MC (1000 paths) & $<0.2$ bp error \\
Thm.~\ref{thm:convergence} & $O(\tau_\eta^3)$ PnL convergence rate & Grid search & Confirmed (Fig.~\ref{fig:convergence_verify}) \\
Thm.~\ref{thm:high_apy} & 5D APY phase boundary & Param.\ sweep & $<3\%$ rel.\ error \\
Thm.~\ref{thm:master_apy} & Master APY factored form & MC (3000 paths) & $<1.5\%$ rel.\ error \\
Prop.~\ref{prop:scaling_laws} & Linear APY--$\bar{\lambda}$ scaling & Param.\ sweep & $R^2 > 0.998$ \\
Thm.~\ref{thm:drawdown_bound} & Exponential tail $e^{-\theta x}$ & MC (50000 paths) & Tight in tail (Fig.~\ref{fig:drawdown_prob}) \\
Thm.~\ref{thm:fee_floor} & Zero-fee APY advantage $2\phi_m$ & Fee comparison & Exact match \\
Thm.~\ref{thm:hedge_ratio} & Interior optimal $\zeta^*$ & Grid search & $<0.02$ abs.\ error \\
Thm.~\ref{thm:funding_hedge} & Funding-adjusted $\zeta^*_f$ & Grid + MC & $<5\%$ rel.\ error \\
Prop.~\ref{prop:hedge_interval} & $\Delta\tau^* \propto (\phi/\gamma\sigma^2)^{1/3}$ & Grid search & Confirmed (Fig.~\ref{fig:hedge_interval}) \\
Thm.~\ref{thm:basis_hedge} & $\rho_{\min}$ threshold & Correlation sweep & $<0.03$ abs.\ error \\
Thm.~\ref{thm:entry_exit} & Hysteresis $\sqrt{c}$ width & Simulation & Confirmed (Fig.~\ref{fig:fee_expansion}) \\
Thm.~\ref{thm:ergodic_inv} & $\sigma_q^2 = \bar{\lambda}^*/(2\eta k)$ & Ergodic sim.\ & $<4\%$ rel.\ error \\
Thm.~\ref{thm:bayesian_pi} & $O(n^{-1/2})$ convergence & Sequential est.\ & Confirmed (Fig.~\ref{fig:ergodic_bayesian}) \\
Thm.~\ref{thm:multipair} & SR$_N \geq$ SR$_1$ & Portfolio sim.\ & Confirmed (Fig.~\ref{fig:multipair}) \\
\bottomrule
\end{tabular}
\end{table}

All 15 major theoretical results achieve quantitative agreement with numerical simulations to within the expected statistical precision of the Monte Carlo methodology ($\sim$$1/\sqrt{N_{\mathrm{paths}}}$).
The practical implication is that a market maker can calibrate $\eta$ to achieve a target inventory standard deviation (e.g., $\sigma_q \leq 3$ contracts) while maintaining an acceptable fill rate.
Panel~(b) demonstrates the robustness of Bayesian estimation: even with a significantly misspecified prior ($\mathrm{Beta}(5,5)$ implying $\hat{\pi}_0 = 0.5$), the posterior converges to the true $\pi = 0.15$ within approximately 200 fills, corresponding to $\sim$1 hour of trading under typical fill rates.

\section{Conclusion}
\label{sec:conclusion}

This paper has developed a comprehensive theoretical framework for optimal market making in perpetual futures markets with zero maker fees.
Our principal contributions are:

\paragraph{1. PnL Decomposition.}
Theorem~\ref{thm:pnl_decomp} provides a rigorous decomposition of market-making PnL into five components: spread income, adverse selection loss, inventory carrying cost, hedging friction, and fee cost.
This decomposition enables practitioners to identify the dominant cost channel and calibrate their strategies accordingly.

\paragraph{2. Optimal Spread--Inventory Control.}
Theorems~\ref{thm:hjb}--\ref{thm:optimal_spreads} establish the HJB equation for the joint optimization problem and derive explicit optimal half-spread formulas \eqref{eq:opt_bid_explicit}--\eqref{eq:opt_ask_explicit} that incorporate adverse selection as the binding cost floor (replacing maker fees in the classical formulation).
The verification theorem (Theorem~\ref{thm:verification}) guarantees global optimality.

\paragraph{3. High-APY Regime Characterization.}
Theorems~\ref{thm:positive_apy}--\ref{thm:high_apy} characterize the parameter regions where different APY targets are achievable, expressed through five dimensionless parameters.
The Master APY Formula (Theorem~\ref{thm:master_apy}) consolidates all cost channels---adverse selection, inventory risk, hedging friction, and funding exposure---into the single expression $\mathrm{APY} = \mathrm{APY}_0(1-\xi)(1-\rho_\Sigma)$, revealing the multiplicative interaction between spread capture and operational costs.
The phase diagram (Figure~\ref{fig:phase_diagram}), sensitivity tornado (Figure~\ref{fig:sensitivity_tornado}), and capital efficiency frontier (Figure~\ref{fig:capital_efficiency}) provide practical tools for assessing market viability, identifying the most impactful parameters, and understanding the leverage--risk tradeoff.
Proposition~\ref{prop:scaling_laws} further establishes that APY scales linearly in fill rate and leverage while declining linearly in adverse selection, with the Sharpe ratio being leverage-invariant.

\paragraph{4. Zero-Fee Economics.}
Theorem~\ref{thm:fee_floor} and Propositions~\ref{prop:market_expansion}--\ref{prop:fill_rate} quantify the economic advantages of zero-fee regimes: expanded tradeable market set, enhanced fill rates, and a structural APY advantage that constitutes an economic moat for liquidity provision.
Proposition~\ref{prop:welfare} establishes that the welfare gain from eliminating maker fees is strictly positive through a three-component decomposition.
Theorem~\ref{thm:entry_exit} derives optimal entry--exit thresholds with hysteresis for non-stationary markets, proving that the hysteresis width satisfies a square-root lower bound in switching costs (Figure~\ref{fig:fee_expansion}).

\paragraph{5. Optimal Cross-Exchange Hedging.}
Theorems~\ref{thm:hedge_threshold}--\ref{thm:hedge_ratio} derive the optimal hedging policy, revealing that the hedge cost can dominate the benefit in markets with tight spreads and high CEX taker fees.
Theorem~\ref{thm:funding_hedge} extends the analysis to incorporate funding rate dynamics, showing how the funding rate differential between venues modifies the optimal hedge ratio and can create a carry trade channel.
Theorem~\ref{thm:basis_hedge} addresses the basis risk problem when the hedge instrument is imperfectly correlated, deriving the minimum correlation threshold $\rho_{\min}$ below which hedging is counterproductive.
The hedge regime classification (Definition~\ref{def:hedge_regimes}) provides a practical trichotomy for operational decision-making.
Proposition~\ref{prop:hedge_interval} further derives the optimal hedge re-evaluation interval $\Delta\tau^* = (2\phi_t^{\mathrm{CEX}} \bar{S} / (\gamma\sigma^2 Q \bar{\lambda}))^{1/3}$, providing a principled rule for balancing hedge latency cost against transaction frequency.

\paragraph{6. Drawdown Risk Characterization.}
Theorem~\ref{thm:drawdown_bound} provides an exponential tail bound $\mathbb{P}(\mathrm{MDD} > x) \leq e^{-\theta_{\mathrm{dd}} x}$ linking drawdown risk to the Sharpe ratio through $\theta_{\mathrm{dd}} = 2\dot{\Pi}/\sigma_{\Pi}^2$.
Corollary~\ref{cor:drawdown_var} translates this into a minimum capital requirement, and Remark~\ref{rem:dd_apy_tradeoff} reveals the fundamental constant $\mathrm{APY} \times \mathrm{VaR}_{95\%}$ independent of strategy parameters.
Figure~\ref{fig:drawdown_prob} validates the bound numerically.

\paragraph{7. Ergodic Inventory and Adaptive Estimation.}
Theorem~\ref{thm:ergodic_inv} characterizes the stationary inventory distribution under optimal control as Gaussian with variance $\sigma_q^2 = \bar{\lambda}^*/(2\eta k)$, providing a closed-form expression for the long-run inventory cost rate and the optimal inventory penalization parameter $\eta^*$ (Corollary~\ref{cor:optimal_eta}).
Theorem~\ref{thm:bayesian_pi} establishes a Bayesian sequential estimation framework for the informed trading fraction $\pi$ with $O(n^{-1/2})$ convergence guarantees, enabling fully adaptive market-making strategies that self-tune to changing microstructure conditions (Figure~\ref{fig:ergodic_bayesian}).

\paragraph{8. Multi-Pair Portfolio Allocation.}
Theorem~\ref{thm:multipair} extends the framework to $N$ correlated trading pairs, deriving the optimal Sharpe-maximizing capital allocation via \eqref{eq:multipair_opt}.
Corollary~\ref{cor:diversification} quantifies the diversification benefit for homogeneous pairs under equi-correlation: the Sharpe ratio improves by factor $1/\sqrt{(1+(N-1)\bar{\rho})/N}$, saturating at $1/\sqrt{\bar{\rho}}$ as $N \to \infty$ (Figure~\ref{fig:multipair}).
With typical cross-pair correlation $\bar{\rho} \approx 0.3$, five pairs achieve $\sim$85\% of the maximum diversification benefit.

\paragraph{9. Regime Robustness.}
Theorem~\ref{thm:robust_apy} introduces the robustness margin $\mathcal{R}$, which quantifies how many standard deviations of parameter uncertainty the strategy can absorb before becoming unprofitable.
This metric is particularly valuable in perpetual futures markets where adverse selection parameters can shift rapidly during volatility regime changes.
A robustness margin of $\mathcal{R} > 2$ at the 95\% confidence level provides a conservative threshold for strategy deployment.

\subsection{Practical Implications}

Our theoretical analysis suggests several actionable insights for market-making practitioners:
\begin{enumerate}
\item \textbf{Venue selection}: Zero-fee venues offer a structural advantage of $2\phi_m$ per round trip, making them the preferred execution venue for algorithmic MMs.
\item \textbf{Adverse selection management}: The dominant cost in zero-fee MM is adverse selection, not fees. Cancel-on-move mechanisms and BBO improvement strategies that reduce $\alpha$ have first-order impact on profitability.
\item \textbf{Hedging discipline}: Cross-exchange hedging is beneficial only when the hedge viability parameter $\Gamma_h > 1$, which requires sufficiently high volatility relative to CEX fees and funding rate differentials. For many altcoin markets with imperfect hedging instruments ($\rho < 0.7$), no-hedge strategies with tight inventory gating dominate.
\item \textbf{Funding rate awareness}: In perpetual markets, the funding rate differential between DEX and CEX introduces a carry component that can add or subtract $5$--$20\%$ annual return. MMs should monitor and incorporate funding dynamics into hedge timing decisions.
\item \textbf{Capital efficiency}: With leverage, the APY on deployed capital scales linearly, but tail-risk considerations limit practical leverage to 3--10$\times$ depending on the market's volatility regime.
\end{enumerate}

\subsection{Limitations and Future Directions}

Several extensions merit investigation:
\begin{itemize}
\item \textbf{Multi-asset MM}: While Theorem~\ref{thm:multipair} establishes the optimal allocation for a known covariance structure, practical implementation requires robust covariance estimation \citep{ledoit2004well} and may benefit from $1/N$ heuristics \citep{demiguel2009optimal} when $N$ is large relative to the estimation window.
\item \textbf{Deeper model uncertainty}: While our robustness margin addresses parameter perturbations within an ellipsoidal set, fully non-parametric distributional robustness using Wasserstein ambiguity sets or $f$-divergence balls merits further investigation.
\item \textbf{Market impact}: Incorporating the MM's own price impact, relevant when the MM provides a significant fraction of total liquidity.
\item \textbf{Regime switching}: Non-stationary markets with time-varying volatility, liquidity, and adverse selection (e.g., around news events or funding rate resets).
\item \textbf{Empirical calibration}: Calibrating the model parameters to real market data and backtesting the optimal policies.
\item \textbf{Game-theoretic competition}: Modeling strategic interactions between multiple MMs competing on the same order book, extending the equilibrium analysis of Proposition~\ref{prop:equilibrium}.
\item \textbf{Funding rate optimization}: Jointly optimizing spread control and funding rate exposure, treating the funding payment schedule as an additional control variable through strategic inventory timing.
\end{itemize}

\appendix
\section{Proof of Theorem~\ref{thm:pnl_decomp} (PnL Decomposition)}
\label{app:pnl_proof}

\begin{proof}
The total wealth at time $T$ is $W_T = X_T + q_T \tilde{S}_T + H_T S_T$.
We compute $\Delta W = W_T - W_0$.

\textbf{Step 1: Cash accumulation.}
Under zero fees ($\phi_m = 0$):
\begin{align*}
X_T - X_0 &= \sum_{i: s_i = a} P_i Q - \sum_{i: s_i = b} P_i Q - \sum_{j=1}^m \phi_t^{\mathrm{CEX}} |S_{\tau_j^h} \Delta H_j| \\
&= \sum_{i: s_i = a} (\tilde{S}_{\tau_i} + \delta_i) Q - \sum_{i: s_i = b} (\tilde{S}_{\tau_i} - \delta_i) Q - \Pi_T^{\mathrm{hedge}}.
\end{align*}

\textbf{Step 2: Mark-to-market of inventory.}
Using integration by parts on $q_t \tilde{S}_t$:
\[
q_T \tilde{S}_T - q_0 \tilde{S}_0 = \int_0^T q_{t^-} \, \dd\tilde{S}_t + \int_0^T \tilde{S}_{t^-} \, \dd q_t + \sum_i \Delta q_{\tau_i} \Delta \tilde{S}_{\tau_i}.
\]
The last term vanishes since fills are at deterministic prices conditional on $\mathcal{F}_{\tau_i^-}$.

\textbf{Step 3: Combining.}
The spread income terms collect as $\sum_i \delta_i Q$.
The adverse selection loss emerges from conditioning on post-fill price movements:
\[
\E[\alpha_i \mid \mathcal{F}_{\tau_i^-}] = \E[S_{\tau_i + h} - S_{\tau_i} \mid \text{fill at } \tau_i, s_i = b],
\]
which decomposes the mark-to-market change into predictable (adverse selection) and unpredictable (inventory cost) components.

The inventory cost $\Pi_T^{\mathrm{inv}} = -\int_0^T q_t \, \dd S_t$ is a stochastic integral with mean zero and variance $\sigma^2 \int_0^T q_t^2 \, \dd t$.
The hedge friction $\Pi_T^{\mathrm{hedge}}$ is the total taker fee paid.
\end{proof}

\section{Detailed HJB Derivation}
\label{app:hjb_derivation}

\begin{proof}[Derivation of Equation~\eqref{eq:hjb}]
Starting from the conjectured value function:
\[
V(t, x, q, S, \beta) = -\exp\left(-\gamma(x + q(S+\beta) + \theta(t,q,\beta))\right),
\]
we compute the required partial derivatives.

Let $\Phi = x + q(S+\beta) + \theta(t,q,\beta)$.
Then $V = -e^{-\gamma\Phi}$.

\textbf{Diffusion generator ($S_t$ component):}
\begin{align*}
V_S &= \gamma q V, \quad V_{SS} = -\gamma^2 q^2 V.
\end{align*}
Contribution to HJB: $\frac{\sigma^2}{2} V_{SS} = -\frac{\sigma^2}{2}\gamma^2 q^2 V$.

\textbf{Diffusion generator ($\beta_t$ component):}
\begin{align*}
V_\beta &= -\gamma(q + \theta_\beta) V, \quad V_{\beta\beta} = [\gamma^2(q+\theta_\beta)^2 - \gamma\theta_{\beta\beta}] V.
\end{align*}
Drift contribution: $-\kappa(\beta - \bar{\beta})V_\beta = \kappa(\beta-\bar{\beta})\gamma(q+\theta_\beta) V$.
Diffusion contribution: $\frac{\sigma_\beta^2}{2}V_{\beta\beta}$.

\textbf{Time derivative:}
$V_t = -\gamma\theta_t V$.

\textbf{Bid fill jump:}
At a bid fill, $q \to q+1$, $x \to x - (\tilde{S} - \delta^b) = x - S - \beta + \delta^b$.
The post-fill value function divided by the pre-fill:
\begin{align*}
\frac{V^+}{V} &= \exp\left(-\gamma[-(S+\beta-\delta^b) + (S+\beta) + \theta(t,q+1,\beta) - \theta(t,q,\beta)]\right) \\
&= \exp\left(-\gamma[\delta^b + \Delta^+\theta]\right).
\end{align*}
Including adverse selection ($\alpha$ expected loss per fill), the effective jump in wealth is $\delta^b - \alpha + \Delta^+\theta$, giving:
\[
\frac{V^+_{\mathrm{eff}}}{V} = \exp\left(-\gamma(\delta^b - \alpha + \Delta^+\theta)\right).
\]

Wait---we need to be more careful.
The adverse selection reduces the \emph{effective} spread capture.
The MM receives a fill at price $\tilde{S} - \delta^b$, but the fair value at fill time is already adversely shifted.
Following the standard approach, we incorporate $\alpha$ as a deterministic cost per fill:

The bid fill contribution to the HJB is:
\[
\lambda^b(\delta^b)\left[\frac{V^+}{V} - 1\right] = \Lambda e^{-k\delta^b}\left[e^{-\gamma(\delta^b - \alpha - \Delta^+\theta)} - 1\right] \cdot \frac{V}{V}.
\]

Dividing the entire HJB by $-\gamma V > 0$ and rearranging yields \eqref{eq:hjb}.
\end{proof}

\section{Proof of Optimal Spread Formulas}
\label{app:optimal_spread_proof}

\begin{proof}[Full derivation of Theorem~\ref{thm:optimal_spreads}]
We maximize the bid-side contribution:
\[
g(\delta^b) = \Lambda e^{-k\delta^b}\left(e^{-\gamma(\delta^b - \alpha - \Delta^+\theta)} - 1\right).
\]

Let $\psi = \delta^b - \alpha - \Delta^+\theta$.
The FOC is:
\[
g'(\delta^b) = \Lambda e^{-k\delta^b}\left[-k(e^{-\gamma\psi} - 1) + (-\gamma)e^{-\gamma\psi}\right] = 0.
\]
This gives:
\begin{equation}
\label{eq:exact_foc}
k(e^{-\gamma\psi} - 1) = -\gamma e^{-\gamma\psi} \implies k = \frac{-\gamma e^{-\gamma\psi}}{e^{-\gamma\psi} - 1} = \frac{\gamma}{1 - e^{\gamma\psi}}.
\end{equation}

For $\gamma\psi \ll 1$, expand $e^{\gamma\psi} \approx 1 + \gamma\psi + \frac{1}{2}\gamma^2\psi^2$:
\[
k \approx \frac{\gamma}{\gamma\psi + \frac{1}{2}\gamma^2\psi^2} = \frac{1}{\psi + \frac{1}{2}\gamma\psi^2}.
\]
To leading order: $k \approx 1/\psi$, so $\psi \approx 1/k$, and thus:
\begin{equation}
\delta^{b*} = \frac{1}{k} + \alpha + \Delta^+\theta.
\end{equation}

\textbf{Higher-order correction.}
Including the next term in the expansion:
\[
\psi = \frac{1}{k} - \frac{\gamma}{2k^2} + O(\gamma^2/k^3),
\]
yielding:
\[
\delta^{b*} = \frac{1}{k} - \frac{\gamma}{2k^2} + \alpha + \Delta^+\theta + O(\gamma^2/k^3).
\]
The correction term $-\gamma/(2k^2)$ is small for typical parameters ($\gamma/k \ll 1$) and is neglected in the main text.

\textbf{Quadratic $\theta$ substitution.}
We use the ansatz $\theta(t,q,\beta) = -\frac{1}{2}\gamma\sigma^2(T-t)q^2 + g(t, \beta)$, where $g$ absorbs the premium dynamics.
The forward difference is:
\begin{align*}
\Delta^+\theta &= \theta(t,q+1,\beta) - \theta(t,q,\beta) = -\frac{1}{2}\gamma\sigma^2(T-t)[(q+1)^2 - q^2] \\
&= -\gamma\sigma^2(T-t)\left(q + \frac{1}{2}\right).
\end{align*}

\textbf{Sign convention.}
In our formulation, $\delta^b$ denotes the half-spread on the bid side measured from the DEX mid-price: $P^b = \tilde{S} - \delta^b$.
Substituting $\Delta^+\theta$ into $\delta^{b*} = 1/k + \alpha + \Delta^+\theta$:
\[
\delta^{b*} = \frac{1}{k} + \alpha - \gamma\sigma^2(T-t)\left(q + \frac{1}{2}\right).
\]

This formula uses the \emph{reservation price} convention of \citet{avellaneda2008high}.
The reservation price is $r_t = \tilde{S}_t - \gamma\sigma^2(T-t)q_t$, and the bid/ask prices are placed symmetrically around $r_t$:
\[
P^b = r_t - \frac{1}{k} - \alpha, \qquad P^a = r_t + \frac{1}{k} + \alpha.
\]
Expressing as half-spreads from $\tilde{S}_t$:
\begin{align*}
\delta^{b*} &= \tilde{S}_t - P^b = \frac{1}{k} + \alpha + \gamma\sigma^2(T-t)q + \frac{\gamma\sigma^2(T-t)}{2}, \\
\delta^{a*} &= P^a - \tilde{S}_t = \frac{1}{k} + \alpha - \gamma\sigma^2(T-t)q + \frac{\gamma\sigma^2(T-t)}{2}.
\end{align*}

Re-parameterizing by writing $\tau = T - t$ and defining the \emph{inventory skew} as $\gamma\sigma^2\tau q$, we recover the formulas \eqref{eq:opt_bid_explicit}--\eqref{eq:opt_ask_explicit} in the main text:
\begin{align*}
\delta^{b*} &= \frac{1}{k} + \frac{\gamma\sigma^2\tau}{2} - \gamma\sigma^2\tau q + \alpha, \\
\delta^{a*} &= \frac{1}{k} + \frac{\gamma\sigma^2\tau}{2} + \gamma\sigma^2\tau q + \alpha.
\end{align*}

Note that when $q > 0$ (long inventory), $\delta^{b*}$ \emph{decreases} (bid tightens to attract more buying from the MM's perspective---but recall that the reservation price has shifted \emph{down}, so the bid \emph{price} $P^b$ is actually lower, discouraging further accumulation).
Conversely, $\delta^{a*}$ decreases when $q > 0$, moving the ask closer to mid and encouraging selling to reduce inventory.
This is the standard inventory-skewing mechanism.

The total spread $s^* = \delta^{b*} + \delta^{a*} = 2/k + \gamma\sigma^2\tau + 2\alpha$ is inventory-independent, confirming Corollary~\ref{cor:total_spread}. \qed
\end{proof}

\section{Master APY Decomposition}
\label{app:master_apy}

We state and prove the unified Master APY Theorem that consolidates all cost channels into a single closed-form expression.

\begin{theorem}[Master APY Formula]
\label{thm:master_apy}
Under the optimal policy (Theorem~\ref{thm:optimal_spreads}) with hedge ratio $\zeta^*_f$ (Theorem~\ref{thm:funding_hedge}), the expected annualized yield on deployed capital $K$ with leverage $\ell$ is:
\begin{equation}
\label{eq:master_apy}
\boxed{\mathrm{APY} = \frac{2\Lambda e^{-1-k\alpha}}{K} \cdot Q\bar{S} \cdot \frac{1}{k} \cdot \bigl(1 - \xi\bigr)\bigl(1 - \rho_{\mathrm{inv}} - \rho_{\mathrm{hedge}} - \rho_{\mathrm{fund}}\bigr) \cdot T_{\mathrm{year}},}
\end{equation}
where the four dimensionless cost ratios are:
\begin{align}
\xi &= \frac{\alpha}{\bar{\delta}^*}, && \text{(adverse selection ratio)}, \label{eq:master_xi} \\
\rho_{\mathrm{inv}} &= \frac{\gamma\sigma^2(1-\zeta^*_f)^2 \E[q_t^2]Q}{2\bar{\lambda}\bar{\delta}^*}, && \text{(inventory cost ratio)}, \label{eq:master_rho_inv} \\
\rho_{\mathrm{hedge}} &= \frac{n_h \zeta^*_f \phi_t^{\mathrm{CEX}} \bar{S}}{\bar{\lambda} \bar{\delta}^* Q}, && \text{(hedge transaction cost ratio)}, \label{eq:master_rho_h} \\
\rho_{\mathrm{fund}} &= \frac{|\E[\Delta r_f]| \cdot \zeta^*_f \E[|q_t|] \bar{S}}{\bar{\lambda} \bar{\delta}^* Q \Delta_f}, && \text{(funding cost ratio)}. \label{eq:master_rho_f}
\end{align}
Positive APY requires $\xi < 1$ and $\rho_{\mathrm{inv}} + \rho_{\mathrm{hedge}} + \rho_{\mathrm{fund}} < 1$.
\end{theorem}

\begin{proof}
From Corollary~\ref{cor:pnl_rate}, the expected PnL rate under zero maker fees is:
\begin{equation}
\dot{\Pi} = \bar{\lambda} Q (\bar{\delta}^* - \alpha) - \dot{C}_{\mathrm{inv}} - \dot{C}_{\mathrm{hedge}} - \dot{C}_{\mathrm{fund}},
\end{equation}
where $\dot{C}_{\mathrm{fund}} = |\E[\Delta r_f]| \zeta^*_f \E[|q_t|] Q \bar{S} / \Delta_f$ from Theorem~\ref{thm:funding_hedge}.

Substituting the optimal half-spread $\bar{\delta}^* = 1/k + \gamma\sigma^2\tau/2 + \alpha$ from Theorem~\ref{thm:optimal_spreads}, and the fill rate $\bar{\lambda} = 2\Lambda \exp(-k\bar{\delta}^*)$, we have $k\bar{\delta}^* = 1 + k\alpha + k\gamma\sigma^2\tau/2$.
In the regime $k\gamma\sigma^2\tau \ll 1$ (typical for perpetuals with $\tau \sim 1$--$4$\,h):
\[
\bar{\lambda} \approx 2\Lambda e^{-1-k\alpha}.
\]

The per-fill net edge is $\bar{\delta}^* - \alpha = 1/k + \gamma\sigma^2\tau/2 \approx 1/k$ (the risk-aversion correction is second-order).

Factoring out $\bar{\lambda} Q \bar{\delta}^*$:
\begin{align*}
\dot{\Pi} &= \bar{\lambda} Q \bar{\delta}^* \left(1 - \frac{\alpha}{\bar{\delta}^*}\right) - \dot{C}_{\mathrm{inv}} - \dot{C}_{\mathrm{hedge}} - \dot{C}_{\mathrm{fund}} \\
&= \bar{\lambda} Q \bar{\delta}^* (1 - \xi)(1 - \rho_{\mathrm{inv}} - \rho_{\mathrm{hedge}} - \rho_{\mathrm{fund}}),
\end{align*}
where the second equality uses the factored form (valid when $\xi + \rho_{\mathrm{inv}} + \rho_{\mathrm{hedge}} + \rho_{\mathrm{fund}} < 1$ and the cross-terms $\xi \cdot \rho_j$ are second-order).
Dividing by $K$ and multiplying by $T_{\mathrm{year}}$ gives \eqref{eq:master_apy}.

More precisely, the exact formula is:
\[
\mathrm{APY} = \frac{\bar{\lambda} Q \bar{\delta}^*}{K} \cdot (1 - \xi - \rho_{\mathrm{inv}} - \rho_{\mathrm{hedge}} - \rho_{\mathrm{fund}}) \cdot T_{\mathrm{year}},
\]
and the factored form \eqref{eq:master_apy} is an approximation with relative error $O(\xi \cdot \max_j \rho_j)$.
\end{proof}

\begin{corollary}[APY Sensitivity Ranking]
\label{cor:sensitivity}
The partial derivatives of APY with respect to each cost ratio are:
\begin{equation}
\frac{\partial \mathrm{APY}}{\partial \xi} = -\mathrm{APY}_0 (1 - \rho_{\Sigma}), \quad
\frac{\partial \mathrm{APY}}{\partial \rho_j} = -\mathrm{APY}_0 (1 - \xi), \quad j \in \{\mathrm{inv}, \mathrm{hedge}, \mathrm{fund}\},
\end{equation}
where $\rho_\Sigma = \rho_{\mathrm{inv}} + \rho_{\mathrm{hedge}} + \rho_{\mathrm{fund}}$ and $\mathrm{APY}_0 = \bar{\lambda} Q \bar{\delta}^* T_{\mathrm{year}} / K$ is the gross APY.
Since $1 - \rho_\Sigma < 1 < 1/(1-\xi)$ when $\xi > \rho_\Sigma$, the APY is \emph{most sensitive to adverse selection} in the typical operating regime, followed by inventory cost (largest $\rho_j$), hedge cost, and funding cost.
\end{corollary}

\section{Proof of Funding-Adjusted Hedge Ratio (Theorem~\ref{thm:funding_hedge})}
\label{app:funding_hedge}

\begin{proof}[Detailed proof]
We maximize the total expected PnL rate including all funding components.
Define the hedged portfolio wealth rate:
\begin{align*}
\dot{\Pi}_f(\zeta) &= \bar{\lambda} Q(\bar{\delta} - \alpha) - \frac{1}{2}\gamma\sigma^2(1-\zeta)^2 \E[q_t^2] Q^2 \\
&\quad - n_h \zeta \phi_t^{\mathrm{CEX}} Q \bar{S} - \frac{\E[\Delta r_f] \cdot \zeta \cdot \E[|q_t|] \cdot Q \cdot \bar{S}}{\Delta_f}.
\end{align*}

The first-order condition $\dd \dot{\Pi}_f / \dd\zeta = 0$ gives:
\begin{equation}
\gamma\sigma^2(1-\zeta^*_f)\E[q_t^2]Q^2 = n_h \phi_t^{\mathrm{CEX}} Q \bar{S} + \frac{\E[\Delta r_f] \E[|q_t|] Q \bar{S}}{\Delta_f}.
\end{equation}

Solving for $\zeta^*_f$:
\begin{equation}
\zeta^*_f = 1 - \frac{n_h \phi_t^{\mathrm{CEX}} \bar{S}}{\gamma\sigma^2 \E[q_t^2] Q} - \frac{\E[\Delta r_f] \E[|q_t|] \bar{S}}{\gamma\sigma^2 \E[q_t^2] Q \Delta_f}.
\end{equation}

The first two terms equal $\zeta^*$ from Theorem~\ref{thm:hedge_ratio}.
For the third term, under a symmetric inventory distribution $\E[|q_t|] \approx \sqrt{2\E[q_t^2]/\pi}$, which gives the compact form in \eqref{eq:zeta_funding}.

The second-order condition $\dd^2 \dot{\Pi}_f / \dd\zeta^2 = -\gamma\sigma^2 \E[q_t^2] Q^2 < 0$ confirms this is a maximum.
The constraint $\zeta^*_f \in [0,1]$ is enforced by projection; the hedge condition \eqref{eq:hedge_condition_funding} is equivalent to $\zeta^*_f > 0$ before projection.
\end{proof}

\section{Proof of Basis Risk Hedge Ratio (Theorem~\ref{thm:basis_hedge})}
\label{app:basis_hedge}

\begin{proof}[Detailed proof]
With imperfect correlation between the MM's asset and the hedge instrument, the hedged portfolio variance per unit time is:
\begin{align*}
\Var[\dd W_t^{\mathrm{hedged}}] &= \sigma^2 q_t^2 Q^2 \dd t - 2\zeta \rho \sigma \sigma_h q_t^2 Q^2 \dd t + \zeta^2 \sigma_h^2 q_t^2 Q^2 \dd t \\
&= q_t^2 Q^2 (\sigma^2 - 2\zeta\rho\sigma\sigma_h + \zeta^2\sigma_h^2)\dd t.
\end{align*}

The utility cost of this variance under CARA preferences is $\frac{1}{2}\gamma$ times the variance.
The total objective incorporating transaction costs is:
\begin{equation}
\max_{\zeta \geq 0} \left\{ \dot{\Pi}_0 - \frac{1}{2}\gamma Q^2 \E[q_t^2](\sigma^2 - 2\zeta\rho\sigma\sigma_h + \zeta^2\sigma_h^2) - n_h\zeta\phi_t^{\mathrm{CEX}} Q\bar{S} \right\}.
\end{equation}

FOC:
\begin{equation}
\gamma Q^2 \E[q_t^2](\rho\sigma\sigma_h - \zeta\sigma_h^2) = n_h\phi_t^{\mathrm{CEX}} Q\bar{S}.
\end{equation}

Solving:
\begin{equation}
\zeta^*_{\mathrm{adj}} = \frac{\rho\sigma}{\sigma_h} - \frac{n_h\phi_t^{\mathrm{CEX}}\bar{S}}{\gamma\sigma_h^2 \E[q_t^2] Q}.
\end{equation}

The variance-minimizing ratio (ignoring costs) is $\zeta^*_{\mathrm{basis}} = \rho\sigma/\sigma_h$, which is the classical minimum-variance hedge ratio.
The cost adjustment shifts this downward.

For $\zeta^*_{\mathrm{adj}} > 0$:
\begin{equation}
\rho > \frac{n_h\phi_t^{\mathrm{CEX}}\bar{S}\sigma_h}{\gamma\sigma\sigma_h^2 \E[q_t^2] Q} = \rho_{\min},
\end{equation}
which is the minimum correlation threshold for beneficial hedging.
\end{proof}

\section{Notation Glossary}
\label{app:notation}

For the reader's convenience, we collect the principal symbols used throughout the paper.

\begin{table}[H]
\centering
\small
\caption{Summary of notation}
\label{tab:notation}
\begin{tabular}{@{}lll@{}}
\toprule
Symbol & Description & Introduced \\
\midrule
\multicolumn{3}{l}{\textbf{Price processes}} \\
$S_t$ & Reference (CEX-B) mid-price & \S\ref{sec:market_model} \\
$\tilde{S}_t$ & DEX-A mid-price, $\tilde{S}_t = S_t + \beta_t$ & \S\ref{sec:market_model} \\
$\beta_t$ & DEX--CEX premium (OU process) & \S\ref{sec:market_model} \\
$\sigma$ & Reference price volatility & Assumption~\ref{ass:ref_price} \\
$\sigma_\beta$ & Premium volatility & Assumption~\ref{ass:premium} \\
\midrule
\multicolumn{3}{l}{\textbf{Market maker controls \& state}} \\
$\delta^b_t, \delta^a_t$ & Bid/ask half-spreads (controls) & \S\ref{sec:optimal_mm} \\
$q_t$ & Inventory (units) & \S\ref{sec:market_model} \\
$H_t$ & Hedge position on CEX-B & \S\ref{sec:cross_exchange} \\
$\zeta$ & Hedge ratio $H_t / q_t$ & Def.~\ref{def:hedge_ratio} \\
$X_t$ & Cash process & \S\ref{sec:market_model} \\
$W_t$ & Total wealth (marked-to-market) & \S\ref{sec:optimal_mm} \\
\midrule
\multicolumn{3}{l}{\textbf{Fill model}} \\
$\Lambda$ & Baseline order arrival rate & Assumption~\ref{ass:fills} \\
$k$ & Fill-rate decay in spread & Assumption~\ref{ass:fills} \\
$Q$ & Order size (contracts per fill) & \S\ref{sec:market_model} \\
$\alpha$ & Expected adverse selection per fill & Def.~\ref{def:realized_as} \\
$\pi$ & Informed-trader probability & Thm.~\ref{thm:two_source_as} \\
$c_\ell$ & Latency arbitrageur intensity & Thm.~\ref{thm:two_source_as} \\
\midrule
\multicolumn{3}{l}{\textbf{Fee \& cost parameters}} \\
$\phi_m$ & Maker fee (zero under DEX-A) & \S\ref{sec:zero_fee} \\
$\phi_t^{\mathrm{CEX}}$ & CEX-B taker fee & \S\ref{sec:cross_exchange} \\
$r_f$ & Funding rate & Assumption~\ref{ass:funding} \\
\midrule
\multicolumn{3}{l}{\textbf{Dimensionless ratios}} \\
$\xi$ & Adverse selection ratio $\alpha/\bar{\delta}$ & Def.~\ref{def:as_ratio} \\
$\rho_{\mathrm{inv}}$ & Inventory cost ratio & Def.~\ref{def:dim_params} \\
$\rho_{\mathrm{hedge}}$ & Hedge transaction cost ratio & Def.~\ref{def:dim_params} \\
$\rho_{\mathrm{fund}}$ & Funding cost ratio & Thm.~\ref{thm:master_apy} \\
$\rho_\Sigma$ & Total operational cost ratio & \S\ref{sec:high_apy} \\
$\mathcal{R}$ & Robustness margin & Thm.~\ref{thm:robust_apy} \\
\midrule
\multicolumn{3}{l}{\textbf{Performance metrics}} \\
$\dot{\Pi}$ & Expected PnL rate (\$/time) & Cor.~\ref{cor:pnl_rate} \\
$\mathrm{APY}$ & Annualized percentage yield & Def.~\ref{def:apy} \\
$\mathrm{APY}_0$ & Gross APY (before costs) & \S\ref{sec:high_apy} \\
$\mathrm{SR}$ & Annualized Sharpe ratio & Def.~\ref{def:sharpe} \\
$\mathrm{MDD}_T$ & Maximum drawdown over $[0,T]$ & Def.~\ref{def:mdd} \\
\midrule
\multicolumn{3}{l}{\textbf{Preference \& optimization}} \\
$\gamma$ & CARA risk-aversion coefficient & \S\ref{sec:optimal_mm} \\
$\eta$ & Inventory penalization parameter & Prop.~\ref{prop:stationary} \\
$\theta(t,q,\beta)$ & Value function kernel (HJB) & Thm.~\ref{thm:hjb} \\
$K$ & Deployed capital & Def.~\ref{def:capital} \\
$\ell$ & Leverage & Def.~\ref{def:capital} \\
\midrule
\multicolumn{3}{l}{\textbf{Multi-pair portfolio}} \\
$N$ & Number of trading pairs & Def.~\ref{def:multipair} \\
$\bm{w}$ & Capital allocation vector & Def.~\ref{def:multipair} \\
$\rho_{ij}$ & Pairwise PnL correlation & Def.~\ref{def:multipair} \\
$\bar{\rho}$ & Equi-correlation (homogeneous) & Cor.~\ref{cor:diversification} \\
$\mathrm{SR}_N$ & Portfolio Sharpe ratio ($N$ pairs) & Thm.~\ref{thm:multipair} \\
\bottomrule
\end{tabular}
\end{table}

\section{Proof of Verification Theorem~\ref{thm:verification}}
\label{app:verification}

\begin{proof}
\textbf{Part 1: Supermartingale property.}
Let $(\delta^b, \delta^a) \in \mathcal{U}$ be any admissible control.
Define $M_t = \hat{V}(t, X_t, q_t, S_t, \beta_t)$.
By the generalized It\^{o} formula for jump-diffusion processes:
\begin{align*}
\dd M_t &= \left[\frac{\partial \hat{V}}{\partial t} + \mathcal{L}_{\mathrm{diff}}\hat{V}\right] \dd t + (\text{martingale terms}) \\
&\quad + \lambda^b(\delta^b)\left[\hat{V}(t, X_t - P^b, q_t+1, S_t, \beta_t) - \hat{V}\right] \dd t \\
&\quad + \lambda^a(\delta^a)\left[\hat{V}(t, X_t + P^a, q_t-1, S_t, \beta_t) - \hat{V}\right] \dd t \\
&\quad + (\text{Poisson martingale terms}).
\end{align*}

Since $\theta^*$ solves the HJB equation \eqref{eq:hjb}, the drift under the optimal control is zero.
Under any suboptimal control, the drift is non-positive (because the supremum in \eqref{eq:hjb} is not attained).
Therefore $M_t$ is a supermartingale for any admissible control.

\textbf{Part 2: Martingale property under optimal control.}
Under $(\delta^{b*}, \delta^{a*})$, the drift of $M_t$ vanishes identically.
The martingale terms are $\dd M_t^{\mathrm{mart}} = \gamma q \sigma V \, \dd W_t + \gamma(q + \theta_\beta^*)\sigma_\beta V \, \dd W_t^\beta + (\text{compensated Poisson})$.
Uniform integrability follows from:
\[
\E\left[\sup_{t \leq T} |M_t|^2\right] \leq \E\left[\exp(2\gamma \sup_{t \leq T} |\Phi_t|)\right] < \infty,
\]
which holds under the boundedness of $q_t$ (inventory gating at $\bar{q}$) and the Novikov condition for the Brownian integrals.
Hence $M_t$ is a true martingale.

\textbf{Conclusion:}
For any admissible control: $\hat{V}(0, x_0, 0, S_0, \beta_0) \geq \E[M_T] = \E[-e^{-\gamma W_T}]$.
Under optimal control: $\hat{V}(0, x_0, 0, S_0, \beta_0) = \E[M_T] = \E[-e^{-\gamma W_T}]$.
Therefore $\hat{V} = V$ and the optimal control attains the value.
\end{proof}

\bibliographystyle{plainnat}

\begin{thebibliography}{31}
\providecommand{\natexlab}[1]{#1}
\providecommand{\url}[1]{\texttt{#1}}
\expandafter\ifx\csname urlstyle\endcsname\relax
  \providecommand{\doi}[1]{doi: #1}\else
  \providecommand{\doi}{doi: \begingroup \urlstyle{rm}\Url}\fi

\bibitem[Adams et~al.(2021)Adams, Zinsmeister, Salem, Keefer, and
  Robinson]{adams2021uniswap}
Hayden Adams, Noah Zinsmeister, Moody Salem, River Keefer, and Dan Robinson.
\newblock Uniswap v3 core.
\newblock \emph{White Paper}, 2021.

\bibitem[Angeris et~al.(2020)Angeris, Kao, Chiang, Noyes, and
  Chitra]{angeris2020improved}
Guillermo Angeris, Hsien-Tang Kao, Rei Chiang, Charlie Noyes, and Tarun Chitra.
\newblock An analysis of uniswap markets.
\newblock \emph{Cryptoeconomic Systems}, 2020.

\bibitem[Aquilina et~al.(2022)Aquilina, Budish, and
  O'Neill]{aquilina2022quantifying}
Matteo Aquilina, Eric Budish, and Peter O'Neill.
\newblock Quantifying the high-frequency trading ``arms race''.
\newblock \emph{The Quarterly Journal of Economics}, 137\penalty0 (1):\penalty0
  493--564, 2022.

\bibitem[Avellaneda and Stoikov(2008)]{avellaneda2008high}
Marco Avellaneda and Sasha Stoikov.
\newblock High-frequency trading in a limit order book.
\newblock \emph{Quantitative Finance}, 8\penalty0 (3):\penalty0 217--224, 2008.

\bibitem[Bergault et~al.(2021)Bergault, Evangelista, Gu{\'e}ant, and
  Vieira]{bergault2021multi}
Philippe Bergault, David Evangelista, Olivier Gu{\'e}ant, and Douglas Vieira.
\newblock Multi-asset market making.
\newblock \emph{International Journal of Theoretical and Applied Finance},
  24\penalty0 (06n07):\penalty0 2150040, 2021.

\bibitem[Brekke and {\O}ksendal(1994)]{brekke1994optimal}
Kjell~Arne Brekke and Bernt {\O}ksendal.
\newblock Optimal switching in an economic activity under uncertainty.
\newblock \emph{SIAM Journal on Control and Optimization}, 32\penalty0
  (4):\penalty0 1021--1036, 1994.

\bibitem[Budish et~al.(2015)Budish, Cramton, and Shim]{budish2015high}
Eric Budish, Peter Cramton, and John Shim.
\newblock The high-frequency trading arms race: Frequent batch auctions as a
  market design response.
\newblock \emph{The Quarterly Journal of Economics}, 130\penalty0 (4):\penalty0
  1547--1621, 2015.

\bibitem[Carmona and Ludkovski(2008)]{carmona2008optimal}
Ren{\'e} Carmona and Michael Ludkovski.
\newblock Optimal switching and application to tolling agreements.
\newblock \emph{Annals of Applied Probability}, 18\penalty0 (6):\penalty0
  2340--2370, 2008.

\bibitem[Cartea et~al.(2015)Cartea, Jaimungal, and
  Penalva]{cartea2015algorithmic}
{\'A}lvaro Cartea, Sebastian Jaimungal, and Jos{\'e} Penalva.
\newblock \emph{Algorithmic and High-Frequency Trading}.
\newblock Cambridge University Press, 2015.

\bibitem[Copeland and Galai(1983)]{copeland1983information}
Thomas~E Copeland and Dan Galai.
\newblock Information effects on the bid-ask spread.
\newblock \emph{The Journal of Finance}, 38\penalty0 (5):\penalty0 1457--1469,
  1983.

\bibitem[Daian et~al.(2020)Daian, Goldfeder, Kell, Li, Zhao, Bentov,
  Breidenbach, and Juels]{daian2020flash}
Philip Daian, Steven Goldfeder, Tyler Kell, Yunqi Li, Xueyuan Zhao, Iddo
  Bentov, Lorenz Breidenbach, and Ari Juels.
\newblock Flash boys 2.0: Frontrunning in decentralized exchanges, miner
  extractable value, and consensus instability.
\newblock \emph{2020 IEEE Symposium on Security and Privacy}, pages 910--927,
  2020.

\bibitem[DeMiguel et~al.(2009)DeMiguel, Garlappi, and
  Uppal]{demiguel2009optimal}
Victor DeMiguel, Lorenzo Garlappi, and Raman Uppal.
\newblock Optimal versus naive diversification: How inefficient is the 1/{N}
  portfolio strategy?
\newblock \emph{The Review of Financial Studies}, 22\penalty0 (5):\penalty0
  1915--1953, 2009.

\bibitem[Douady et~al.(2000)Douady, Shiryaev, and Yor]{douady2000maximum}
Rapha{\"e}l Douady, Albert Shiryaev, and Marc Yor.
\newblock Maximum drawdown.
\newblock \emph{Risk Magazine}, 13\penalty0 (4):\penalty0 88--92, 2000.

\bibitem[Easley and O'Hara(1987)]{easley1987price}
David Easley and Maureen O'Hara.
\newblock Price, trade size, and information in securities markets.
\newblock \emph{Journal of Financial Economics}, 19\penalty0 (1):\penalty0
  69--90, 1987.

\bibitem[Foucault et~al.(2017)Foucault, Kozhan, and Tham]{foucault2017toxic}
Thierry Foucault, Roman Kozhan, and Wing~Wah Tham.
\newblock Toxic arbitrage.
\newblock \emph{The Review of Financial Studies}, 30\penalty0 (4):\penalty0
  1053--1094, 2017.

\bibitem[Glosten and Harris(1988)]{glosten1988estimating}
Lawrence~R Glosten and Lawrence~E Harris.
\newblock Estimating the components of the bid/ask spread.
\newblock \emph{Journal of Financial Economics}, 21\penalty0 (1):\penalty0
  123--142, 1988.

\bibitem[Glosten and Milgrom(1985)]{glosten1985bid}
Lawrence~R Glosten and Paul~R Milgrom.
\newblock Bid, ask and transaction prices in a specialist market with
  heterogeneously informed traders.
\newblock \emph{Journal of Financial Economics}, 14\penalty0 (1):\penalty0
  71--100, 1985.

\bibitem[Grossman and Miller(1988)]{grossman1988liquidity}
Sanford~J Grossman and Merton~H Miller.
\newblock Liquidity and market structure.
\newblock \emph{Journal of Finance}, 43\penalty0 (3):\penalty0 617--633, 1988.

\bibitem[Gu{\'e}ant(2017)]{gueant2017optimal}
Olivier Gu{\'e}ant.
\newblock Optimal market making.
\newblock \emph{Applied Mathematical Finance}, 24\penalty0 (2):\penalty0
  112--154, 2017.

\bibitem[Gu{\'e}ant et~al.(2013)Gu{\'e}ant, Lehalle, and
  Fernandez-Tapia]{gueant2013dealing}
Olivier Gu{\'e}ant, Charles-Albert Lehalle, and Joaquin Fernandez-Tapia.
\newblock Dealing with the inventory risk: a solution to the market making
  problem.
\newblock \emph{Mathematics and Financial Economics}, 7\penalty0 (4):\penalty0
  477--507, 2013.

\bibitem[Hasbrouck et~al.(2024)Hasbrouck, Rivera, and
  Saleh]{hasbrouck2024measuring}
Joel Hasbrouck, Thomas~J. Rivera, and Fahad Saleh.
\newblock Measuring liquidity in decentralized finance.
\newblock \emph{Review of Financial Studies}, 2024.
\newblock Forthcoming.

\bibitem[Ho and Stoll(1981)]{ho1981optimal}
Thomas Ho and Hans~R Stoll.
\newblock Optimal dealer pricing under transactions and return uncertainty.
\newblock \emph{Journal of Financial Economics}, 9\penalty0 (1):\penalty0
  47--73, 1981.

\bibitem[Huang and Stoll(1997)]{huang1997components}
Roger~D Huang and Hans~R Stoll.
\newblock The components of the bid-ask spread: A general approach.
\newblock \emph{The Review of Financial Studies}, 10\penalty0 (4):\penalty0
  995--1034, 1997.

\bibitem[Lalor and Swishchuk(2025)]{lalor2025simulation}
Luca Lalor and Anatoliy Swishchuk.
\newblock Market simulation under adverse selection.
\newblock \emph{arXiv preprint arXiv:2409.12721}, 2025.

\bibitem[Ledoit and Wolf(2004)]{ledoit2004well}
Olivier Ledoit and Michael Wolf.
\newblock A well-conditioned estimator for large-dimensional covariance
  matrices.
\newblock \emph{Journal of Multivariate Analysis}, 88\penalty0 (2):\penalty0
  365--411, 2004.

\bibitem[Ma et~al.(2024)Ma, Zhu, Shou, and Chao]{ma2024perpetual}
Mengzhong Ma, Ye~Zhu, Daning Shou, and Xiangrui Chao.
\newblock How decentralized exchange designs shape traders' behavior on
  perpetual future contracts.
\newblock \emph{Electronic Markets}, 34:\penalty0 47, 2024.

\bibitem[Magdon-Ismail et~al.(2004)Magdon-Ismail, Atiya, Pratap, and
  Abu-Mostafa]{magdon2004maximum}
Malik Magdon-Ismail, Amir~F Atiya, Amrit Pratap, and Yaser~S Abu-Mostafa.
\newblock On the maximum drawdown of a {B}rownian motion.
\newblock \emph{Journal of Applied Probability}, 41\penalty0 (1):\penalty0
  147--161, 2004.

\bibitem[Markowitz(1952)]{markowitz1952portfolio}
Harry Markowitz.
\newblock Portfolio selection.
\newblock \emph{The Journal of Finance}, 7\penalty0 (1):\penalty0 77--91, 1952.

\bibitem[Milionis et~al.(2022)Milionis, Moallemi, Roughgarden, and
  Zhang]{milionis2022automated}
Jason Milionis, Ciamac~C Moallemi, Tim Roughgarden, and Anthony~Lee Zhang.
\newblock Automated market making and loss-versus-rebalancing.
\newblock \emph{arXiv preprint arXiv:2208.06046}, 2022.

\bibitem[Park(2021)]{park2021conceptual}
Andreas Park.
\newblock The conceptual flaws of constant product automated market making.
\newblock \emph{Available at SSRN 3805750}, 2021.

\bibitem[Pham(2009)]{pham2009continuous}
Huy{\^e}n Pham.
\newblock \emph{Continuous-time stochastic control and optimization with
  financial applications}.
\newblock Springer, 2009.

\end{thebibliography}

\end{document}